\documentclass[10pt]{article} 
\usepackage[accepted]{tmlr}

\usepackage{amsmath,amsfonts,bm}

\def\eqref#1{equation~\ref{#1}}

\def\1{\bm{1}}

\DeclareMathAlphabet{\mathsfit}{\encodingdefault}{\sfdefault}{m}{sl}
\SetMathAlphabet{\mathsfit}{bold}{\encodingdefault}{\sfdefault}{bx}{n}

\newcommand{\E}{\mathbb{E}}

\newcommand{\R}{\mathbb{R}}

\usepackage{hyperref}
\usepackage{url}

\usepackage{changes}
\usepackage[makeroom]{cancel}

\usepackage{amssymb}
\usepackage{mathtools}
\usepackage{amsthm}
\usepackage[utf8]{inputenc} 
\usepackage[T1]{fontenc}    
\usepackage{hyperref}       
\usepackage{url}            
\usepackage{booktabs}       
\usepackage{amsfonts}       
\usepackage{nicefrac}       
\usepackage{microtype}      
\usepackage{xcolor}         
\usepackage{subfigure}
\usepackage{hyperref}
\usepackage[capitalize,noabbrev]{cleveref}

\usepackage{algorithm}

\newtheorem{theorem}{Theorem}[section]
\newtheorem{lemma}[theorem]{Lemma}
\newtheorem{definition}[theorem]{Definition}

\newtheorem{assumption}[theorem]{Assumption}

\newcommand*\lrw[1]{\left\langle#1\right\rangle}

\newcommand*\lrn[1]{\left\|#1\right\|}

\newcommand{\prob}{{\mathbb P}}

\renewcommand{\tilde}{\widetilde}

\renewcommand{\d}{\mathrm{d}}

\DeclareMathOperator{\D}{{\mathcal D}}
\DeclareMathOperator{\cS}{{\mathcal S}}
\DeclareMathOperator{\M}{{\mathcal M}}

\newcommand{\bbeta}{{\boldsymbol \beta}}

\newcommand{\btheta}{{\boldsymbol \theta}}

\let\originaleqref=\ref
\renewcommand{\eqref}{\originaleqref}

\title{On Convergence of Federated Averaging Langevin Dynamics}

\author{ 
Wei Deng\thanks{\texttt{weideng056@gmail.com}. Machine Learning Research, Morgan Stanley, NY.}  \quad Qian Zhang\thanks{\texttt{zhan3761@purdue.edu}. Department of Statistics, Purdue University, West Lafayette, IN. Co-first author.} \quad Yi-An Ma \thanks{\texttt{yianma.ucsd@gmail.com}. Halicio\~glu Data Science Institute, University of California, San Diego, La Jolla, CA} \quad Zhao Song\thanks{\texttt{zsong@adobe.com}. Adobe Research, San Jose, CA} \quad Guang Lin\thanks{\texttt{guanglin@purdue.edu}. Department of Mathematics and School of Mechanical Engineering, Purdue University.}
}

\begin{document}

\maketitle

\begin{abstract}
We propose a federated averaging Langevin algorithm (FA-LD) for uncertainty quantification and mean predictions  with distributed clients. In particular, we generalize beyond normal posterior distributions and consider a general class of models. We develop theoretical guarantees for FA-LD for strongly log-concave distributions with non-i.i.d data and study how the injected noise and the stochastic-gradient noise, the heterogeneity of data, and the varying learning rates affect the convergence. Such an analysis sheds light on the optimal choice of local updates to minimize the communication cost. Important to our approach is that the communication efficiency does not deteriorate with the injected noise in the Langevin algorithms. In addition, we examine in our FA-LD algorithm both independent and correlated noise used over different clients. We observe there is a trade-off between the pairs among communication, accuracy, and data privacy. As local devices may become inactive in federated networks, we also show convergence results based on different averaging schemes where only partial device updates are available. In such a case, we discover an additional bias that does not decay to zero.
\end{abstract}

\section{Introduction}

Federated learning (FL) allows multiple parties to jointly train a consensus model without sharing user data. 
Compared to the classical centralized learning regime, federated learning keeps training data on local clients, such as mobile devices or hospitals, where data privacy, security, and access rights are a matter of vital interest.
This aggregation of various data resources heeding privacy concerns yields promising potential in areas of internet of things~\citep{cys+20}, healthcare~\citep{lgd+20,lmx+19}, text data~\citep{hsc+20}, and fraud detection~\citep{zygw20}. 

A standard formulation of federated learning is a distributed optimization framework that tackles communication costs, client robustness, and data heterogeneity across different clients \citep{lsts20}. Central to the formulation is the efficiency of the communication, which directly motivates the communication-efficient federated averaging (FedAvg)~\citep{mmr+17}. FedAvg introduces a global model to synchronously aggregate multi-step local updates on the available clients and yields distinctive properties in communication. However, FedAvg often stagnates at inferior local modes empirically due to the data heterogeneity across the different clients \citep{Charles20, Woodworth20}. To tackle this issue, \citet{kkm+20, FedSplit} proposed stateful clients to avoid the unstable convergence, which are, however, not scalable with respect to the number of clients in applications with mobile devices \citep{agxr21}. In addition, the optimization framework often fails to quantify the uncertainty accurately for the parameters of interest, which {is} crucial for building estimators, hypothesis tests, and credible intervals. Such a problem leads to unreliable statistical inference and casts doubts on the credibility of the prediction tasks or  diagnoses in medical applications.

To unify optimization and uncertainty quantification in federated learning, we resort to a \emph{Bayesian treatment by sampling from a global posterior distribution}, where the latter is aggregated by infrequent communications from local posterior distributions. We adopt a popular approach for inferring posterior distributions for large datasets, the stochastic gradient Markov chain Monte Carlo (SG-MCMC) method~\citep{Welling11,VollmerZW2016,Teh16,Chen14,yian2015}, which enjoys theoretical guarantees beyond convex scenarios \citep{Maxim17, Yuchen17, Mangoubi18,ma19}. 
In particular, we examine in the federated learning setting the efficacy of the stochastic gradient Langevin dynamics (SGLD) algorithm, which differs from stochastic gradient descent (SGD) in an additionally injected noise for exploring the posterior. 
The close resemblance naturally inspires us to adapt the optimization-based FedAvg to a distributed sampling framework. 
Similar ideas have been proposed in federated posterior averaging~\citep{agxr21}. {Empirical studies} and analyses on Gaussian posteriors have shown promising potential of this approach. 
Compared to the appealing theoretical guarantees of optimization-based algorithms in federated learning~\citep{FedSplit,agxr21}, the convergence properties of approximate sampling algorithms in federated learning is far less understood. To fill this gap, we proceed by asking the following question:
\begin{center}
    {\it Can we build a unified algorithm with convergence guarantees for sampling in FL?}
\end{center}
In this paper, we make a first step in answering this question in the affirmative. We propose the federated averaging Langevin dynamics for posterior inference beyond the Gaussian distribution. We list our contributions as follows:
\begin{itemize}
    \item We present the first non-asymptotic convergence analysis for FA-LD {for} simulating strongly log-concave distributions on non-i.i.d data; {Our theoretical analysis reveals that SGLD is not communication efficient in federated learning. However, we find that FA-LD based on local updates, with the synchronization frequency depending on the condition number, proves to be an effective approach for alleviating communication overhead.}

    \item The convergence analysis indicates that injected noise, data heterogeneity, and stochastic-gradient noise are all driving factors that affect the convergence. Such an analysis provides a concrete guidance on the optimal number of local updates to minimize communications. 
    \item  We can activate partial device updates to avoid straggler’s effects in practical applications and tune the correlation of injected noises to protect privacy. 
    \item We also provide differential privacy guarantees, which shed light on the trade-off between data privacy and accuracy given a limited budget.
\end{itemize}


\section{Related work}

\paragraph{Concurrent Works}

{Our work predominantly focuses on convex scenarios. However, for those interested in non-convex scenarios, we would like to direct readers to a noteworthy study conducted by \citet{fl_Isoperimetry}, which assumes the Logarithmic Sobolev Inequality (LSI) to hold and leverages the compression operator (also QLSD \citep{Maxime2021}) to reduce communication costs in federated learning. The LSI assumption allows for the consideration of multi-modal distributions and provides theoretical guarantees for more practical applications. Although the compression operator may be less communication efficient than the local-step update, this work \citet{fl_Isoperimetry} intriguingly lays the foundation for future studies on Bayesian federated learning in non-convex scenarios based on local-step schemes.}



It is important to note that our averaging scheme is deterministic, which may have limitations in scenarios where the activation of all devices is costly. For interested readers, we recommend referring to the study conducted by \citet{Vincent_2022} on federated averaging Langevin dynamics, which extends our deterministic averaging scheme to probabilistic. 

\paragraph{Federated Learning}
Current federated learning follows two paradigms. The first paradigm asks every client to learn the model using private data and communicate in model parameters. The second one uses encryption techniques to guarantee secure communication between clients.  In this paper, we focus on the first {paradigm} \citep{dcm+12,ss15,mmra16,mmr+17,hlsy21}. 
There is a long list of works showing provable convergence for FedAvg types of algorithms in the field of optimization \citep{lhy+20,ljz+21,hlsy21,kmr19,yyz19,wts+19,kkm+20}. One line of research \citet{lhy+20,kmr19,yyz19,wts+19,kkm+20} focuses on standard assumptions in optimization (such as, convex, smooth, strongly-convex, bounded gradient). 
Extensions to general partial device participation, and arbitrary communication schemes have been well addressed in \citet{Avdyukhin21, Haddadpour19}. There are other noteworthy Bayesian personalized federated learning algorithms worth exploring. For instance, \citet{FedPop, Personalized_FL_Bayes} present interesting approaches within this domain.


\paragraph{Distributed Monte Carlo methods}

Sub-posterior aggregation was initially proposed in \citet{Neiswanger13, wang13, Minsker14} to accelerate MCMC methods to cope with large datasets. Other parallel MCMC algorithms \citep{Nishihara14, Ahn14_icml, chen16_distributed, Chowdhury18, Li19_v2} propose to improve the efficiency of Monte Carlo computation in distributed or asynchronous systems. \citet{gghz20} proposed stochastic gradient Monte Carlo methods in decentralized systems. \citet{agxr21, F-SGLD, hw+21} introduced empirical studies of posterior averaging in federated learning. 

\section{Preliminaries}\label{sec:preli_fa}

\subsection{An optimization perspective on federated averaging}
\label{fl_ag}
Federated averaging (FedAvg) is a standard algorithm in federated learning and is typically formulated into a distributed optimization framework as follows
\begin{align}\label{optim_perspective}
    \min_{\theta} \ell(\theta):=\frac{\sum_{c=1}^N \ell^c(\theta)}{\sum_{c=1}^N n_c},\quad \ell^c(\theta):= \sum_{i=1}^{n_c} l(\theta; x_{c, i}),
\end{align}
where $\theta\in\mathbb{R}^d$, $l(\theta;x_{c,j})$ is a certain loss function based on $\theta$ and the data point $x_{c,j}$.

FedAvg algorithm requires the following three {steps}:
\begin{itemize}
    \item \emph{Broadcast}: The center server \emph{broadcasts} the latest model, $\theta_k$, to all local clients.
    \item \emph{Local updates}: For any $c\in [N]$, the $c$-th client first sets the auxiliary variable $\beta_k^c=\theta_k$ and then conducts $K\geq 1$ local steps: $\beta_{k+1}^c=\beta_k^c-\frac{\eta}{n_c}\nabla \tilde \ell^c(\beta_k^c),$
where $\eta$ is the learning rate and $\nabla \tilde \ell^c$ is the unbiased estimate of the exact gradient $\nabla \ell^c$.
    \item \emph{Synchronization}: The local models are sent to the center server and then aggregated into a unique model $\theta_{k+K}:=\sum_{c=1}^N p_c \beta_{k+K}^c$, where $p_c$ as the weight of the $c$-th client such that $p_c=\frac{n_c}{\sum_{i=1}^{N} n_i}\in(0, 1)$ and $n_c>0$ is the number of data points in the $c$-th client. 
\end{itemize}
From the optimization perspective, \citet{lhy+20} proved the convergence of the FedAvg algorithm on non-i.i.d data such that a larger number of local steps $K$ and a higher order of data heterogeneity slows down the convergence. Notably, Eq.(\eqref{optim_perspective}) can be interpreted as maximizing the likelihood function, which is a special case of maximum a posteriori estimation (MAP) given a uniform prior.
\subsection{Stochastic gradient Langevin dynamics}
Posterior inference offers the exact uncertainty quantification ability of the predictions. A popular method for posterior inference with large dataset is the stochastic gradient Langevin dynamics (SGLD) ~\citep{Welling11}, which injects additional noise into the stochastic gradient such that $$\theta_{k+1}=\theta_k-\eta\nabla \tilde f(\theta_k)+\sqrt{2\tau \eta}\xi_k,$$ where $\tau$ is the temperature and $\xi_k$ is a standard  Gaussian vector. $f(\theta):=\sum_{c=1}^N \ell^c(\theta)$ is an energy function. $ \nabla \tilde f(\theta)$ is an unbiased estimate of $\nabla f(\theta)$. In the longtime limit, $\theta_k$ converges weakly to the distribution $\pi(\theta)\propto \exp(-{f(\theta)}/{\tau})$ \citep{Teh16} as $\eta\rightarrow 0$.

\section{Posterior inference via federated averaging Langevin dynamics}\label{sec:posterior_inference}
The increasing concern for uncertainty estimation in federated learning motivates us to consider the simulation of the distribution $\pi(\theta)\propto \exp(-{f(\theta)}/{\tau} )$ with distributed clients.

\paragraph{Problem formulation} We propose the federated averaging Langevin dynamics (FA-LD) based on the FedAvg framework in section \ref{fl_ag}. We follow the same \emph{broadcast} step and \emph{synchronization} step but propose to inject random noises for \emph{local updates}. In particular, we consider the following scheme: for any $c\in [N]$, the $c$-th client first sets $\theta_k^c=\theta_k$ and then conducts $K\geq 1$ local steps:
\begin{align}\label{local_independent_noise}
     &\text{\emph{Local updates: }}\notag\\
     &\qquad\qquad\qquad\ \beta_{k+1}^c=\theta_k^c-\eta\nabla \tilde f^c(\theta_k^c)+\sqrt{2\eta\tau} \Xi_k^c\\
     \label{synchronization_main_paper_first}
&\text{\emph{Synchronization: }}\notag\\
&\qquad\qquad\qquad\ \  \theta_{k+1}^c=\left\{  
             \begin{array}{lr}  
             \beta_{k+1}^c \qquad\qquad \text{if } k+1 \text{ mod } K\neq 0 \\  
&               \\
             \sum_{c=1}^N p_c \beta_{k+1}^c \  \text{ if } k+1 \text{ mod } K=0
             \end{array}  
\right.  
\end{align}
where $\nabla f^c(\theta)=\frac{1}{p_c}\nabla \ell^c(\theta)$; $\nabla \tilde f^c(\theta)$ is the unbiased estimate of $\nabla f^c(\theta)$; $\Xi_k^{c}$ is some Gaussian vector in Eq.(\ref{noise_def}). Summing Eq.(\eqref{local_independent_noise}) from clients $c=1$ to $N$, we have
\begin{align*}
    \beta_{k+1}&=\theta_k-\eta {\bm{\widetilde Z}_k}+\sqrt{2\eta\tau}\xi_k,
\end{align*}
\begin{equation}
\begin{split}\label{decomposition}
    &\text{where}\quad \beta_k=\sum_{c=1}^N p_c \beta_k^c,\quad  \theta_k=\sum_{c=1}^N p_c \theta_k^c,\quad {\bm{\widetilde Z}_k}=\sum_{c=1}^N p_c \nabla \tilde f^c(\theta_k^c), \quad \xi_k=\sum_{c=1}^N p_c \Xi_k^c.
\end{split}
\end{equation}

    

By the nature of \emph{synchronization}, we always have $\beta_k=\theta_k$ for any $k\geq 0$ and the process follows
\begin{equation}
\label{fed_avg_langevin_dynamics_main}
\theta_{k+1}=\theta_k-\eta {\bm{\widetilde Z}_k}+\sqrt{2\eta\tau}\xi_k,
\end{equation}
which resembles SGLD except that $\theta_k$ is \emph{not accessible when $k\text{ mod } K\neq 0$}. Since our target is to simulate from $\pi(\theta)\propto \exp( - f(\theta)/\tau )$, we expect $\xi_k$ is a standard Gaussian vector. By the concentration of independent Gaussian variables, it is natural to set
\begin{align}\label{noise_def}
    \Xi_k^c=\xi_k^c/\sqrt{p_c},
\end{align}
where $\xi_k=\sum_{c=1}^N p_c \Xi_k^c=\sum_{c=1}^N \sqrt{p_c} \xi_k^c$ and $\xi_k^c$ is a also standard Gaussian vector. Now we present the algorithm based on independent inject noise ($\rho=0$)  and the full-device update (\ref{synchronization_main_paper_first}) in Algorithm \ref{alg:alg_main_text_partial_main}, where $\rho$ is the the correlation coefficient and will be further studied in section \ref{correlated_sec}. We observe Eq.(\eqref{local_client_main_paper}) maintains a temperature $\tau/p_c>\tau$ to converge to the stationary distribution $\pi$. Such a mechanism may limit the disclosure of individual data and shows a potential to protect the data privacy.

\begin{algorithm*}[h]\caption{Federated averaging Langevin dynamics Algorithm (FA-LD), informal version of Algorithm \ref{alg:alg_main_text_partial}. Denote by $\theta_k^c$ the model parameter in the $c$-th client at the $k$-th step. Denote the one-step intermediate result by $\beta_k^c$. $\xi_k^c$ is an independent standard $d$-dimensional Gaussian vector at iteration $k$ for each client $c\in[N]$; $\dot{\xi}_k$ is a $d$-dimensional Gaussian vector shared by all the clients; $\rho$ denotes the correlation coefficient. $\mathcal{S}_k$ is sampled via a device-sampling rule based on scheme \text{I} or \text{II}.}\label{alg:alg_main_text_partial_main}
\begin{equation}\label{local_client_main_paper}
    \beta_{k+1}^c=\theta_k^c-\eta\nabla \tilde f^c(\theta_k^c)+\sqrt{2\eta\tau \rho^2}\dot\xi_k + \sqrt{2\eta \tau (1-\rho^2)/p_c}\xi_k^c,
\end{equation}
\begin{equation*}  
\theta_{k+1}^c=\left\{  
             \begin{array}{lr}  
             \beta_{k+1}^c \ \ \qquad\qquad\qquad\quad\text{if } k+1 \text{ mod } K\neq 0 \\  
              & \\
             \Pi_{\text{device}}(\beta_{k+1}^c) \ \qquad \qquad \text{if } k+1 \text{ mod } K=0.
             \end{array}
\right.             
\end{equation*}
where $\small{\Pi_{\text{device}}(\beta_{k+1}^c)=\sum_{c=1}^N p_c \beta_{k+1}^c}$ for full device and $\small{=\sum_{c\in \mathcal{S}_{k+1}} \frac{1}{S} \beta_{k+1}^c}$ for partial device.
\end{algorithm*}

\section{Convergence analysis}\label{sec:convergence}


In this section, we show that FA-LD converges to the stationary distribution $\pi(\theta)$ in the 2-Wasserstein ($W_2$) distance at a rate of $O({1}/{\sqrt{T_{\epsilon}}})$ for strongly log-concave and smooth density. The $W_2$ distance is defined between a pair of Borel probability measures $\mu$ and $\nu$ on $\R^d$ as follows  
\begin{align*}
\small
    W_2(\mu, \nu):=\inf_{\textcolor{black}{\Gamma}\in \text{Couplings}(\mu, \nu)}\left(\int\|\bbeta_{\mu}-\bbeta_{\nu}\|_2^2 d \textcolor{black}{\Gamma}(\bbeta_{\mu}, \bbeta_{\nu})\right)^{\frac{1}{2}},
\end{align*}
where $\|\cdot\|_2$ denotes the $\ell_2$ norm on $\mathbb{R}^d$ and the pair of random variables $(\bbeta_{\mu}, \bbeta_{\nu})\in \R^d\times\R^d$ is a coupling with the marginals following $\mathcal{L}(\bbeta_{\mu})=\mu$ and $\mathcal{L}(\bbeta_{\nu})=\nu$. Note that $\mathcal{L}(\cdot)$ denotes a distribution of a random variable. 

\subsection{Assumptions}

We make standard assumptions on the smoothness and convexity of the functions $f^1, f^2,\cdots, f^N$, which naturally yields appealing tail properties of the stationary measure $\pi$. Thus, we no longer require a restrictive assumption on the bounded gradient in $\ell_2$ norm as in \citet{Koloskova19, yyz19, lhy+20}. In addition, to control the distance between $\nabla f^c$ and $\nabla \tilde f^c$, we also assume a bounded variance of the stochastic gradient in assumption \ref{def:variance_main}.

\begin{assumption}[Smoothness]\label{def:smooth_main} For each $c\in [N]$, $f^c$ is $L$-smooth if for some $L>0$ and $ \forall x, y\in \R^d$
\begin{align*}
f^c(y)\leq f^c(x)+\langle \nabla f^c(x),y-x \rangle+\frac{L}{2}\| y-x \|^2_2. 
\end{align*}
\end{assumption}

\begin{assumption}[Strongly convexity]\label{def:strong_convex_main}
For each $c\in [N]$, $f^c$ is $m$-strongly convex if for some $m>0$
\begin{align*}
f^c(x)\geq f^c(y)+\langle \nabla f^c(y),x-y \rangle + \frac{m}{2} \| y-x \|_2^2. 
\end{align*}
\end{assumption}

\begin{assumption}[Bounded variance]\label{def:variance_main}
For each $c\in [N]$, the variance of noise in the stochastic gradient $\nabla \tilde f^c(x)$ in each client is upper bounded such that 
\begin{align*}
\mathbb{E}[ \| \nabla \tilde f^c(x) - \nabla f^c(x) \|_2^2] \leq \sigma^2 d,\quad \forall x\in \R^d.
\end{align*}
\end{assumption}


\paragraph{Quality of non-i.i.d data} Denote by $\theta_*$ the global minimum of $f$. Next, we quantify the degree of the non-i.i.d data by $\gamma:=\max_{c\in[N]}\lrn{\nabla f^c(\theta_*)}_2$, which is non-negative and yields a larger scale if the data is less identically distributed.


\subsection{Proof sketch}

The proof hinges on showing the one-step result in the $W_2$ distance. To facilitate the analysis, we first define an auxiliary continuous-time process $(\bar\theta_t)_{t\geq 0}$ that synchronizes for \textcolor{black}{almost} any time $t\geq 0$ 
\begin{align}
\label{continuous_dynamics_main}
\d \bar\theta_t = - \nabla f(\bar\theta_t) \d t + \sqrt{2\tau} \d \overline{W}_t,
\end{align}
where $\overline{W}$ is a $d$-dimensional Brownian motion. The continuous-time algorithm is known to converge to the stationary distribution $\pi(\bar\theta)\propto e^{-\frac{f(\bar\theta)}{\tau}}$. Assume that $\bar\theta_0$ simulates from the stationary distribution $\pi$, then it follows that $\bar\theta_t\sim\pi$ for any $t\geq 0$.

\subsubsection{Dominated contraction in federated learning}

The first target is to show a contraction property of $\lrn{{\bar\theta}-\theta-\eta(\nabla f({\bar\theta})-{\bm{Z}})}_2^2$, {where ${\bm{Z}}=\sum_{c=1}^N p_c \nabla f^c(\theta^c)$}. {The key challenge is that $\bar\theta$ follows from the continuous diffusion Eq.(\eqref{continuous_dynamics_main}), while $\theta$ follows from Alg.\ref{alg:alg_main_text_partial_main}} based on distributed clients with infrequent communications $\theta=\sum_{c=1}^N p_c \theta^c$ every $K$ iterations, {as such, a divergence issue appears when $\theta\neq \theta^c$. To tackle this issue, we first} consider a standard decomposition 
\begin{align*}
\scriptsize
    &\quad\lrn{{\bar\theta}-\theta-\eta(\nabla f({\bar\theta})-{\bm{Z}})}_2^2=\lrn{{\bar\theta}-\theta}_2^2 -2\eta \underbrace{\langle {\bar\theta}-\theta, \nabla f({\bar\theta})-{\bm{Z}}\rangle}_{\mathcal{I}}+\eta^2 \lrn{\nabla f({\bar\theta})-{\bm{Z}}}_2^2.
\end{align*}
Using Eq.(\eqref{decomposition}), we decompose $\mathcal{I}$ and apply Jensen's inequality to obtain a lower bound of $\mathcal{I}$. In what follows, we have the following lemma. 
\begin{lemma}[Dominated contraction property, informal version of Lemma \ref{contraction}]
\label{contraction_main}
Assume assumptions \ref{def:smooth_main} and \ref{def:strong_convex_main} hold. For any learning rate $\eta \in (0, \frac{1}{L+m}]$, any $\bar\theta$ and $\{\theta^c\}_{c=1}^N \in\mathbb{R}^d$, 
we have
\begin{align*}
\tiny
    &\lrn{{\bar\theta}-\theta-\eta(\nabla f({\bar\theta})-{\bm{Z}})}_2^2\leq (1-\eta m) \cdot \|{\bar\theta}-\theta \|_2^2 +4\eta L\underbrace{\sum_{c=1}^N p_c \cdot  \|\theta^c-\theta \|_2^2 }_{\text{divergence term}},
\end{align*}
\end{lemma}
where  $\theta=\sum_{c=1}^N p_c \theta^c$ and ${\bm{Z}}=\sum_{c=1}^N p_c \nabla f^c(\theta^c)$.

\subsubsection{Bounding divergence}

The following result shows that given a finite number of local steps $K$, the divergence between $\theta^c$ in local client and $\theta$ in the center is bounded in $\ell_2$ norm. Notably, since the Brownian motion leads to a lower order term $O(\eta)$ instead of $O(\eta^2)$, a na\"{i}ve proof framework such as \citet{lhy+20} may lead to a crude upper bound for the final convergence.  

\begin{lemma}[Bounded divergence, informal version of Lemma \ref{divergence}]\label{divergence_main}
Assume assumptions  \ref{def:smooth_main}, \ref{def:strong_convex_main}, and \ref{def:variance_main} hold. For any $\eta \in (0 , 2/m)$ and $\lrn{\theta_0^c-\theta_*}_2^2\leq d\mathcal{D}^2$ for any $c\in[N]$ and some constant $\mathcal{D}$, we have the $\ell_2$ upper bound of the divergence between local clients and the center
\begin{align*}
\small
    \sum_{c=1}^N p_c\E{\|\theta_k^c-\theta_k \|_2^2}&\leq O((K-1)^2\eta^2 d) +O((K-1)\eta d).\notag
\end{align*}
\end{lemma}
The result relies on a uniform upper bound in $\ell_2$ norm, which avoids bounded gradient assumptions.

\subsubsection{Coupling to the stationary process}

Note that $\bar\theta_t$ is initialized from the stationary distribution $\pi$. 
The solution to the continuous-time process Eq.(\eqref{continuous_dynamics_main}) follows:
\begin{align}
\label{solution_continuous_dynamics_main}
    \bar\theta_t=\bar\theta_0 -\int_0^t \nabla f(\bar\theta_s)\d s + \sqrt{2\tau}\cdot\overline{W}_t, \qquad \forall t\geq 0.
\end{align}

Set $t\rightarrow(k+1)\eta$ and $\bar\theta_0\rightarrow\bar\theta_{k\eta}$ for Eq.(\eqref{solution_continuous_dynamics_main}) and consider a \emph{synchronous coupling} such that $W_{(k+1)\eta}-W_{k\eta}:=\sqrt{\eta}\xi_k$ is used to cancel the noise terms, we have
\begin{align}
\label{continuous_one_step_main}
    \bar\theta_{(k+1)\eta}=\bar\theta_{k\eta}-\int_{k\eta}^{(k+1)\eta}\nabla f(\bar\theta_s)\d s + \sqrt{2\tau\eta}\xi_k.
\end{align}

Subtracting Eq.(\eqref{fed_avg_langevin_dynamics_main}) from Eq.(\eqref{continuous_one_step_main}) and taking square and expectation on both sides yield that
\begin{align*}
    &\E{\|\bar\theta_{(k+1)\eta}-\theta_{k+1} \|^2_2}\leq  (1-{\eta m}/{2} ) \cdot \E{\|\bar\theta_{k\eta}-\theta_k\|_2^2}+\text{divergence term} + \text{time error}.
\end{align*}

Eventually, we arrive at the one-step error bound for establishing the convergence results.

\begin{lemma}[One step update, informal version of Lemma \ref{one_step_Dalalyan}]\label{one_step_Dalalyan_main}

Assume assumptions \ref{def:smooth_main}, \ref{def:strong_convex_main}, and \ref{def:variance_main} hold. Consider Algorithm \ref{alg:alg_main_text_partial_main} with any $\eta \in (0 , \frac{1}{2L})$ and $\lrn{\theta_0^c-\theta_*}_2^2\leq d\mathcal{D}^2$, $\rho=0$, and full device participation for any $c\in[N]$, where $\theta_*$ is the global minimum for the function $f$. Then
\begin{align*}
    &W_2^2(\mu_{k+1}, \pi)\leq  (1-{\eta m}/{2}) \cdot W^2_2(\mu_{k}, \pi)+ O(\eta^2 d((K-1)^2+\kappa)),
\end{align*}
where $\mu_k$ denotes the probability measure of $\theta_k$ and $\kappa=L/m$ is the condition number.
\end{lemma}
Given small enough $\eta$, the above Lemma indicates that the algorithm will eventually converge
\vspace{-2mm}
\subsection{Full device participation}

\subsubsection{Convergence based on independent noise}
\label{ind_converge}
When the synchronization step is conducted at every iteration $k$, the FA-LD algorithm is essentially the standard SGLD algorithm \citep{Welling11}. Theoretical analysis based on the 2-Wasserstein distance has been established in \citet{dm+16, Dalalyan17, dk19}. However, in scenarios of $K> 1$ with distributed clients, a divergence between the global variable $\theta_k$ and local variable $\theta^c_k$ appears and unavoidably affects the performance. The upper bound on the sampling error is presented as follows.

\begin{theorem}[Main result, informal version of Theorem \ref{main_theorem}]\label{main_paper_theorem} Assume assumptions \ref{def:smooth_main}, \ref{def:strong_convex_main}, and \ref{def:variance_main} hold. Given Algorithm \ref{alg:alg_main_text_partial_main} with $\eta\in (0, \frac{1}{2L}]$, $\rho=0$, full device, and well initialized $\{\theta_0^c\}_{c=1}^N$, we have 
\begin{align*}
    &W_2(\mu_{k}, \pi) \leq  \left(1- {\eta m}/{4}\right)^k \cdot \bigg(\sqrt{2d}\big(\mathcal{D} +  \sqrt{\tau/m} \big)\bigg)+30\kappa\sqrt{\eta m d } \cdot \sqrt{((K-1)^2+\kappa)H_0} .\notag
\end{align*}
where $\mu_k$ denotes the density of $\theta_k$ at iteration $k$, $K$ is the local updates, $\kappa :=L/m$, $\gamma:=\max_{c\in[N]}\lrn{\nabla f^c(\theta_*)}_2$, and $H_{0} := \mathcal{D}^2+\max_{c\in[N]}\frac{\tau}{mp_c} +\frac{\gamma^2}{m^2 d}+\frac{\sigma^2}{m^2}$.

\end{theorem}

We observe that the initialization, injected noise, data heterogeneity, and stochastic gradient noise all affect the convergence. Similar to \citet{lhy+20}, FA-LD with $K$-local steps resembles the one-step SGLD with a large learning rate \textcolor{black}{and the result is consistent with the optimal rate \citep{dm+16}, despite multiple inaccessible local updates. Nevertheless, given more smoothness assumptions, we may obtain a better dimension dependence \citep{dm+16, Ruilin22}}. Bias reduction  \citep{SCAFFOLD} can be further adopted to alleviate the data heterogeneity. 


\textbf{Computational Complexity} To achieve the precision $\epsilon$ based on the learning rate $\eta$, we can set 
\begin{align*}
    &30\kappa\sqrt{\eta m d} \cdot  \sqrt{(K^2+\kappa)H_0}  \leq {\epsilon}/{2},\\
    &\exp\big(-\frac{\eta m}{4} T_{\epsilon} \big) \cdot  \sqrt{2d} (\mathcal{D} +  \sqrt{\tau/m}  ) \leq {\epsilon}/{2}.
\end{align*}
It yields $\eta m\leq  O\bigg(\frac{\epsilon^2}{d\kappa^2  {(K^2+\kappa)H_0}}\bigg),\quad T_{\epsilon}\geq \Omega\bigg(\frac{\log({d}/{\epsilon})}{m\eta}\bigg).$


{Let $T_{\epsilon}$ denote the number of iterations required to achieve the target accuracy of $\epsilon$. By substituting into the upper bound of $\eta m$, we can determine that setting $T_{\epsilon}=\Omega(\epsilon^{-2}d\kappa^2 {(K^2+\kappa)H_0} \cdot \log({d}/{\epsilon}))$ is sufficient. This finding is consistent with the results obtained in \cite{dk19} in terms of dimension dependence when treating $W_2(\mu_0, \pi)$ as a constant. Furthermore, this result has been extended to more general distributional assumptions \citep{fl_Isoperimetry} and has been further refined through variance reduction or bias reduction techniques \citep{Vincent_2022}.}

\textbf{Optimal choice of $K$.}  Note that $H_0 = \Omega(\mathcal{D}^2)$, thus the number of communication rounds is of the order $\frac{T_{\epsilon}}{K}=\Omega\bigg( K+\frac{\kappa}{K}\bigg),$
where the value of $\frac{T_{\epsilon}}{K}$ first decreases and then increases w.r.t. $K$, which indicates setting $K$ either too large or too small leads to high communication costs. Ideally, $K$ should be selected in the scale of $\Omega(\sqrt{\kappa})$. Combining the definition of $T_{\epsilon}$, this shows that the optimal $K$ for FA-LD is in the order of $O(\sqrt{T_{\epsilon}})$, which matches the optimization-based results \citep{Stich19, lhy+20}. {We also acknowledge that our analysis regarding the optimal choice of $K$ is locally optimal, partly due to the suboptimal dependence on the condition number. We believe that a more refined analysis, as presented in \citep{Durmus_LMC_convex}, can help us refine our dependence on $\kappa$ and enhance our understanding of the optimal choice of local steps $K$.}

\subsubsection{Convergence guarantees via varying learning rates}

\begin{theorem}[Informal version of Theorem \ref{main_theorem_decay}]\label{main_paper_theorem_decay} Assume assumptions \ref{def:smooth_main}, \ref{def:strong_convex_main}, and \ref{def:variance_main} hold. Consider Algorithm \ref{alg:alg_main_text_partial_main} with $\rho=0$, full device, an initialization satisfying $\lrn{\theta_0^c-\theta_*}_2^2\leq d\mathcal{D}^2$ for any $c\in[N]$, and the varying learning rate following $\eta_{k}=\frac{1}{2L+(1/12)m k}$. Then for any $k\geq 0$, we have $\ W_2(\mu_{k}, \pi)\leq 45\kappa\sqrt{ ((K-1)^2+\kappa)H_0}\cdot\big(\eta_k m d\big)^{1/2}, \ \  \forall k \geq 0.$
\end{theorem}

To achieve the precision $\epsilon$, we need $ W_2(\mu_{k}, \pi)\leq \epsilon$, i.e. $W_2(\mu_{k}, \pi)\leq 45\kappa\sqrt{ (K^2+\kappa)H_0} \cdot \bigg(\frac{md}{2L+(1/12){mk}}\bigg)^{1/2}.$
We therefore require ${\Omega} ( \epsilon^{-2} d )$ iterations to achieve the precision $\epsilon$, which improves the $\Omega( \epsilon^{-2} d \log( {d}/{\epsilon} ))$ rate for FA-LD with a fixed learning rate by a $O(\log(d/\epsilon))$ factor.

\subsubsection{Privacy-accuracy trade-off via correlated noises}\label{correlated_sec}

The local updates in Eq.(\ref{local_independent_noise}) with $\Xi_k^c=\xi_k^c/\sqrt{p_c}$ requires all the local clients to generate the independent noise $\xi^c_k$. Such a mechanism enjoys the implementation convenience and yields a potential to protect the data privacy and alleviates the security issue. However, the large scale noise inevitably slows down the convergence. To handle this issue, the independent noise can be generalized to correlated noise based on a correlation coefficient $\rho$ between different clients. {Replacing Eq.(\eqref{local_independent_noise})} with 
\begin{equation}\label{local_client_diff_seeds_main_paper}
    \beta_{k+1}^c=\theta_k^c-\eta\nabla \tilde f^c(\theta_k^c)+\sqrt{2\eta\tau \rho^2}\dot{\xi}_k + \sqrt{2\eta(1-\rho^2)\tau/p_c}\xi_k^c,
\end{equation}
where $\dot{\xi}_k$ is a $d$-dimensional standard Gaussian vector shared by all the clients at iteration $k$ and $\dot\xi_k$ is independent with $\xi_k^c$ for any $c\in[N]$. Following the synchronization step based on Eq.(\eqref{synchronization_main_paper_first}), we have
\begin{equation}
\label{fed_avg_langevin_dynamics_pp_main_paper}
\theta_{k+1}=\theta_k-\eta \nabla \tilde f(\theta_k)+\sqrt{2\eta\tau}\xi_k,
\end{equation}
where $\xi_k=\rho \dot\xi_k + \sqrt{1-\rho^2}\sum_{c=1}^N \sqrt{p_c}\xi_k^c$. Since the variance of i.i.d variables is additive, it is clear that $\xi_k$ follows the standard $d$-dimensional Gaussian distribution. The correlated noise implicitly reduces the temperature and naturally yields a trade-off between federation and accuracy. 

Since the inclusion of correlated noise doesn't affect the iterate of Eq.(\eqref{fed_avg_langevin_dynamics_pp_main_paper}), the algorithm property maintains the same except the scale of the temperature $\tau$ and efficacy of federation are changed. Based on a target correlation coefficient $\rho\geq 0$, Eq.(\eqref{local_client_diff_seeds_main_paper}) is equivalent to applying a temperature $T_{c,\rho}=\tau(\rho^2+(1-\rho^2)/p_c)$. In particular, setting $\rho=0$ leads to $T_{c, 0}=\tau/p_c$, which exactly recovers Algorithm \ref{alg:alg_main_text_partial_main}; however, setting $\rho=1$ leads to $T_{c, 1}=\tau$, where the injected noise in local clients is reduced by $1/p_c$ times. Now we adjust the analysis as follows
\begin{theorem}[Informal version of Theorem \ref{correlated_noise_supp}]\label{correlated_noise_main} Assume assumptions \ref{def:smooth_main}, \ref{def:strong_convex_main}, and \ref{def:variance_main} hold.  Consider Algorithm \ref{alg:alg_main_text_partial_main} with $\rho\in[0, 1]$, $\eta\in (0, \frac{1}{2L}]$ and $\lrn{\theta_0^c-\theta_*}_2^2\leq d\mathcal{D}^2$ for any $c\in[N]$, we have
\begin{align*}
    &W_2(\mu_{k}, \pi) \leq  (1-{\eta m}/{4} )^k \cdot \bigg(\sqrt{2d}\big(\mathcal{D} +  \sqrt{\tau/m} \big)\bigg)+30\kappa\sqrt{\eta m d } \cdot \sqrt{((K-1)^2+\kappa)H_{\rho}},\notag
\end{align*}
where $\mu_k$ denotes the probability measure of $\theta_k$, $H_{\rho}: = { \mathcal{D}^2}+\frac{1}{m}\max_{c\in[N]} T_{c,\rho} +{\frac{\gamma^2}{m^2d}}+{\frac{\sigma^2}{m^2}}$.
\end{theorem}

Such a mechanism leads to a trade-off between data privacy and accuracy and may motivate us to exploit the optimal $\rho$ under differential privacy theories \citep{mama15}.

\subsection{Partial device participation}

Full device participation enjoys appealing convergence properties. However, it suffers from the straggler's effect in real-world applications, where the communication is limited by the slowest device. Partial device participation handles this issue by only allowing a small portion of devices in each communication and greatly increased the communication efficiency 
in a federated network. 

The first device-sampling scheme \text{I} \citep{LS20} selects a total of $S$ devices, where the $c$-th device is selected with a probability $p_c$. The first theoretical justification for convex optimization has been proposed by \citet{lhy+20}. The second device-sampling scheme \text{II} is to uniformly select $S$ devices without replacement. We follow  \citet{lhy+20} and assume $S$ indices are selected uniformly without replacement. In addition, the convergence also requires an additional assumption on balanced data \citep{lhy+20}. Both schemes are formally defined in section \ref{unbiased_sampling_schems_appendix}.

\begin{theorem}[Informal version of Theorem \ref{theorem_partial}]\label{thm:partial_II}
Under mild assumptions, we run Algorithm \ref{alg:alg_main_text_partial_main} with $\rho\in[0, 1]$, a fixed $\eta\in (0, \frac{1}{2L}]$ and $\lrn{\theta_0^c-\theta_*}_2^2\leq d\mathcal{D}^2$ for any $c\in[N]$, we have
\vspace{-2mm}
\begin{align*}
\tiny
    &W_2(\mu_{k}, \pi) \leq  (1-{\eta m}/{4} )^k \cdot \bigg(\sqrt{2d}\big(\mathcal{D} +  \sqrt{\tau/m} \big)\bigg)+30\kappa\sqrt{\eta m d } \cdot \sqrt{ H_{\rho}(K^2+\kappa)}+O\bigg(\sqrt{\frac{d}{S}(\rho^2+N(1-\rho^2)) C_S}\bigg),
\end{align*}
where $C_S=1$ for {Scheme I} and $C_S=\frac{N-S}{N-1}$ for {Scheme II}. 
\end{theorem}
\vspace{-2mm}

Partial device participation leads to an extra bias regardless of the scale of $\eta$. To reduce it, we suggest to consider highly correlated injected noise, such as $\rho=1$, to reduce the impact of the injected noise. Further setting $O(\sqrt{{d}/{S}})\leq {\epsilon}/{3}$ and following a similar $\eta$ as in section \ref{ind_converge}, we can achieve the precision $\epsilon$ within $\Omega( \epsilon^{-2} d \log( {d}/{\epsilon} ))$ iterations given enough devices satisfying $S = \Omega( \epsilon^{-2} d )$.


The device-sampling scheme \text{I} provides a viable solution to handle the straggler's effect, which is rather robust to the data heterogeneity and doesn't require the data to be balanced. In more practical cases where a system can only operate based on the first $S$ messages for the local updates, Scheme \text{II} can achieve a reasonable approximation given more balanced data with  uniformly sampled device. If $S=1$, our Scheme \text{II} matches the result in the Scheme \text{I}; If $S=N$, then our Scheme II recovers the result in the full device setting; If $S= N - o(N)$, our Scheme II bound is better than scheme I.


\section{Differential Privacy Guarantees}
We consider the $(\epsilon,\delta)$-differential privacy with respect to the substitute-one relation $\simeq_{s}$ \cite{NEURIPS2018_3b5020bb}. Two datasets $\cS\simeq_{s}\cS'$ if they have the same size and differ by exactly one data point. For $\epsilon\ge 0$ and $\delta\in[0,1]$, a mechanism $\M$ is $(\epsilon,\delta)$-differentially private w.r.t. $\simeq_{s}$ if for any pair of input datasets $\cS\simeq_{s}\cS'$, 
and every measurable subset $E\subset \textup{Range}(\M)$, we have
\begin{equation} \label{eq:DP-def}
\prob[\M(\cS)\in E]\le e^{\epsilon}\prob[\M(\cS')\in E]+\delta.
\end{equation}



Since partial device participation is more general, we focus on analyzing the differential privacy guarantee based on updates with partial devices. Here, we present the result under scheme II. For the result under scheme I, please refer to Theorem \ref{thm:privacy_alg3_full} in the appendix.
\begin{theorem}[Partial version of Theorem \ref{thm:privacy_alg3_full}] \label{thm:privacy_alg3}
Assume assumptions \ref{assump:bdd_sens} and \ref{assump:grad_est} holds. For any $\delta_0\in(0,1)$, if $\eta\in\left(0,\frac{\tau(1-\rho^2)\gamma^2\min_{c\in[N]}p_{c}}{\Delta_l^2\log(1.25/\delta_0)}\right]$, then Algorithm $\ref{alg:alg_main_text_partial_main}$ under scheme II is $(\epsilon^{(3)}_{K,T},\delta^{(3)}_{K,T})$-differentially private w.r.t. $\simeq_{s}$ after $T$ ($T=EK$ with $E\in \mathbb{N}, E\ge 1$) iterations where
\begin{align*}
\small
&\epsilon^{(3)}_{K,T}=\tilde\epsilon_K\min\left\{\sqrt{\frac{2T}{K}\log\left(\frac{1}{\delta_2}\right)} + \frac{TS (e^{\epsilon_{K}}-1)}{KN},\ \frac{T}{K}\right\},\\
&\delta^{(3)}_{K,T}=\frac{S}{N}\gamma T\delta_0+ \frac{TS}{KN}\delta_1+\delta_2,
\end{align*}
with $\tilde\epsilon_K= \log\left(1+\frac{S}{N}\left(e^{\epsilon_K}-1\right)\right)$, 
$\epsilon_K=\epsilon_1\min\left\{\sqrt{2K\log(1/\delta_1)} + K(e^{\epsilon_1}-1),\ 
K\right\}$, \\
$\epsilon_1=2\Delta_l \sqrt{\frac{\eta\log(1.25/\delta_0)}{\tau(1-\rho^2)\min_{c\in [N]}p_{c}}}$,
and $\delta_1,\delta_2\in[0,1)$.
\end{theorem}
According to Theorem \ref{thm:privacy_alg3} and section \ref{sec:DP_discussion}, Algorithm \ref{alg:alg_main_text_partial_main} is at least $(\frac{T}{K}\log\left(1+\frac{S}{N}(e^{K\epsilon_1}-1)\right),\frac{S}{N}\gamma T\delta_0)$-differentially private. 
Moreover, if 
$\eta=O\left(\frac{\tau(1-\rho^2)N^2\min_{c\in[N]}p_c\log(1/\delta_2)}{\Delta_l^2 S^2 T\log(1/\delta_0) \log(1/\delta_1)}\right)$, then we have that $\epsilon_{K,T}^{(3)}=O\left(\frac{S\Delta_l}{N}\sqrt{\frac{\eta T \log(1/\delta_0)\log(1/\delta_1)\log(1/\delta_2)}{\tau(1-\rho^2)\min_{c\in[N]}p_c}}\right)$.

There is a trade-off between privacy and utility. By Theorem \ref{thm:privacy_alg3}, $\epsilon^{(3)}_{K,T}$ is an increasing function of $\frac{\eta}{\tau(1-\rho^2)}$, $\frac{S}{N}$, and $T$. $\delta^{(3)}_{K,T}$ is an increasing function of $\frac{S}{N}$, $\gamma$, and $T$. However, by Theorem \ref{thm:partial_II}, the upper-bound of $W_{2}(\mu_T,\mu)$ is a decreasing function of $\rho$, $T$, $S$ and is an increasing function of $\tau$ and $N$. There is an optimal $\eta$ to minimize 
$W_{2}(\mu_T,\mu)$ for fixed $T$ while we can make $\epsilon^{(3)}_{K,T}$ arbitrarily small by decreasing $\eta$ for any fixed $T$. In practice, users can tune hyper-parameters based on DP and accuracy budget. For example, under some DP budget $(\epsilon_*,\delta_*)$, we can select the largest $\rho\in[0,1]$ and $S\in[N]$ such that $\epsilon^{(3)}_{K,T}\le \epsilon_*$ and $\delta^{(3)}_{K,T}\le \delta_*$ to achieve the target error $W_{2}(\mu_{T},\mu)$.

\label{DP_section}

\vspace{3mm}
\section{Experiments}\label{simulation_local_step}
\subsection{Simulations} 

For each $c\in[N]$, where $N=50$, we sample $\theta_c$ from a 2d Gaussian distribution $N(0,\alpha I_2)$ and sample $n_c$ points from  $N(\theta_c,\Sigma)$, where 
$
\small{\Sigma=\left[\begin{matrix}
5 & -2\\
-2 & 1
\end{matrix}\right]}
$. 
Thus, $l(\theta;x_{c,i})=\frac{1}{2}(\theta-x_{c,i})^{\top}\Sigma^{-1}(\theta-x_{c,i})+\log(2\pi |\Sigma|^{\frac{1}{2}})$,  $\ell^c(\theta)=\sum_{i=1}^{n_c}l(\theta;x_{c,i})$. 
The temperature is $\tau=1$. The target density is $N(u,\frac{1}{n}\Sigma)$ with $u=\frac{1}{n}\sum_{c=1}^N\sum_{i=1}^{n_c}x_{c,i}$. We choose a Gaussian posterior to facilitate the calculation of the W2 distance to verify theoretical properties. 

We repeat each experiment $R=300$ times. At the $k$-th communication round, we obtain a set of $R$ simulated parameters $\{\theta_{k,j}\}_{j=1}^R$, where $\theta_{k,j}$ denotes the parameter at the $k$-th round in the $j$-th independent run. The underlying distribution $\mu_k$ at round $k$ is approximated by a Gaussian variable with the empirical mean ${u}_{k}=\frac{1}{R}\sum_{j=1}^R\theta_{k,j}$ and covariance matrix ${\Sigma}_{k}=\frac{1}{R-1}\sum_{j=1}^R (\theta_{k,j}-{u}_k)(\theta_{k,j}-{u}_k)^{\top}$. 

 \begin{figure*}[htbp]
  \vskip -0.0in
    \centering
    \subfigure[Study of $K$]{
    \begin{minipage}[t]{0.19\linewidth}
    \centering
    \label{fig:optimalK}
    \includegraphics[width=1.1in]{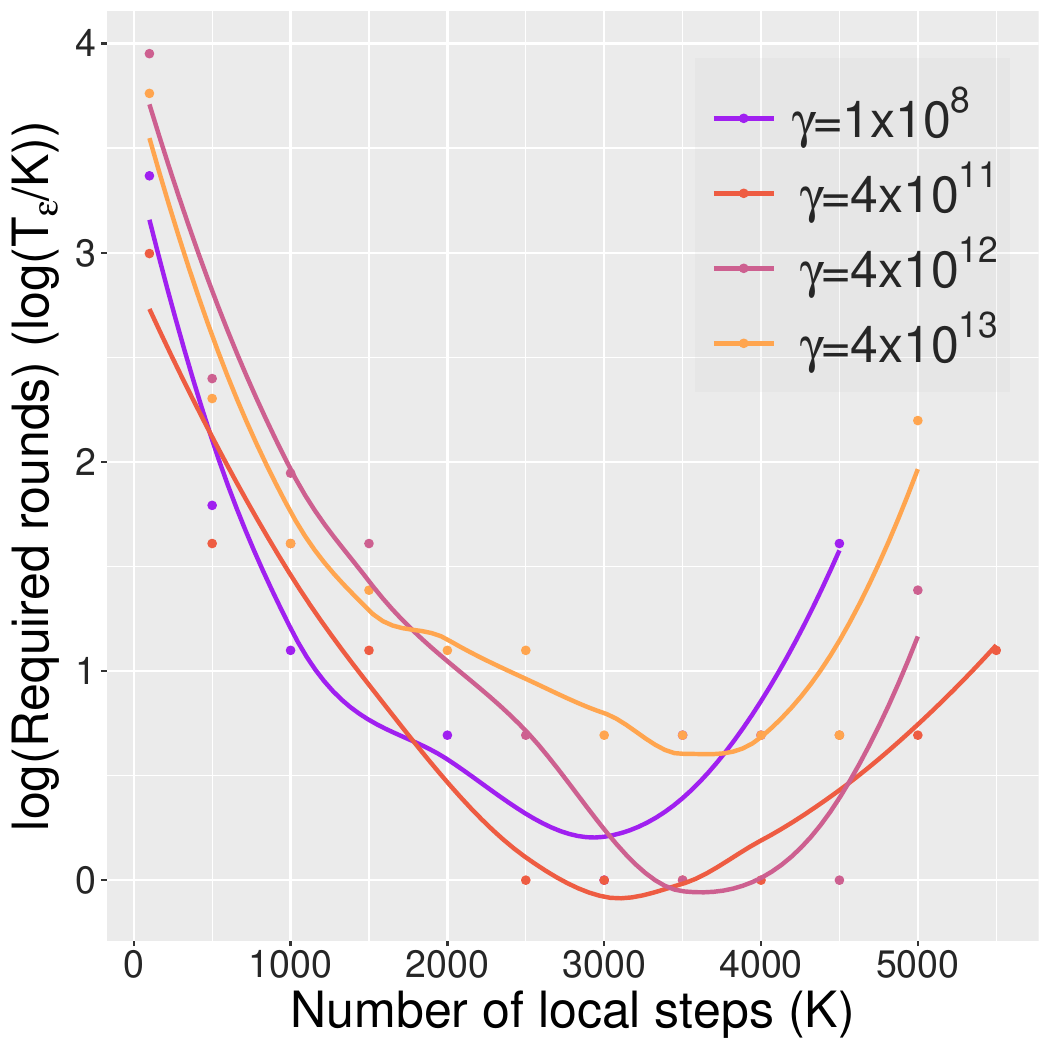}
    \end{minipage}%
    }%
    \subfigure[Study of $\gamma$]{
    \begin{minipage}[t]{0.19\linewidth}
    \centering
    \label{fig:alpha}
    \includegraphics[width=1.1in]{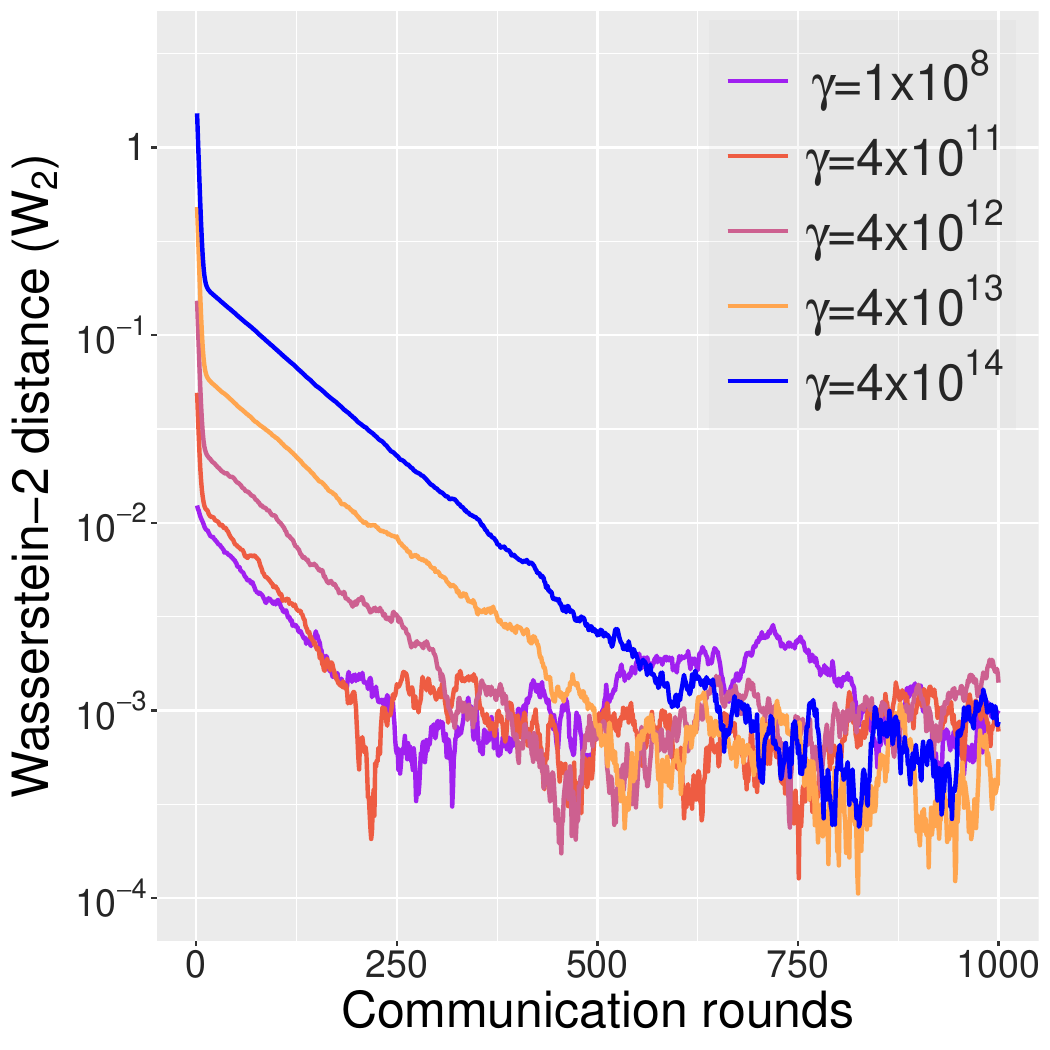}
    \end{minipage}%
    }%
    \subfigure[Study of $\rho$]{
    \begin{minipage}[t]{0.19\linewidth}
    \centering
    \label{fig:rho}
    \includegraphics[width=1.1in]{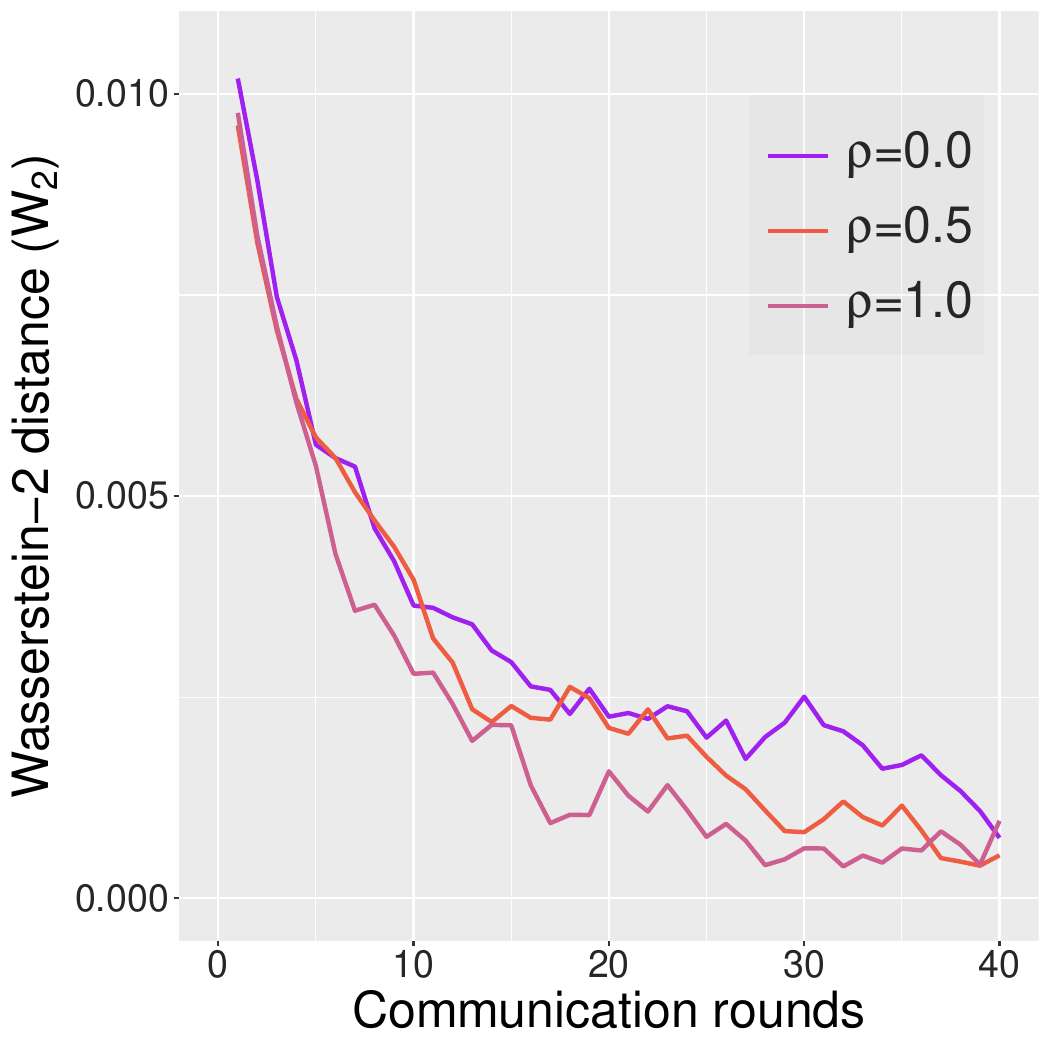}
    \end{minipage}%
    }%
    \subfigure[True density]{
    \begin{minipage}[t]{0.19\linewidth}
    \centering
    \label{fig:true_density}
    \includegraphics[width=1.1in]{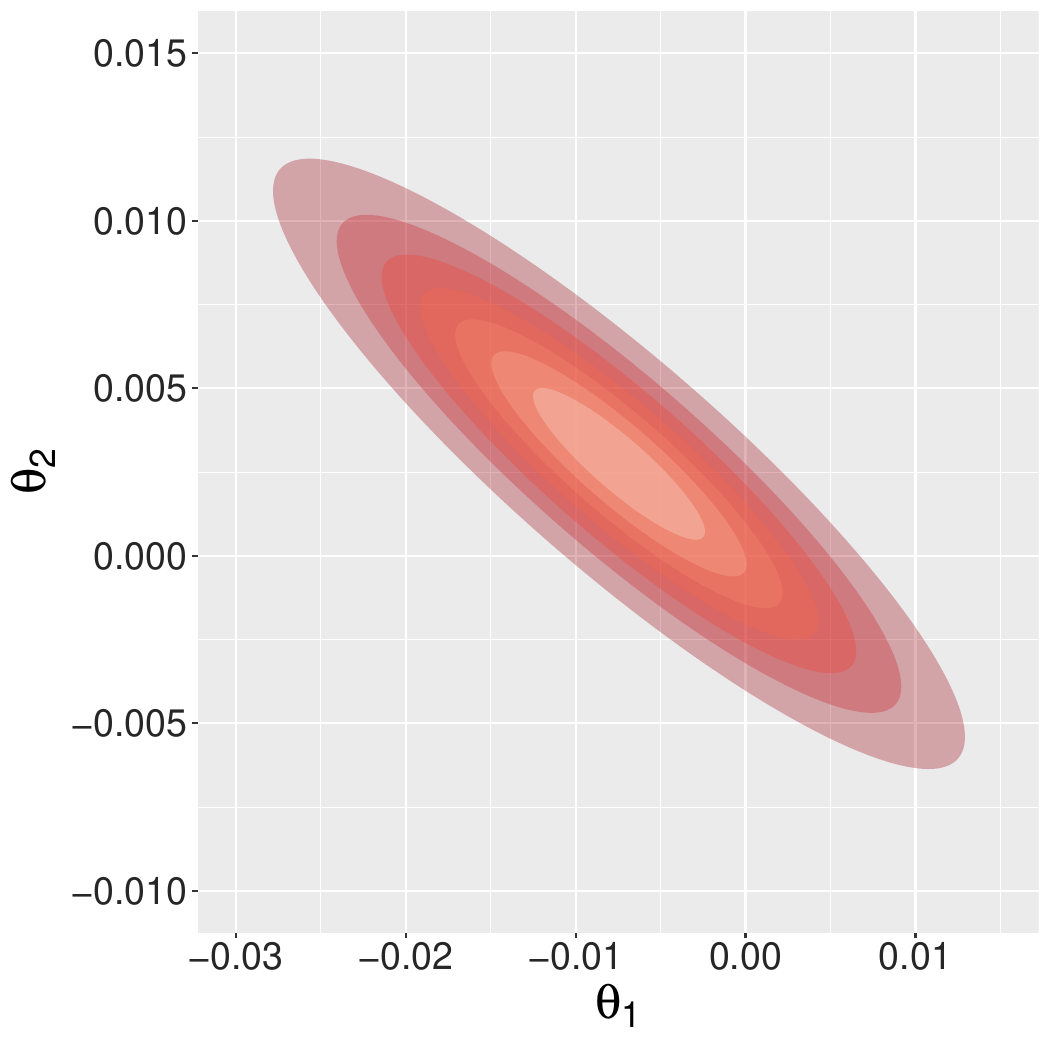}
    \end{minipage}%
    }%
    \subfigure[Empirical density]{
    \begin{minipage}[t]{0.19\linewidth}
    \centering
    \label{fig:empirical_density}
    \includegraphics[width=1.1in]{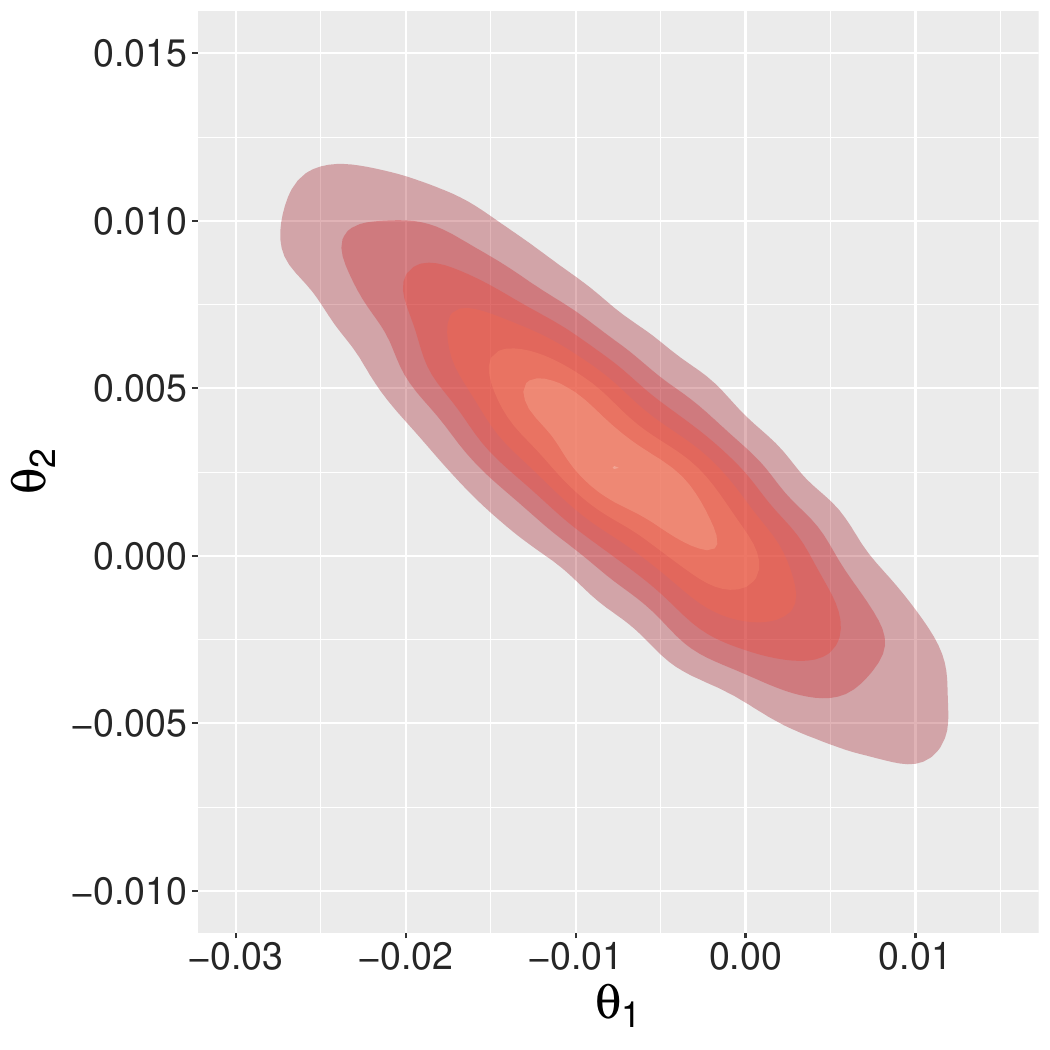}
    \end{minipage}%
    }%
  \vskip -0.0in
  \caption{Convergence of FA-LD based on full devices. In Figure \ref{fig:optimalK},  points may coincide.}
  \label{figure:full_device}
  \vskip -0.0in
\end{figure*}

\emph{Optimal local steps: } We study the choices of local step $K$ for Algorithm \ref{alg:alg_main_text_partial_main} based on $\rho=0$, full device, and different $\alpha$'s, which corresponds to different levels of data heterogeneity modelled by $\gamma$. We choose $\alpha=0, 1, 10, 100, 1000$ and the corresponding $\gamma$ is around $1\times 10^{8},4\times 10^{11}, 4\times 10^{12}, 4\times 10^{13}$, and $4\times 10^{14}$, respectively.  We fix $\eta=10^{-7}$. We evaluate the (log) number of communication rounds to achieve the  accuracy $\epsilon=10^{-3}$ and denote it by $T_{\epsilon}$. As shown in Figure \ref{fig:optimalK}, a small $K$ leads to an excessive amount of communication costs; by contrast, a large $K$ results in large biases, which in turn requires high communications. The optimal $K$ 
is around 3000 
and \emph{the communication savings can be as large as 30 times}. 

\emph{Data heterogeneity and correlated noise: } We study the impact of $\gamma$ on the convergence based on $\rho=0$, full device, different $\gamma$ from $\{1\times 10^{8},4\times 10^{11}, 4\times 10^{12}, 4\times 10^{13}$, and $4\times 10^{14}\}$. We set $K=10$. As shown in Figure \ref{fig:alpha}, the $W_2$ distances under different $\gamma$ all converge to some levels around $10^{-3}$ after sufficient computations. Nevertheless, a larger $\gamma$ does slow down the convergence, which suggests adopting more balanced data to facilitate the computations. In Figure \ref{fig:rho}, we study the impact of $\rho$ on the convergence of the algorithm. We choose $K=100$ and $\gamma=10^8$ and observe that a larger correlation slightly accelerates the computation, although it risks in privacy concerns.

\emph{Approximate samples: } In Figure \ref{fig:empirical_density}, we plot the empirical density according to the samples from Algorithm \ref{alg:alg_main_text_partial_main} with $\rho=0$, full device, $K=10$ and $\gamma=10^{8}$, $\eta=10^{-7}$. For comparison, we show the true density plot of the target distribution in Figure \ref{fig:true_density}. The empirical density approximates the true density very well, which indicates that the potential of FA-LD in federated learning.

\begin{figure*}[htbp]
    \vskip -0.0in
    \centering
    \subfigure[\scriptsize{Full devices: $S=50$}]{
    \begin{minipage}[t]{0.19\linewidth}
    \centering
    \label{full_device_baseline}
    \includegraphics[width=1.1in]{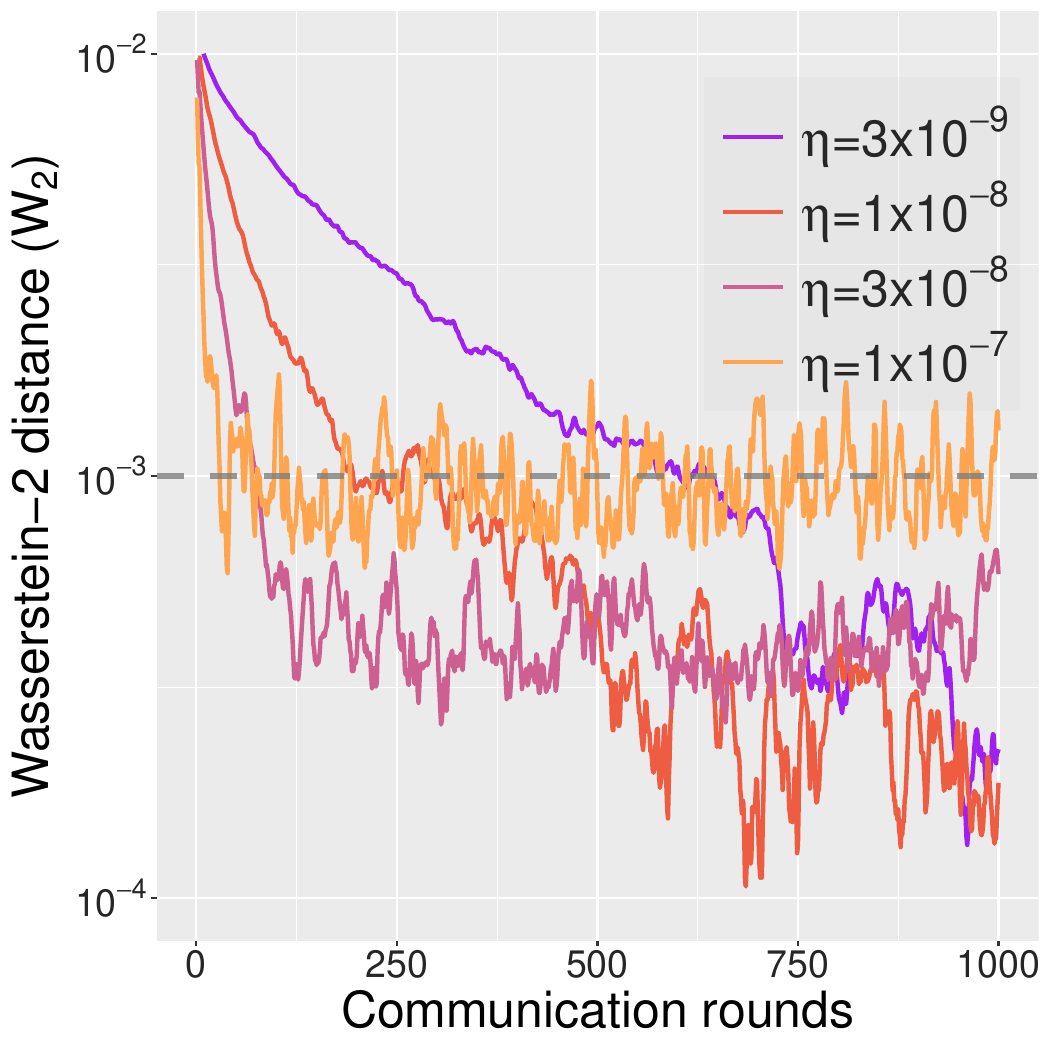}
    \end{minipage}%
    }%
    \subfigure[\scriptsize{Scheme I: $S=40$}]{
    \begin{minipage}[t]{0.19\linewidth}
    \centering
    \label{partial_device_s1_40}
    \includegraphics[width=1.1in]{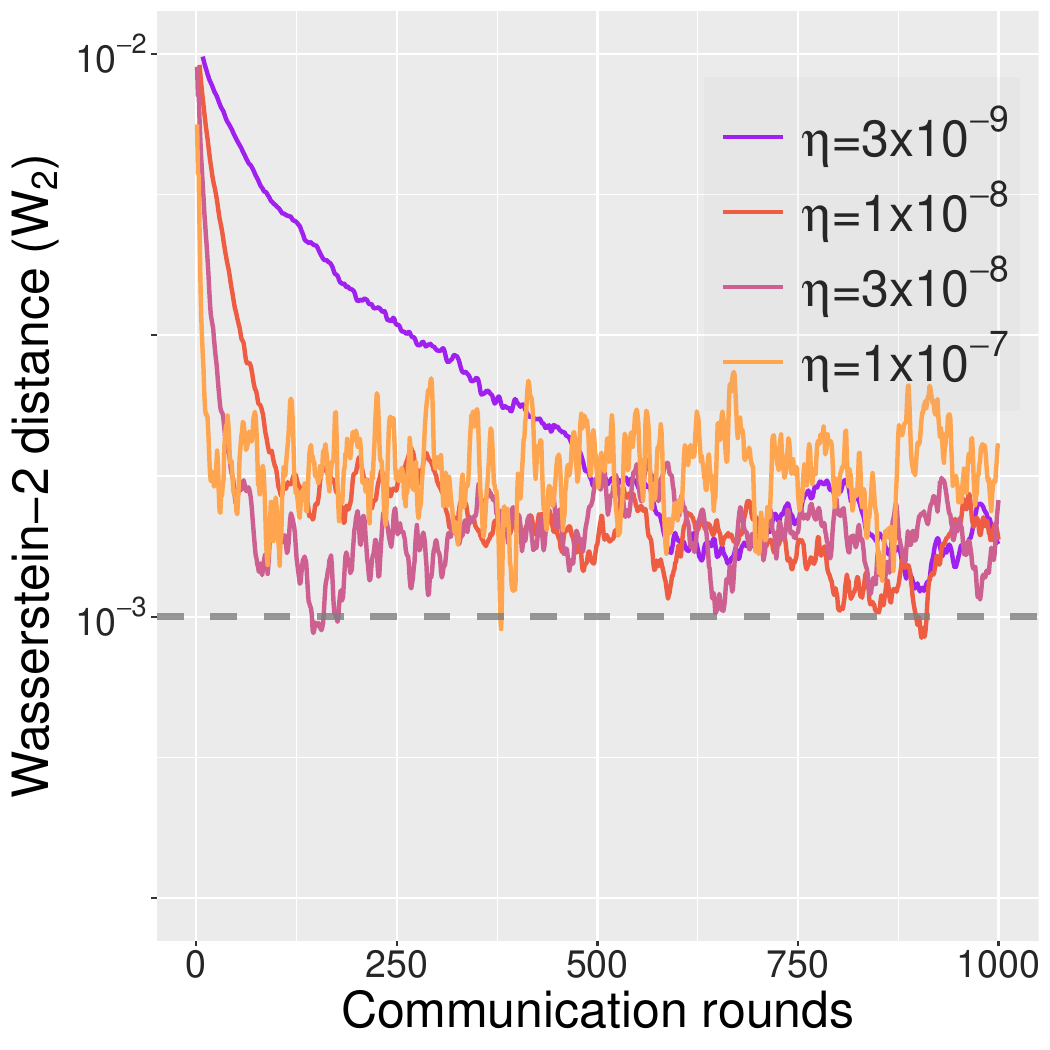}
    \end{minipage}%
    }%
    \subfigure[\scriptsize{Scheme II: $S=40$}]{
    \begin{minipage}[t]{0.19\linewidth}
    \centering
    \label{partial_device_s2_40}
    \includegraphics[width=1.1in]{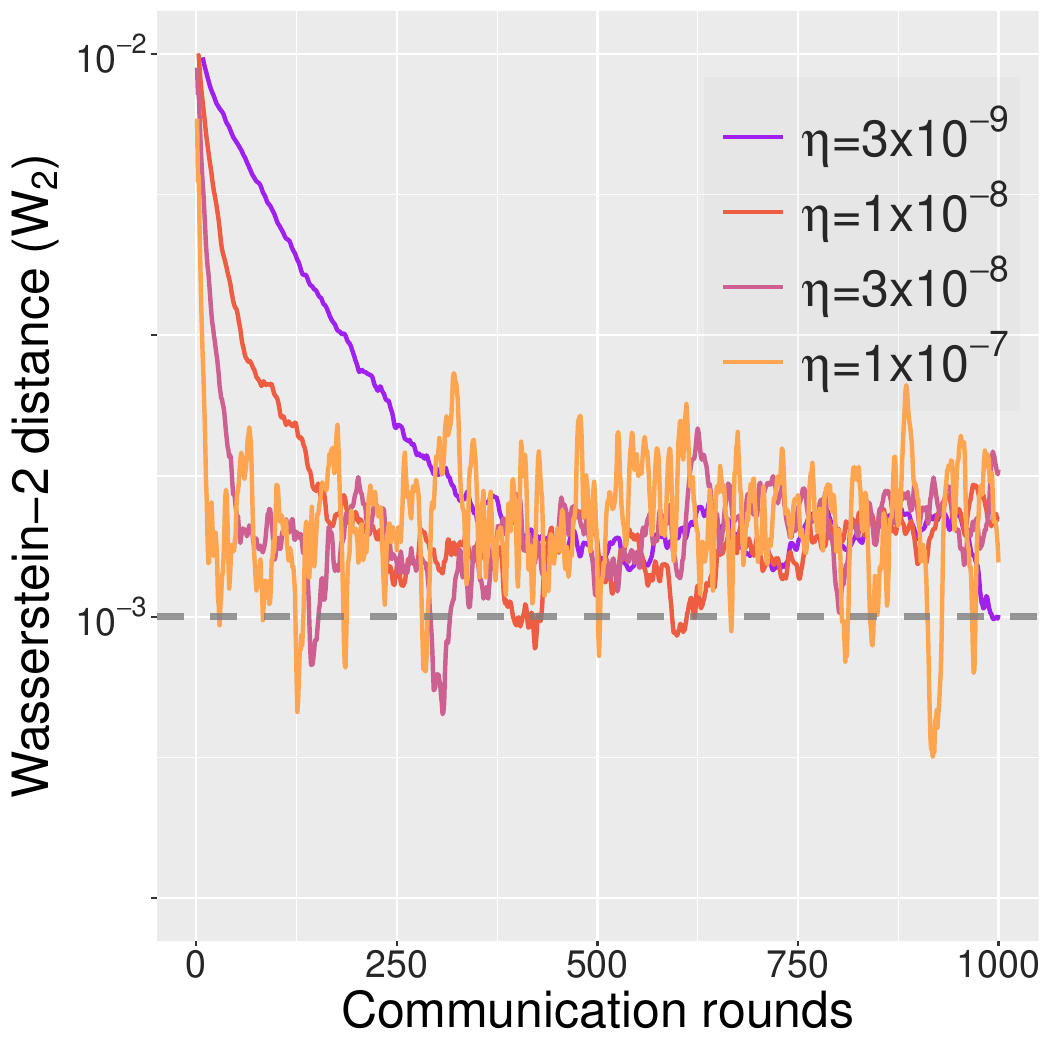}
    \end{minipage}%
    }%
    \subfigure[\scriptsize{Scheme I: $S=30$}]{
    \begin{minipage}[t]{0.19\linewidth}
    \centering
    \label{partial_device_s1_30}
    \includegraphics[width=1.1in]{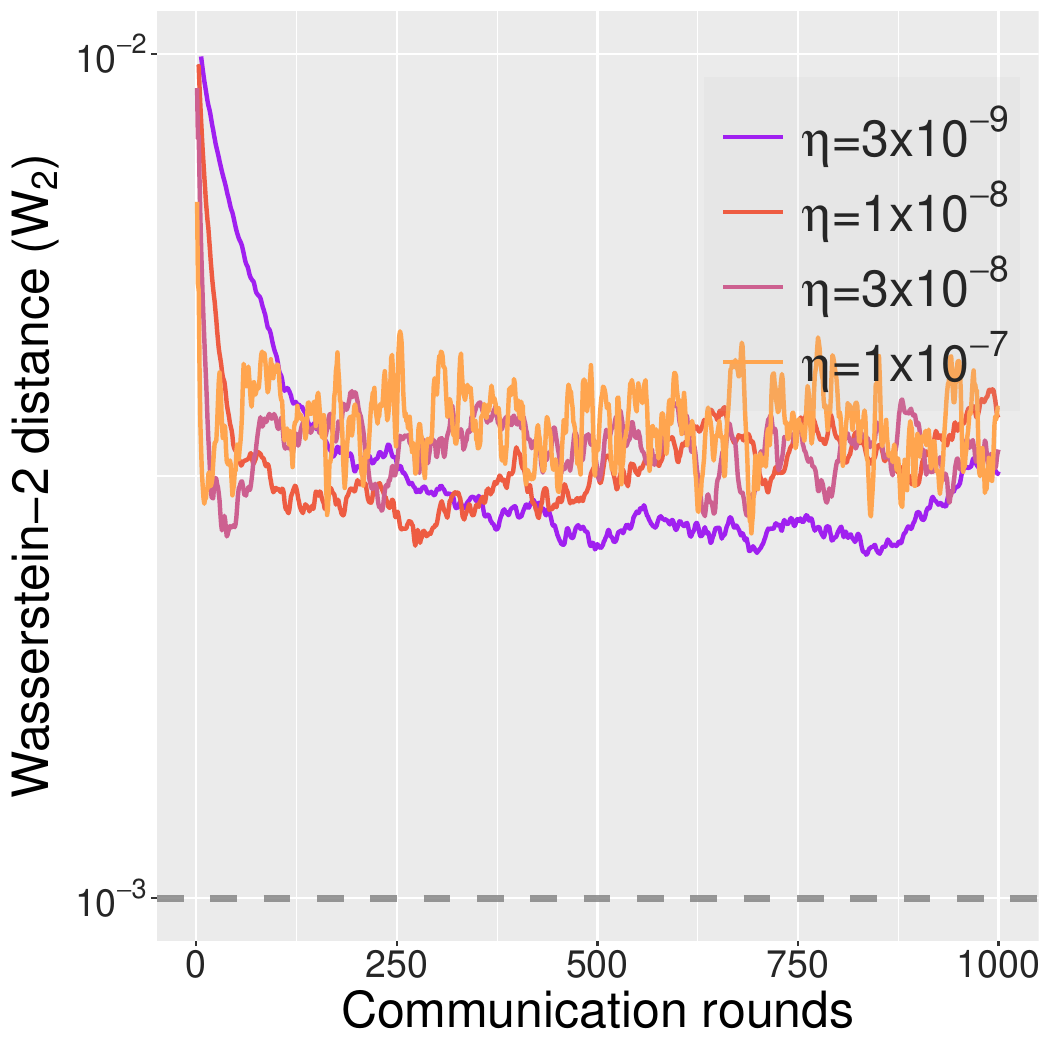}
    \end{minipage}%
    }%
    \subfigure[\scriptsize{Scheme II: $S=30$}]{
    \begin{minipage}[t]{0.19\linewidth}
    \centering
    \label{partial_device_s2_30}
    \includegraphics[width=1.1in]{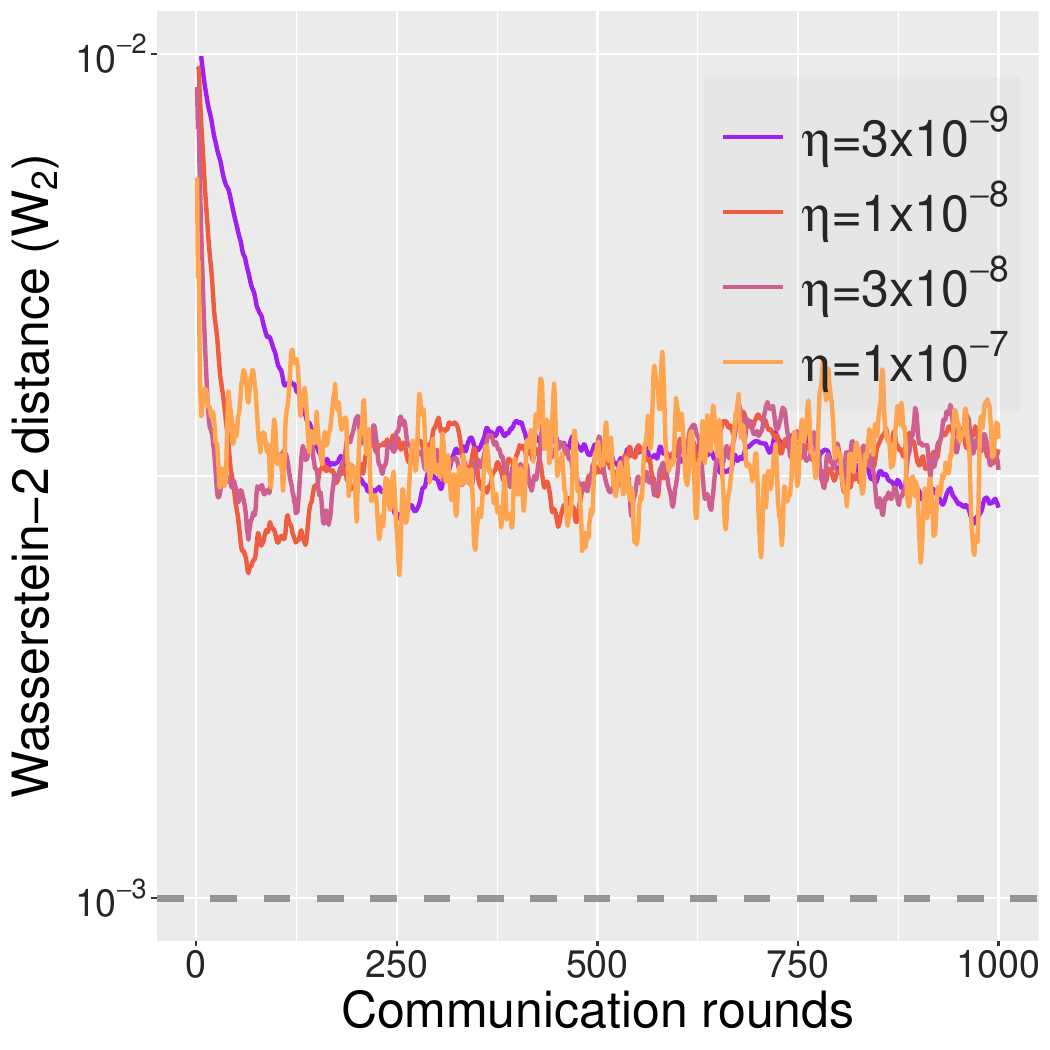}
    \end{minipage}%
    }%
    \vskip -0.0in
  \caption {Convergence of FA-LD based on different device-sampling schemes. }
  \label{figure:partial_device}
    \vskip 0.05in
\end{figure*}

\emph{Partial device participation}: We study the convergence of two popular device-sampling schemes I and II. We fix the number of local steps $K=100$ and the total devices $N=50$. We try to sample $S$ devices based on different fixed learning rates $\eta$. The full device updates are also presented for a fair evaluation. As shown in Figure \ref{full_device_baseline}, larger learning rates converge faster but lead to larger biases; small learning rates, by contrast, yield diminishing biases consistently, where is in accordance with Theorem \ref{main_paper_theorem}. However, in partial device scenarios, the bias becomes much less dependent on the learning rate in the long run. We observe in Figure \ref{partial_device_s1_40}, Figure \ref{partial_device_s2_40}, Figure \ref{partial_device_s1_30}, and Figure \ref{partial_device_s2_30} that the bias caused by partial devices becomes dominant as we decrease the number of partial devices $S$ for both schemes. Unfortunately, such a phenomenon still exists even when the algorithms converge, which suggests that the proposed partial device updates may be only appropriate for the early period of the training or simulation tasks with low accuracy demand.

\subsection{(Fashion) MNIST} 


To evaluate the performance of FA-LD under different local steps $K$ on real-world datasets, we conducted experiments using the MNIST and Fashion-MNIST datasets. We applied FA-LD to train a logistic regression model with the cross entropy loss. To ensure fairness in our evaluation, we randomly split the training dataset into $N=10$ subsets of equal size, creating 10 clients. In each experimental setting, we collected one parameter sample after every 10 communication rounds. We then averaged the predicted probabilities made by all previously collected parameter samples. This allowed us to calculate three test statistics: accuracy, Brier Score (BS) \citep{brier1950verification}, and Expected Calibration Error (ECE) \citep{guo2017calibration} on the test dataset. We tune the step sizes $\eta$ for the best performance and plot the curves of those test statistics against communication rounds under different local steps $K=1, 10, 20, 50, 100$ in Figure \ref{figure:MNIST}. 

We set the temperature parameter $\tau$ to a value of 0.05. During the training process, we calculated the stochastic gradient of the energy function at each step using a batch size of 200 for each client. To facilitate a clear observation of the convergence behavior under different local step values $K$, we introduced a warmup period of 500 communication rounds for each experiment. 


Based on the findings depicted in Figure \ref{figure:MNIST}, it is evident that under the \emph{same communication budget}, FA-LD with $K=1$ (equivalent to the standard SGLD algorithm) performs the poorest in terms of all three test statistics. This outcome highlights the importance of incorporating multiple local updates in federated learning settings. Furthermore, it is worth noting that the optimal local step value $K$ may differ across different test statistics. For instance, in the case of the MNIST dataset, the optimal $K$ range for achieving the highest accuracy lies between 50 and 100 (refer to Figure \ref{fig:M-accu}). On the other hand, the optimal $K$ value for minimizing the Brier Score (BS) is approximately 20 (see Figure \ref{fig:M-brier}). These results indicate that the choice of $K$ should be carefully considered and tailored to the specific evaluation metric of interest.

\begin{figure*}[htbp]
\vskip -0.0in
    \centering
    \subfigure[Accuracy (M)]{
    \begin{minipage}[t]{0.16\linewidth}
    \centering
    \label{fig:M-accu}
    \includegraphics[width=\linewidth]{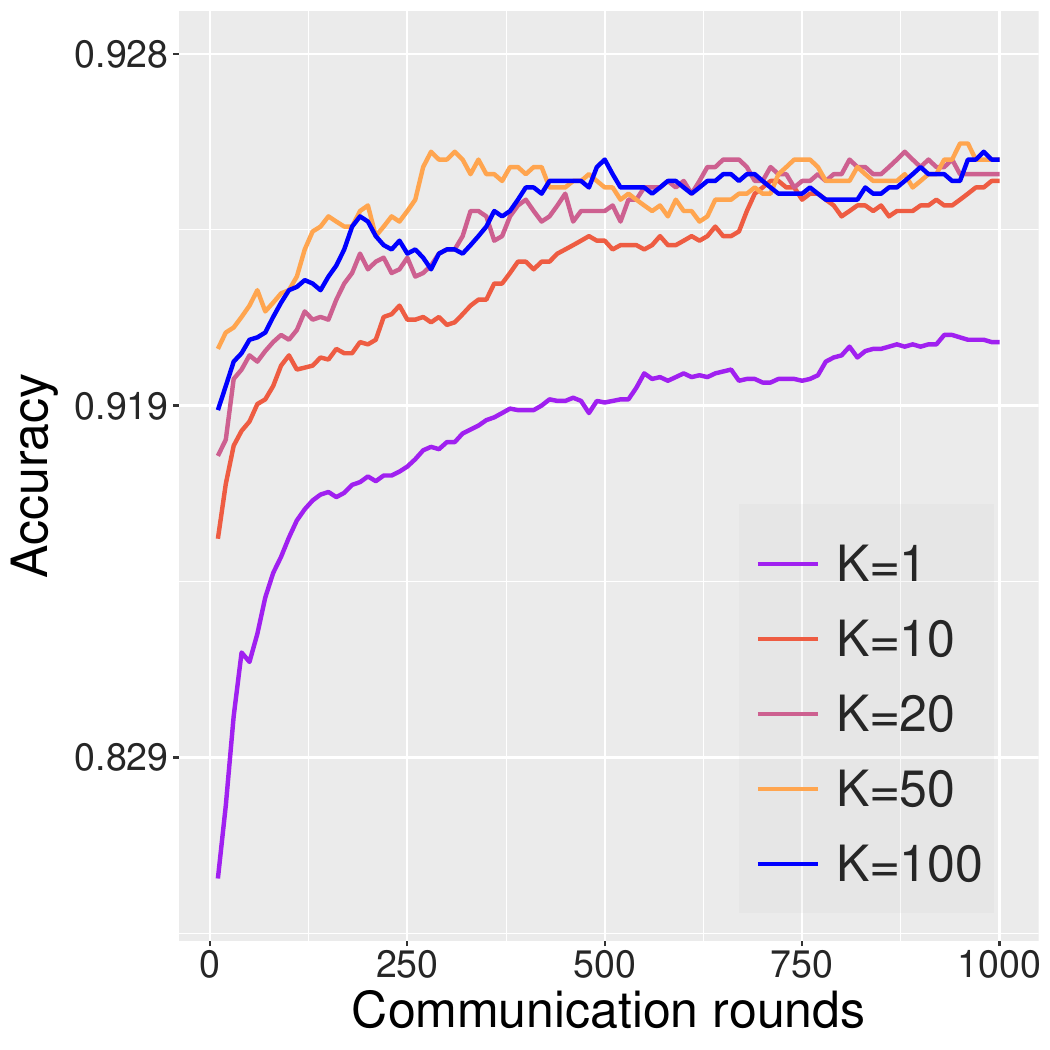}
    \end{minipage}%
    }%
    \subfigure[BS (M)]{
    \begin{minipage}[t]{0.16\linewidth}
    \centering
    \label{fig:M-brier}
    \includegraphics[width=\linewidth]{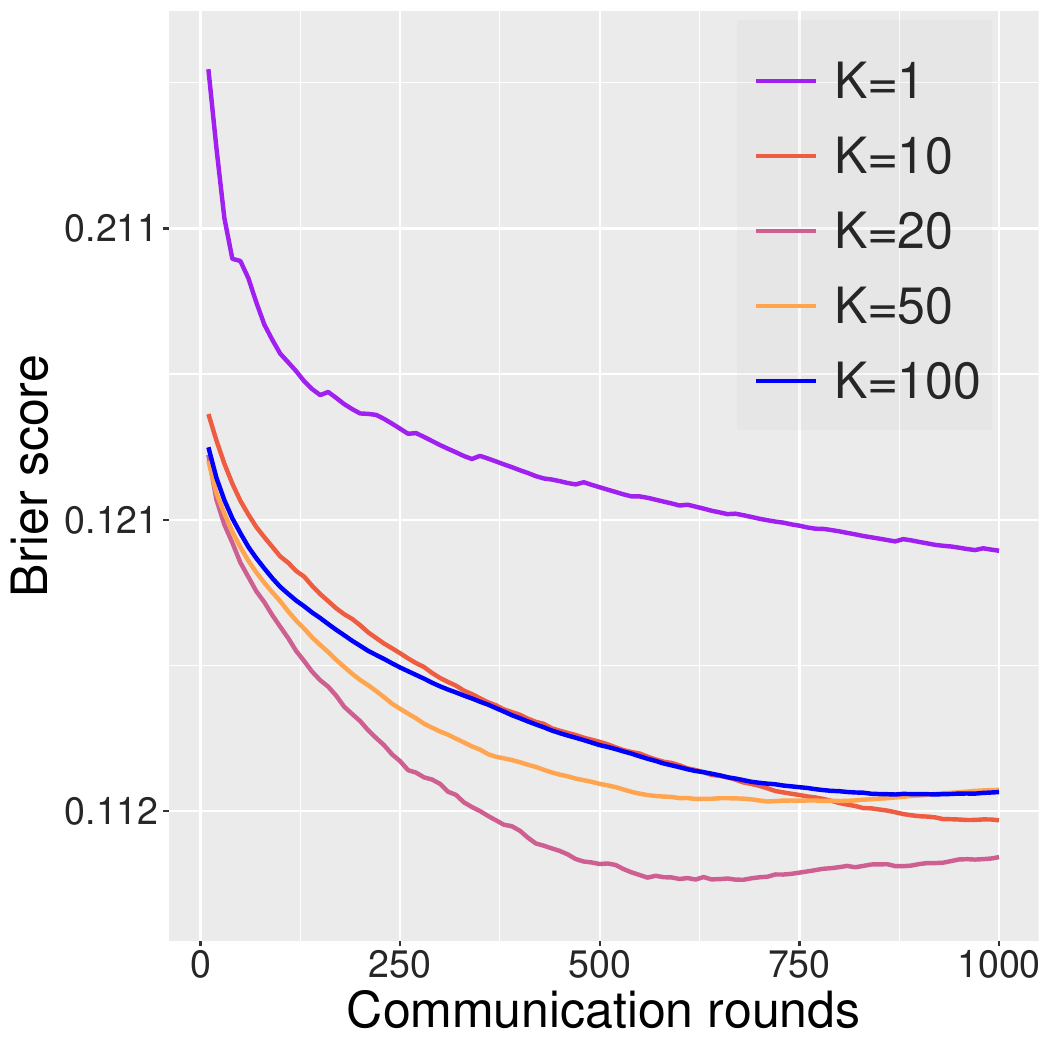}
    \end{minipage}%
    }%
    \subfigure[ECE (M)]{
    \begin{minipage}[t]{0.16\linewidth}
    \centering
    \label{fig:M-ECE}
    \includegraphics[width=\linewidth]{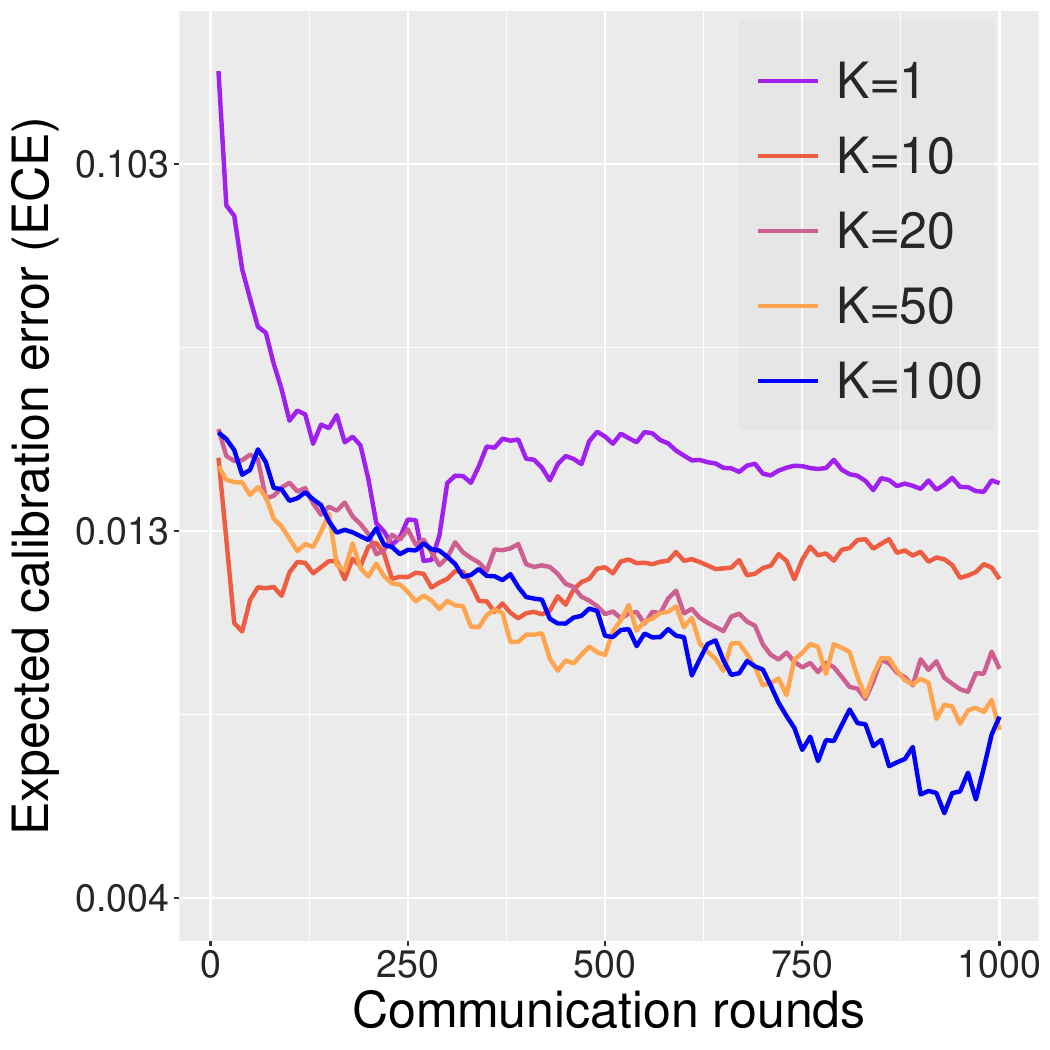}
    \end{minipage}%
    }%
    \subfigure[Accuracy (F)]{
    \begin{minipage}[t]{0.16\linewidth}
    \centering
    \label{fig:FM-accu}
    \includegraphics[width=\linewidth]{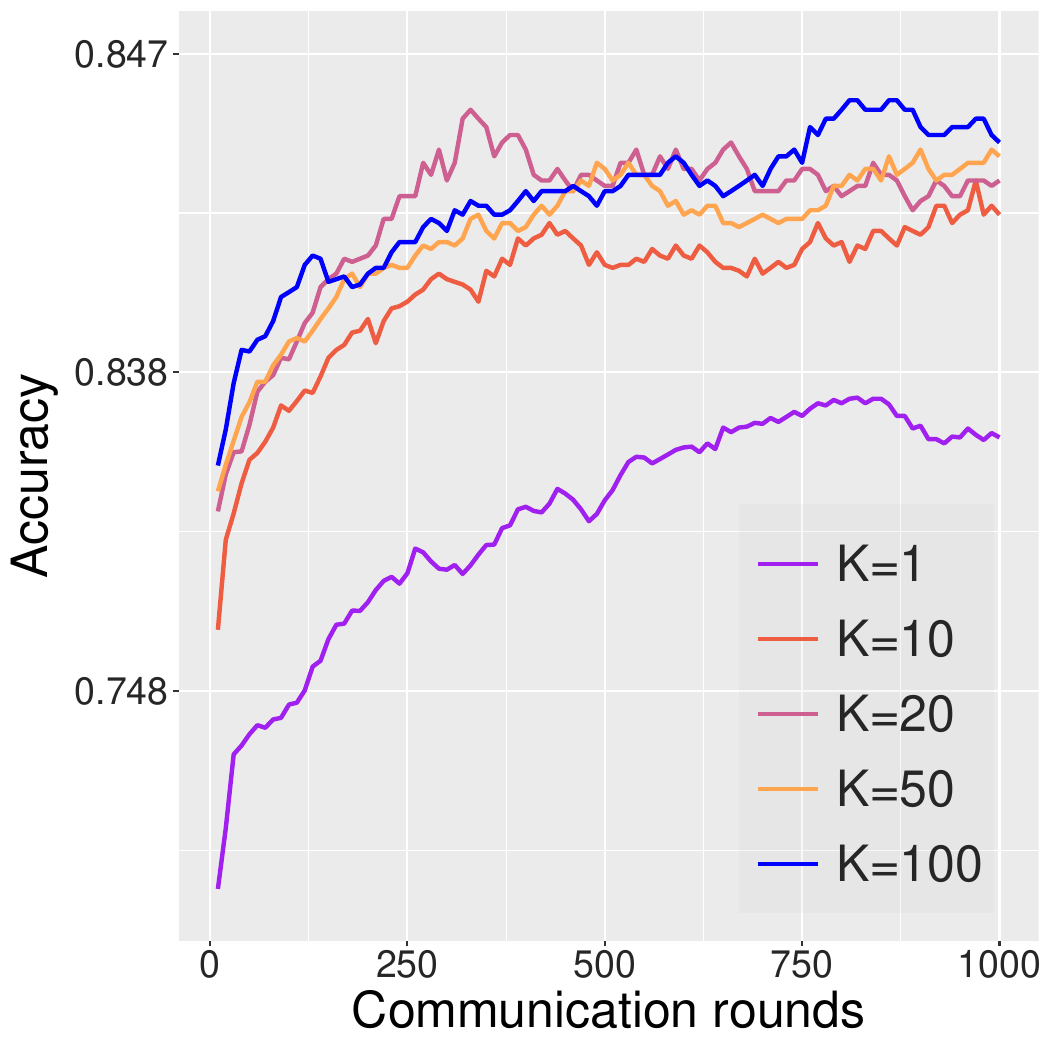}
    \end{minipage}%
    }%
    \subfigure[BS (F)]{
    \begin{minipage}[t]{0.16\linewidth}
    \centering
    \label{fig:FM-brier}
    \includegraphics[width=\linewidth]{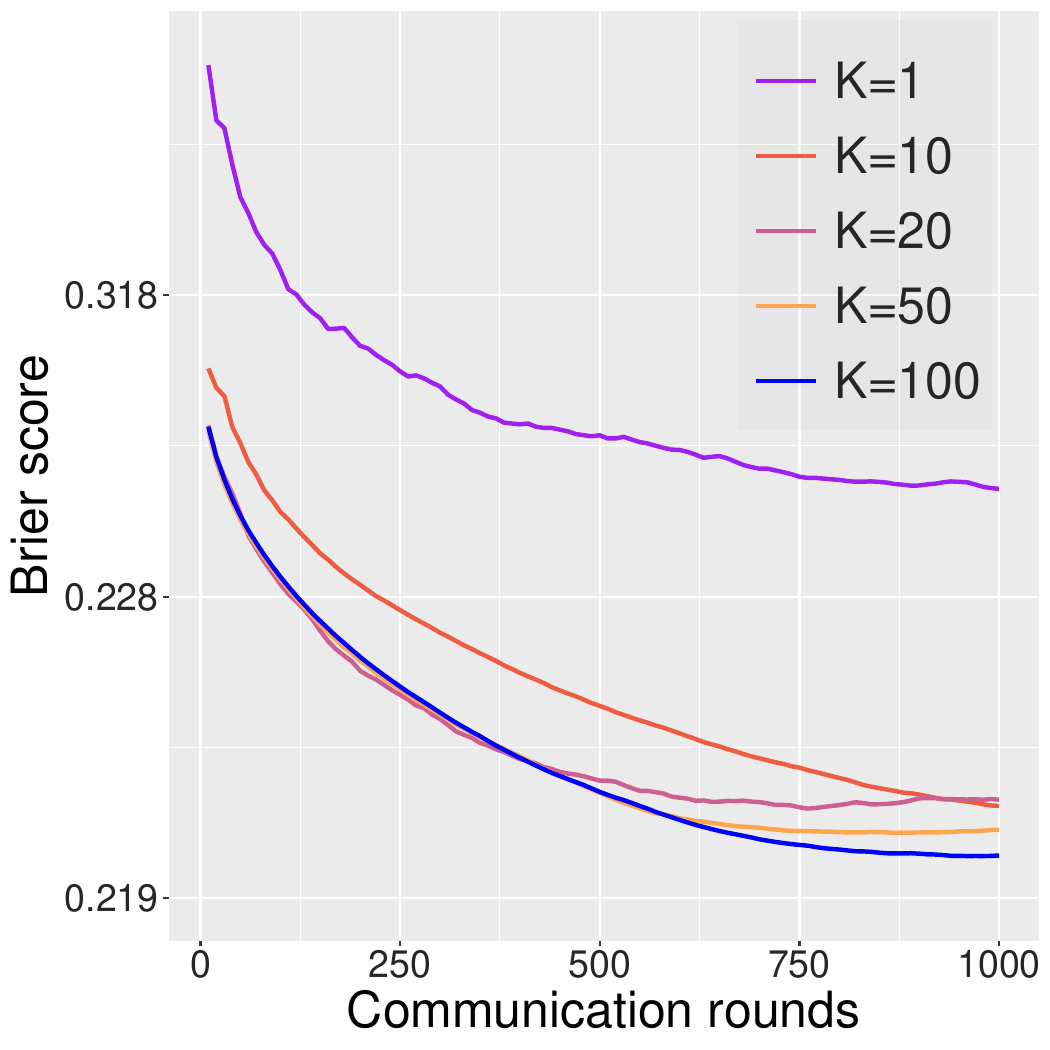}
    \end{minipage}%
    }%
    \subfigure[ECE (F)]{
    \begin{minipage}[t]{0.16\linewidth}
    \centering
    \label{fig:FM-ECE}
    \includegraphics[width=\linewidth]{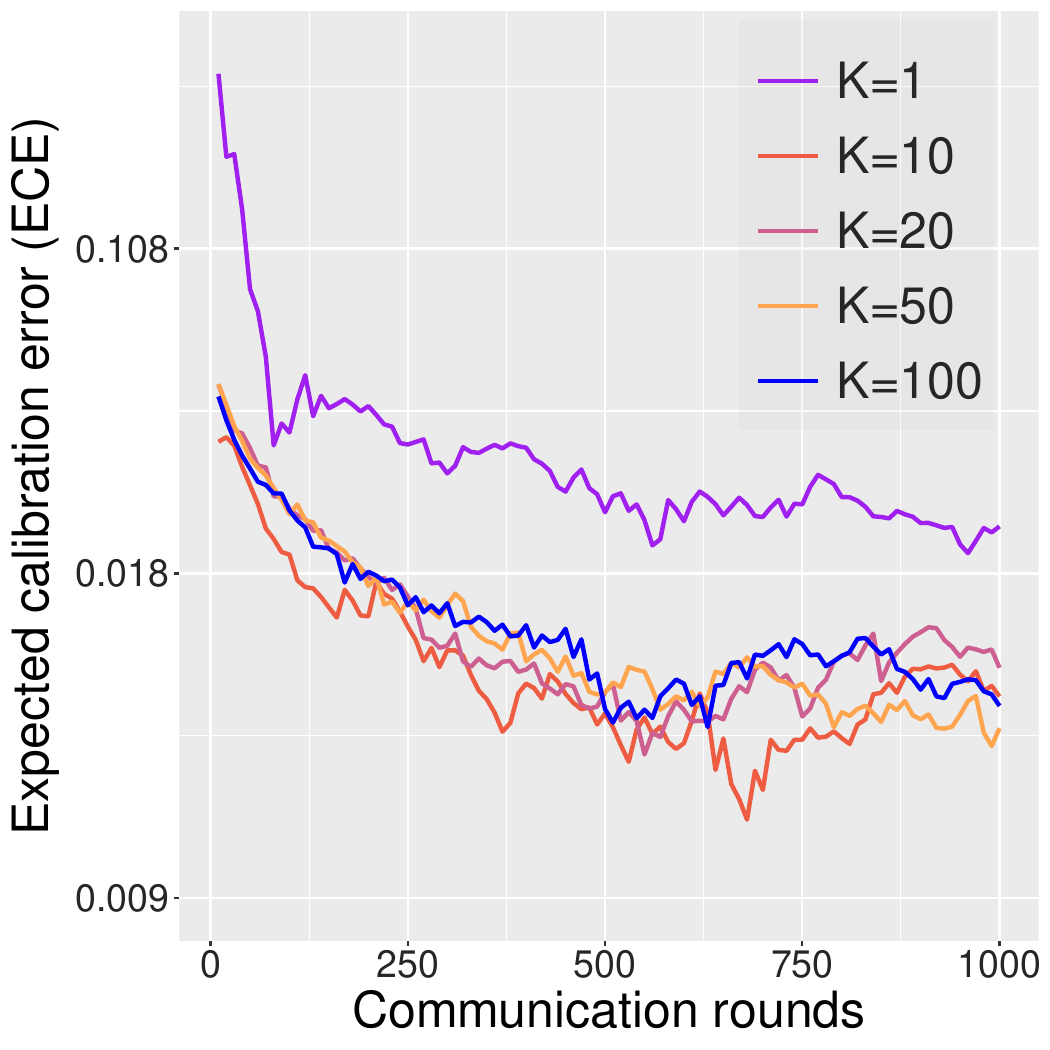}
    \end{minipage}%
    }%
  \vskip -0.0in
  \caption{Convergence of FA-LD on the MNIST (M) and Fashion-MNIST (F) dataset.}
  \label{figure:MNIST}
  \vskip -0.0in
\end{figure*}

Based on the observations from Figure \ref{figure:MNIST_half}, we can conclude that for the MNIST dataset, FA-LD with $K=1$ performs the worst across all three test statistics. Similarly, for the Fashion-MNIST dataset, FA-LD with $K=1$ exhibits the poorest performance in terms of accuracy and BS, and it does not yield the best results in terms of ECE. These findings underscore the benefits of incorporating multiple local updates in federated learning, particularly when operating under a fixed communication budget. Among the local step values $K=1, 10, 20, 50, 100$, our analysis, as depicted in Figure \ref{fig:M-half-accu}, \ref{fig:M-half-brier}, and \ref{fig:M-half-ECE}, reveals that for the MNIST dataset, the optimal choice for the local step $K$ is 20, considering accuracy, BS, and ECE. In contrast, for the Fashion-MNIST dataset, the optimal $K$ value is 20 for accuracy, 10 for BS, and 50 for ECE, as evident from Figure \ref{fig:FM-half-accu}, \ref{fig:FM-half-brier}, and \ref{fig:FM-half-ECE}, respectively. These results emphasize the importance of selecting an appropriate $K$ value that aligns with the specific evaluation metric of interest in order to achieve optimal performance.

It is important to note that the optimal choice of the local step value $K$ can differ when considering the presence or absence of a warmup period, as illustrated in Figure \ref{figure:MNIST}. This observation holds significant implications for selecting the appropriate $K$ value in federated learning scenarios. For instance, let's consider the Fashion-MNIST dataset. Without a warmup period, the optimal $K$ value for minimizing the Brier Score (BS) is 100, as indicated in Figure \ref{fig:FM-brier}. However, when a warmup period consisting of the first 500 communication rounds is introduced, the optimal $K$ value for BS shifts to 10, as depicted in Figure \ref{fig:FM-half-brier}. This discrepancy suggests that the communication budget also influences the determination of the optimal local step $K$. 
This observation underscores the importance of considering the impact of the communication budget and warmup periods when determining the optimal local step value. It is crucial to assess the interplay between these factors to make informed decisions and achieve the best possible performance.

\begin{figure*}[htbp]
    \vskip -0.0in
    \centering
    \subfigure[Accuracy (M)]{
    \begin{minipage}[t]{0.16\linewidth}
    \centering
    \label{fig:M-half-accu}
    \includegraphics[width=\linewidth]{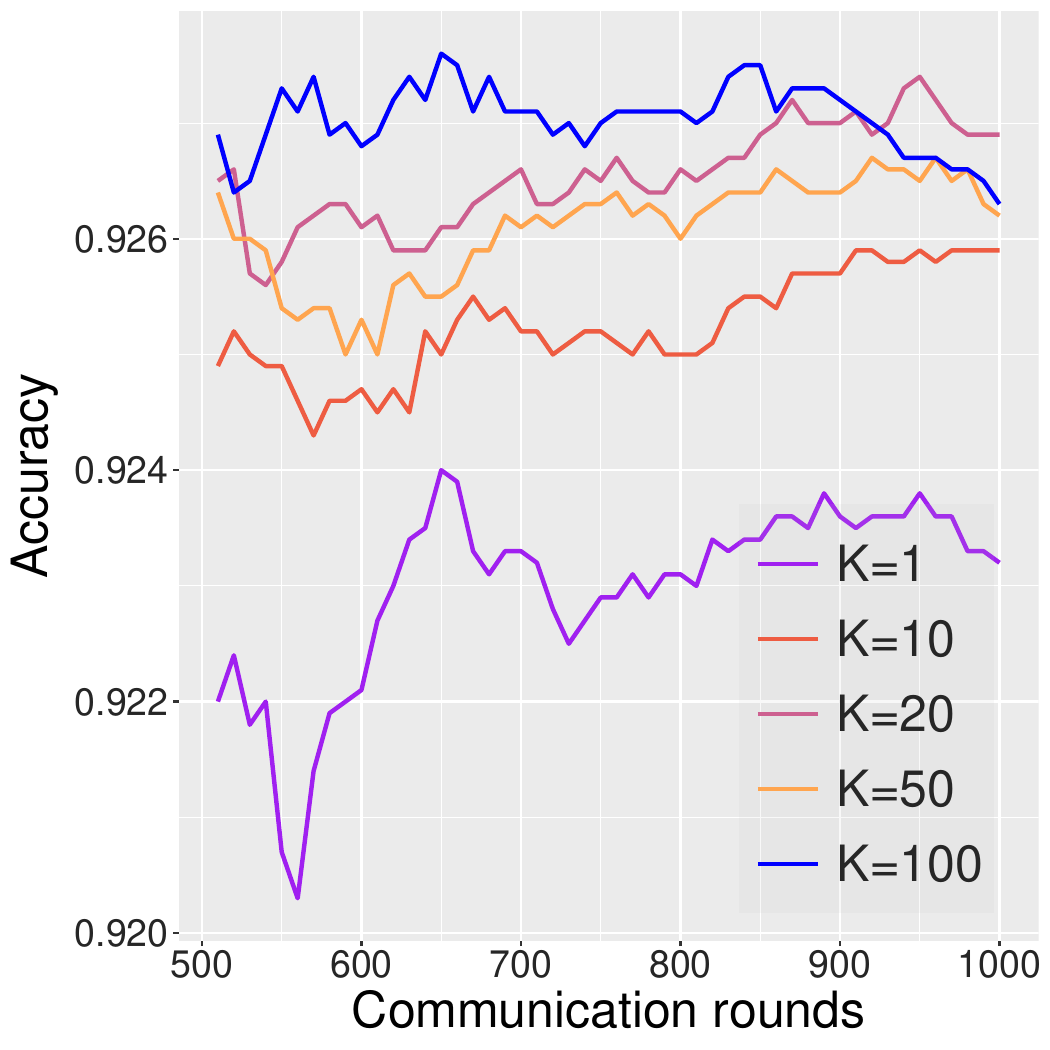}
    \end{minipage}%
    }%
    \subfigure[BS (M)]{
    \begin{minipage}[t]{0.16\linewidth}
    \centering
    \label{fig:M-half-brier}
    \includegraphics[width=\linewidth]{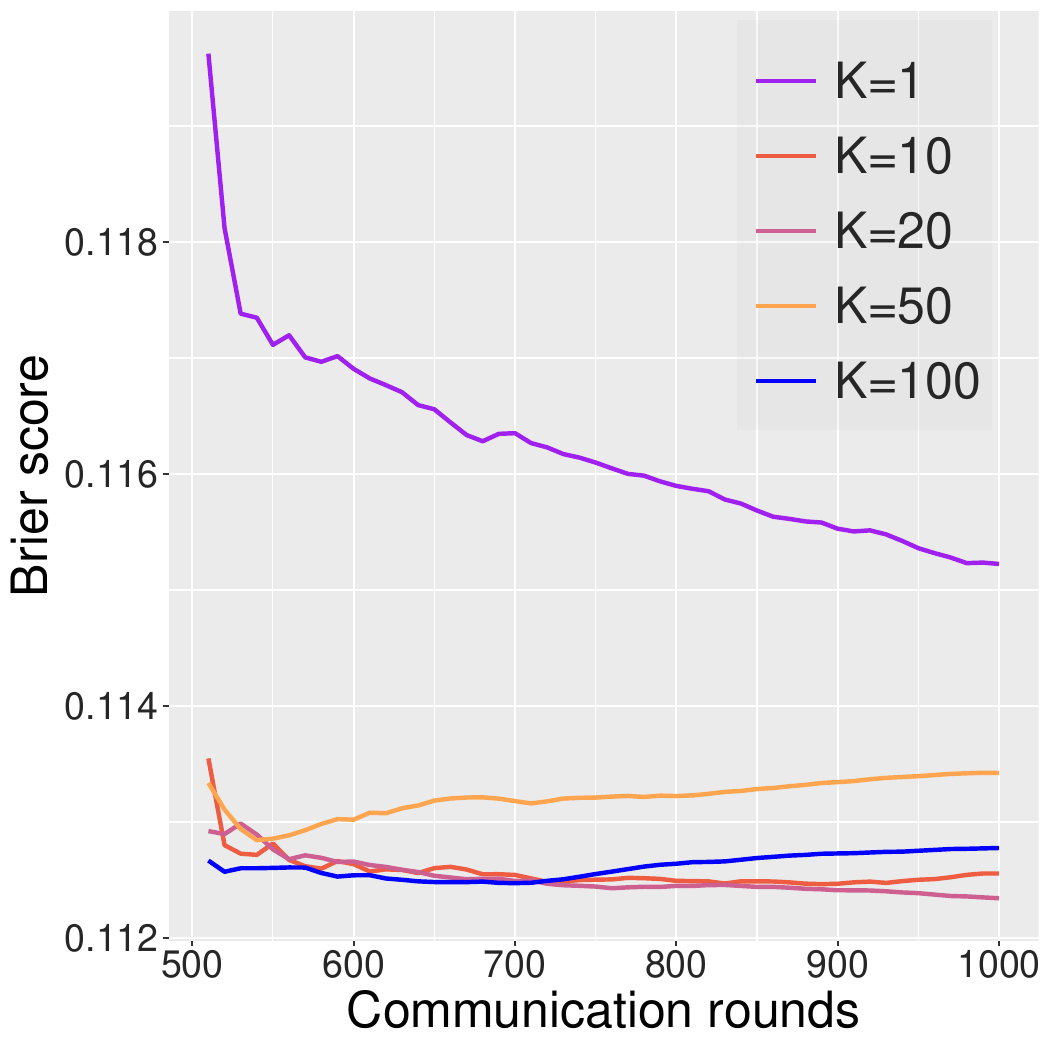}
    \end{minipage}%
    }%
    \subfigure[ECE (M)]{
    \begin{minipage}[t]{0.16\linewidth}
    \centering
    \label{fig:M-half-ECE}
    \includegraphics[width=\linewidth]{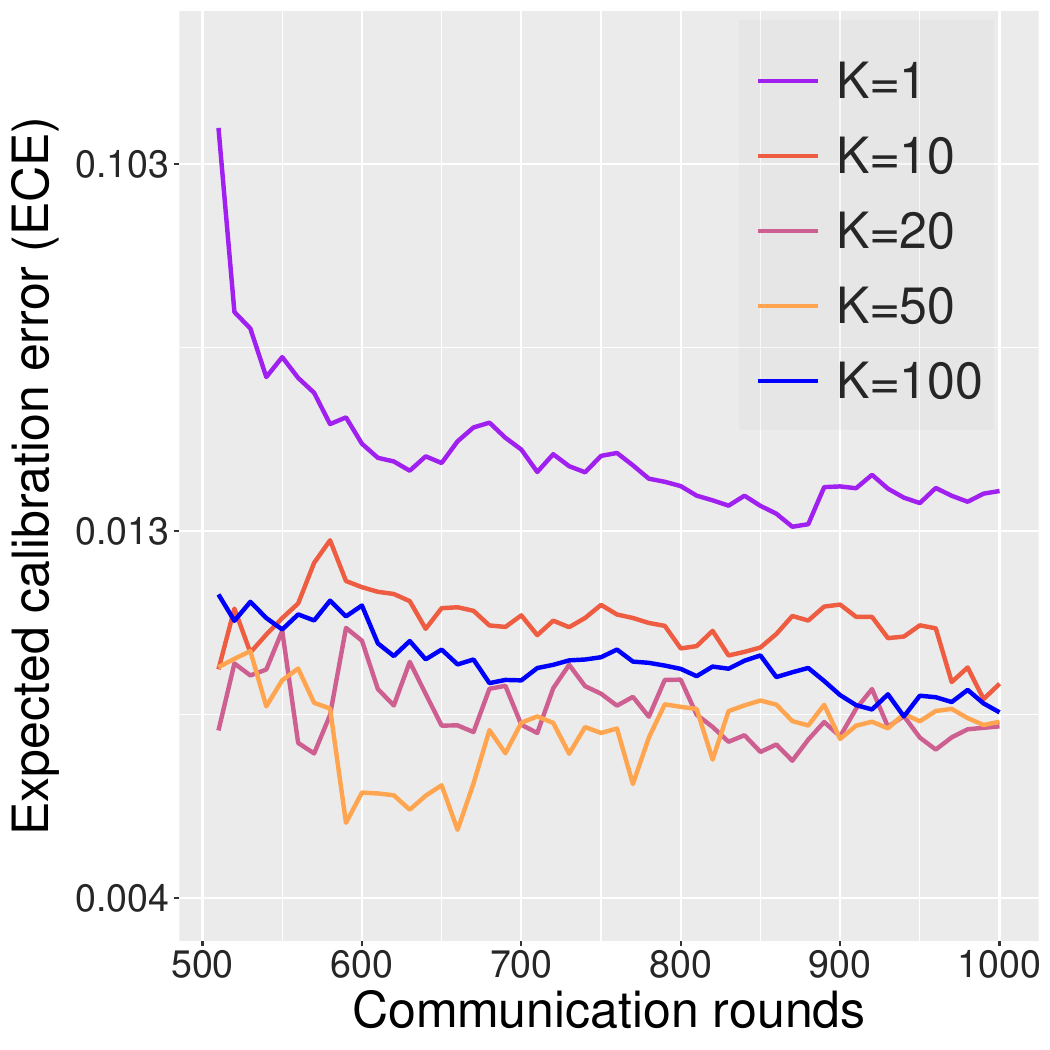}
    \end{minipage}%
    }%
    \subfigure[Accuracy (F)]{
    \begin{minipage}[t]{0.16\linewidth}
    \centering
    \label{fig:FM-half-accu}
    \includegraphics[width=\linewidth]{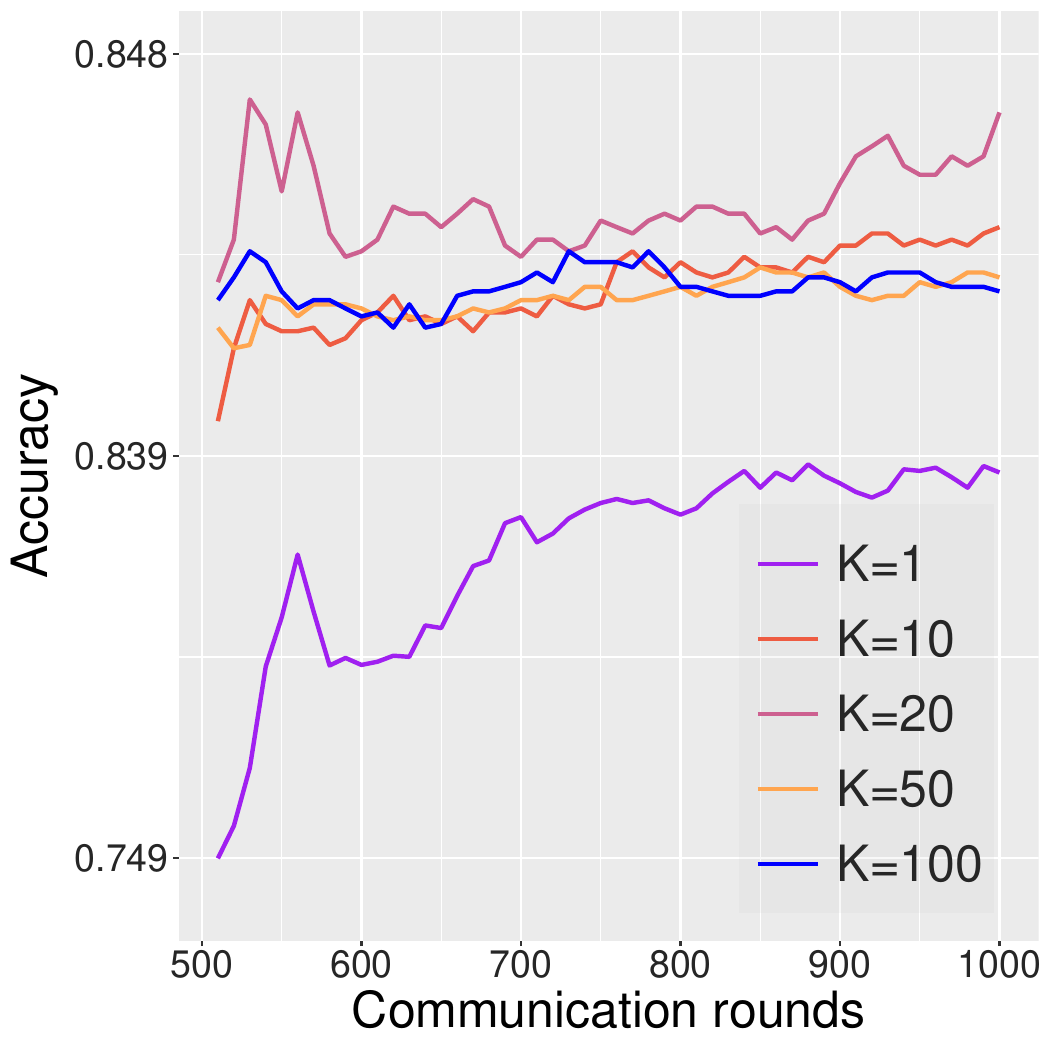}
    \end{minipage}%
    }%
    \subfigure[BS (F)]{
    \begin{minipage}[t]{0.16\linewidth}
    \centering
    \label{fig:FM-half-brier}
    \includegraphics[width=\linewidth]{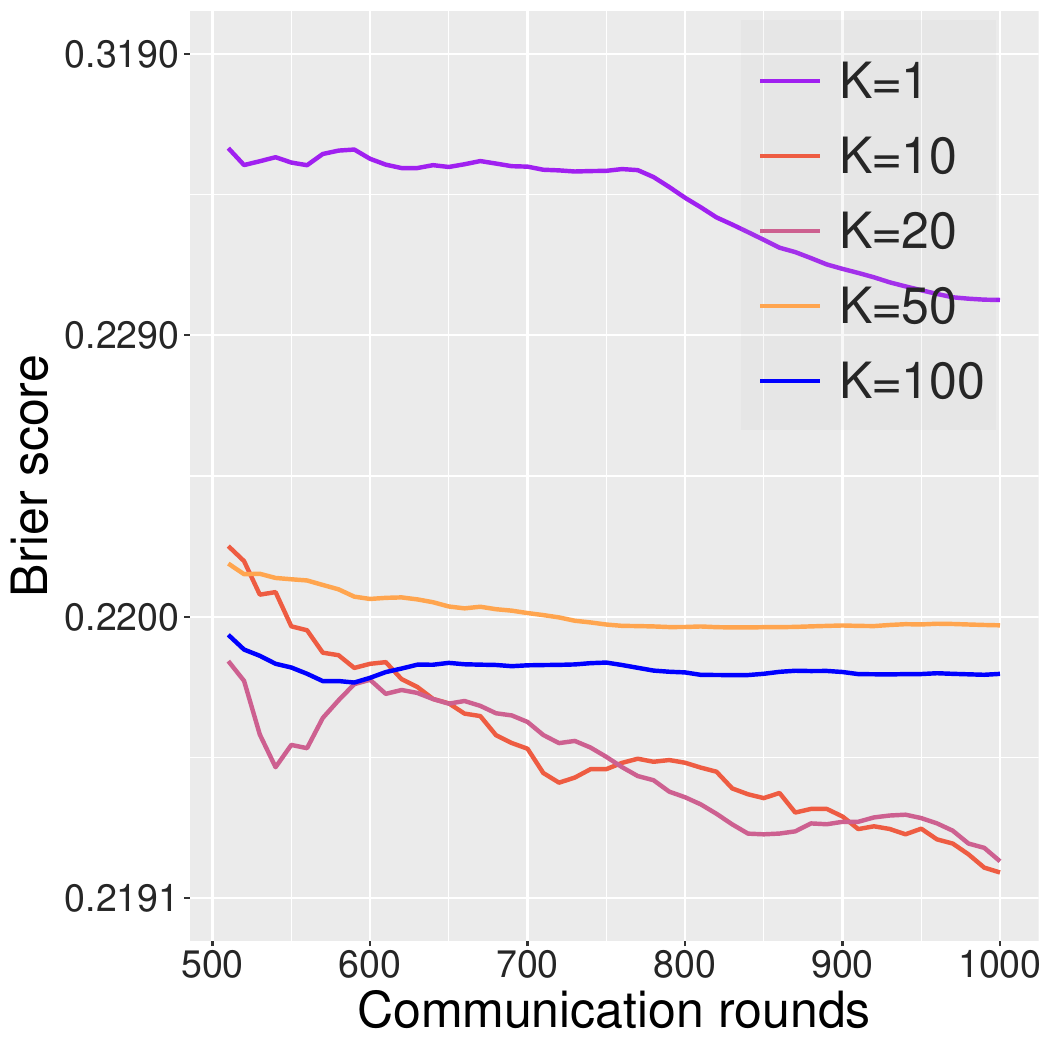}
    \end{minipage}%
    }%
    \subfigure[ECE (F)]{
    \begin{minipage}[t]{0.16\linewidth}
    \centering
    \label{fig:FM-half-ECE}
    \includegraphics[width=\linewidth]{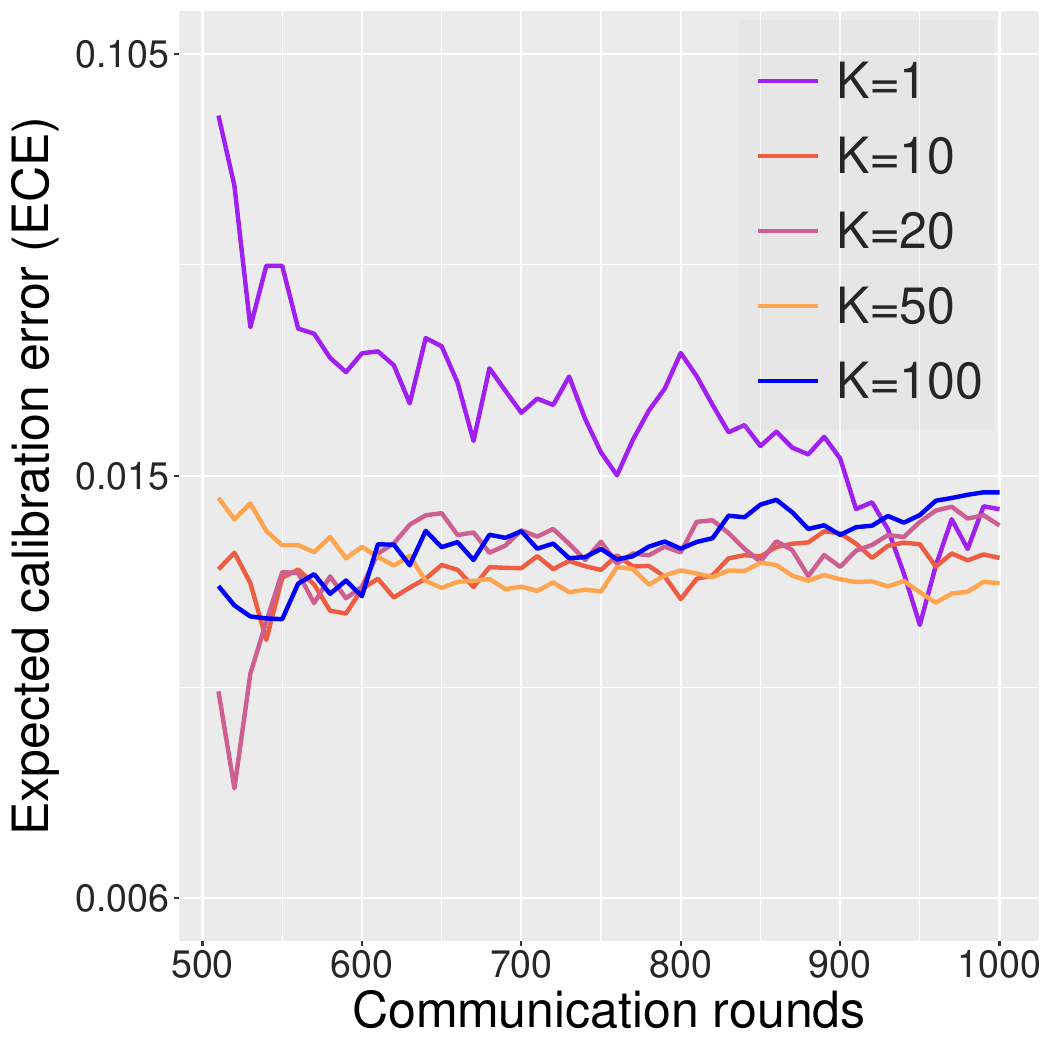}
    \end{minipage}%
    }%
    \vskip -0.0in
  \caption{Convergence of FA-LD on the MNIST (M) and Fashion-MNIST (M) dataset with a warmup period of the first 500 communication rounds.}
  \label{figure:MNIST_half}
    \vskip -0.0in
\end{figure*}

\section{Conclusion}\label{sec:concl}


We propose a first convergence analysis for federated averaging Langevin dynamics (FA-LD) with distributed clients. 
The theoretical guarantees yield a concrete guidance on the selection of the optimal number of local updates to minimize communication costs. In addition, the convergence highly depends on the data heterogeneity and the injected noises, where the latter also inspires us to consider correlated injected noise and partial device updates to balance between differential privacy and prediction accuracy with theoretical guarantees.

\bibliography{ref}
\bibliographystyle{tmlr}

\appendix

\newpage

\begin{Large}
\begin{center}
    \textbf{Supplimentary Material for \textit{``On Convergence of Federated Averaging Langevin Dynamics''}}
\end{center}
\end{Large}

\paragraph{Roadmap.}
In Section~\ref{sec:preli}, we layout the formulation of the algorithm, basic notations, and definitions. In Section~\ref{sec:full_device_participation}, we present the main convergence analysis for full device participation. We discuss the optimal number of local updates based on a fixed learning rate, the acceleration achieved by varying learning rates, and the privacy-accuracy trade-off through correlated noises. In Section~\ref{sec:partial_device_participation}, we analyze the convergence of partial device participation through two device-sampling schemes. In Section~\ref{sec:bouding_contraction_discretization_divergence}, we provide lemmas to upper bound the contraction, discretization and divergence for proving the main convergence results. In Section~\ref{sec:uniform_upper_bound}, we include supporting lemmas to prove results in the previous section. In Section~\ref{sec:initial_condition}, we establish the initial condition. In Section~\ref{dp_guarantee}, we prove differential privacy guarantees.

\section{Preliminaries}\label{sec:preli}

\subsection{Basic notations and backgrounds} 

Let $N$ denote the number of clients. Let $T_{\epsilon}$ denote the number of global steps to achieve the precision $\epsilon$. Let $K$ denote the number of local steps. For each $c \in [N]:=\{1,2,\cdots, N\}$, we use $f^c$ and $\nabla f^c$ denote the loss function and gradient of the function $f^c$ in client $c$. Notably, $\nabla f$ is not a standard gradient operator acting on $ f$ when multiple local steps are adopted ($K>1$). For the stochastic gradient oracle, we denote by $\nabla \tilde f^c(\cdot)$ the \emph{unbiased} estimate of the exact gradient $\nabla f^c$ of client $c$. In addition, we denote $p_c$ as the weight of the $c$-th client such that $p_c\geq 0$ and $\sum_{c=1}^N p_c=1$. $\xi_k^c$ is an independent standard $d$-dimensional Gaussian vector at iteration $k$ for each client $c\in[N]$ and $\dot\xi_k$ is a unique Gaussian vector shared by all the clients.

\begin{algorithm*}[h]\caption{Federated averaging Langevin dynamics algorithm (FA-LD). Denote by $\theta_k^c$ the model parameter in the $c$-th client at the $k$-th step. Denote the immediate result of one step SGLD update from $\theta_k^c$ by $\beta_k^c$. $\xi_k^c$ is an independent standard $d$-dimensional Gaussian vector at iteration $k$ for each client $c\in[N]$. A global synchronization is conducted every $K$ steps. This is a transformed version of Algorithm~\ref{alg:alg_main_text_partial_main} with $\rho=0$ and full device participation for ease of analysis.}\label{alg:alg_main_text_independent_noise}
\begin{equation}\label{local_client}
    \beta_{k+1}^c=\theta_k^c-\eta\nabla \tilde f^c(\theta_k^c)+\sqrt{2\eta\tau/p_c}\xi_k^c,
\end{equation}
\begin{equation}  
\label{synchronization}
\theta_{k+1}^c=\left\{  
             \begin{array}{lr}  
             \beta_{k+1}^c \qquad\qquad\qquad \text{if } k+1 \text{ mod } K\neq 0 \\  
              & \\
             \sum_{c=1}^N p_c \beta_{k+1}^c \ \qquad \text{if } k+1 \text{ mod } K=0.
             \end{array}  
\right.  
\end{equation} 
\end{algorithm*}

Inspired by \citet{lhy+20}, we define two virtual sequences 
\begin{equation}
\label{virtual_seq}
\beta_k=\sum_{c=1}^N p_c \beta_k^c, \qquad \theta_k=\sum_{c=1}^N p_c \theta_k^c,
\end{equation}
which are \emph{both inaccessible when $k \text{ mod } K\neq 0$}. For the gradients and injected noise, we also define 
\begin{equation}
\label{sum_grad}
\nabla f(\theta_k)=\sum_{c=1}^N p_c \nabla f^c(\theta_k^c), \quad {\bm{\widetilde Z}_k}=\sum_{c=1}^N p_c \nabla \tilde f^c(\theta_k^c),\quad \btheta_k=(\theta_k^1,\cdots, \theta_k^N), \quad \xi_k=\sum_{c=1}^N \sqrt{p_c} \xi_k^c.
\end{equation}

In what follows, it is clear that $\E{{\bm{\widetilde Z}_k}}=\sum_{c=1}^N p_c \E{\nabla \tilde f^c(\theta^c)}={\sum_{c=1}^N p_c {\nabla  f^c(\theta^c)}}:={\bm{Z}_k}$ for any $\theta^c\in\R^d$ and any $c\in[N]$. Summing Eq.(\eqref{local_client}) from clients $c=1$ to $N$ and combining Eq.(\eqref{virtual_seq}) and Eq.(\eqref{sum_grad}), we have
\begin{align}
\label{fed_avg_langevin_dynamics_preliminary}
    \beta_{k+1}&=\theta_k-\eta { \bm{\widetilde Z}_k}+\sqrt{2\eta\tau}\xi_k.
\end{align}
Moreover, we always have $\beta_k=\theta_k$ whether $k+1 \text{ mod } K=0$ or not by Eq.(\eqref{synchronization}) and Eq.(\eqref{virtual_seq}). In what follows, we can write
\begin{equation}
\label{fed_avg_langevin_dynamics}
\theta_{k+1}=\theta_k-\eta { \bm{\widetilde Z}_k}+\sqrt{2\eta\tau}\xi_k,
\end{equation}
which resembles the SGLD algorithm \citep{Welling11} except that the construction of stochastic gradients is different and $\theta_k$ is \emph{not accessible when $k\text{ mod } K\neq 0$}. {Our derivation of Eq.(\eqref{fed_avg_langevin_dynamics}) is motivated by \citep{lhy+20}, where a similar federated algorithm based on SGD in developed in section A.1.}

\subsection{Assumptions and definitions}

\begin{assumption}[Smoothness, restatement of Assumption \ref{def:smooth_main}]\label{def:smooth} For each $c\in [N]$, we say $f^c$ is $L$-smooth if for some $L>0$
\begin{align*}
f^c(y)\leq f^c(x)+\langle \nabla f^c(x),y-x \rangle+\frac{L}{2}\| y-x \|^2_2\quad \forall x, y\in \R^d.
\end{align*}
\end{assumption}

Note that the above assumption is equivalent to saying that
\begin{align*}
\| \nabla f^c(y)-\nabla f^c(x) \|_2 \leq L \| y-x \|_2,\quad \forall x, y\in \R^d.
\end{align*}

\begin{assumption}[Strong convexity, restatement of Assumption \ref{def:strong_convex_main}]\label{def:strong_convex}
For each $c\in [N]$, $f^c$ is $m$-strongly convex if for some $m>0$
\begin{align*}
f^c(x)\geq f^c(y)+\langle \nabla f^c(y),x-y \rangle + \frac{m}{2} \| y-x \|_2^2\quad \forall x, y\in \R^d.
\end{align*}
\end{assumption}

An alternative formulation for strong convexity is that
\begin{align*}
\langle \nabla f^c(x)-\nabla  f^c(y), x-y\rangle \geq m \lrn{x-y}_2^2\quad \forall x, y\in \R^d.
\end{align*}

\begin{assumption}[Bounded variance, restatement of Assumption \ref{def:variance_main}]\label{def:variance}
For each $c\in [N]$, the variance of noise in the stochastic gradient $\nabla \tilde f^c(x)$ in each client is upper bounded such that 
\begin{align*}
\mathbb{E}[ \| \nabla \tilde f^c(x) - \nabla f^c(x) \|_2^2] \leq \sigma^2 d,\quad \forall x\in \R^d.
\end{align*}
\end{assumption}

The bounded variance in the stochastic gradient is a rather standard assumption and has been widely used in \citet{ccbj18, dk19, lhy+20}. Extension of bounded variance to unbounded cases such as $\mathbb{E}[ \| \nabla \tilde f^c(x) - \nabla f^c(x) \|_2^2]\leq \delta (L^2 x^2 + B^2)$ for some $M$ and $\delta\in[0,1)$ is quite straightforward and has been adopted in assumption A.4 stated in \citet{Maxim17}. The proof framework remains the same.

\paragraph{Quality of non-i.i.d data} Denote by $\theta_*$ the global minimum of $f$. Next, we quantify the degree of the non-i.i.d data by $\gamma:=\max_{c\in[N]}\lrn{\nabla f^c(\theta_*)}_2$, which is a non-negative constant and yields a smaller scale if the data is more evenly distributed.

\begin{definition}\label{def:H_kappa_gamma}
We define parameter $T_{c, \rho}$ $H_{\rho}$, $\kappa$ and $\gamma^2$
\begin{align*}
    T_{c,\rho}: = & ~ \tau(\rho^2+(1-\rho^2)/p_c),\\
    H_{\rho}: = & ~  \underbrace{ \mathcal{D}^2}_{\small{\mathrm{initialization}}}+\underbrace{ \frac{1}{m}\max_{c\in[N]} T_{c,\rho}}_{\small{\mathrm{injected~noise}}} +\underbrace{\frac{\gamma^2}{m^2d}}_{\small{\mathrm{data~heterogeneity}}}+\underbrace{\frac{\sigma^2}{m^2}}_{\small{\mathrm{stochastic~noise}}},\\
    \kappa := & ~ L / m , \\
    \gamma^2 : = & ~ \max_{c \in [N]} \| \nabla f^c (\theta_*) \|_2^2 .
\end{align*}
\end{definition}

\section{Full device participation}\label{sec:full_device_participation}

\subsection{One-step update}
\paragraph{Wasserstein distance}

We define the 2-Wasserstein distance between a pair of Borel probability measures $\mu$ and $\nu$ on $\R^d$ as follows  
\begin{align*}
    W_2(\mu, \nu):=\inf_{\textcolor{black}{\Gamma}\in \text{Couplings}(\mu, \nu)}\left(\int\|\bbeta_{\mu}-\bbeta_{\nu}\|_2^2 d \textcolor{black}{\Gamma}(\bbeta_{\mu}, \bbeta_{\nu})\right)^{\frac{1}{2}},
\end{align*}
where $\|\cdot\|_2$ denotes the $\ell_2$ norm on $\mathbb{R}^d$ and the pair of random variables $(\bbeta_{\mu}, \bbeta_{\nu})\in \R^d\times\R^d$ is a coupling with the marginals following $\mathcal{L}(\bbeta_{\mu})=\mu$ and $\mathcal{L}(\bbeta_{\nu})=\nu$. $\mathcal{L}(\cdot)$ denotes a distribution of a random variable.

The following result 
provides a crucial contraction property based on distributed clients with infrequent synchronizations. 
\begin{lemma}[Dominated contraction property, restatement of Lemma \ref{contraction_main}]
\label{contraction}
Assume assumptions \ref{def:smooth} and \ref{def:strong_convex} hold. For any learning rate $\eta \in (0, \frac{1}{L+m}]$, any $\bar\theta, \{\theta^c\}_{c=1}^N\in\mathbb{R}^d$, 
we have
\begin{align*}
\small
&\quad\lrn{{\bar\theta}-\theta-\eta(\nabla f({\bar\theta})-{\bm{Z}})}_2^2\leq (1-\eta m) \cdot \|{\bar\theta}-\theta \|_2^2+4\eta L\sum_{c=1}^N p_c \cdot  \|\theta^c-\theta \|_2^2,
\end{align*}
\end{lemma}
where $\theta=\sum_{c=1}^N p_c \theta^c${, although $\theta$ is not accessible when $k \text{ mod } K\neq 0$ as discussed in Eq.(\ref{virtual_seq})}; ${\bm{Z}}=\sum_{c=1}^N p_c \nabla f^c(\theta^c)$. We postpone the proof into Section \ref{DCP}. 

The following result ensures a bounded gap between $\bar\theta_{s}$ and $\bar\theta_{\eta\lfloor\frac{s}{\eta} \rfloor}$ in $\ell_2$ norm for any $s\geq 0$ and $c\in[N]$. We postpone the proof of Lemma~\ref{lem:discretization} into Section~\ref{dis_eroor}.

\begin{lemma}[Discretization error]\label{lem:discretization}
Assume assumptions  \ref{def:smooth}, \ref{def:strong_convex}, and \ref{def:variance} hold. For any $s\geq 0$, any learning rate $\eta \in (0 , 2/m)$ and $\lrn{\theta_0^c-\theta_*}_2^2\leq d\mathcal{D}^2$ for any $c\in[N]$, the iterates of $(\bar \theta_s)$ based on the continuous dynamics of Eq.(\eqref{continuous_dynamics_main}) satisfy the following estimate
\begin{align*}
    \E{ \big\| \bar\theta_{s} - \bar\theta_{\eta\lfloor\frac{s}{\eta} \rfloor} \big\|_2^2} \leq \textcolor{black}{2\eta^2 d\kappa L\tau}+16\eta d\tau .
\end{align*}
\end{lemma}

The following result  
shows that given a finite number of local steps $K$, the divergence between $\theta^c$ in local client and $\theta$ in the center is bounded in $\ell_2$ norm. Notably, since the non-differentiable Brownian motion leads to a lower order term $O(\eta)$ instead of $O(\eta^2)$ in $\ell^2$ norm, a na\"{i}ve proof may lead to a crude upper bound.  We delay the proof of Lemma~\ref{divergence} into Section~\ref{bounded_divergence}.
\begin{lemma}[Bounded divergence, restatement of Lemma \ref{divergence_main}]\label{divergence}
Assume assumptions  \ref{def:smooth}, \ref{def:strong_convex}, and \ref{def:variance} hold. For any learning rate $\eta \in (0 , 2/m)$ and $\lrn{\theta_0^c-\theta_*}_2^2\leq d\mathcal{D}^2$ for any $c\in[N]$, we have the $\ell_2$ upper bound of the divergence between local clients and the center as follows
\begin{align*}
    \sum_{c=1}^N p_c\E{\|\theta_k^c-\theta_k \|_2^2}&\leq 112(K-1)^2\eta^2 d L^2 H_{\rho} +8(K-1)\eta d \tau(\rho^2 + N(1-\rho^2)),\notag
\end{align*}
where $H_{\rho}, \kappa$ and $\gamma^2$ are defined as Definition~\ref{def:H_kappa_gamma}. 
\end{lemma}

The following presents a standard result for bounding the gap between ${\bm{Z}}$ and ${\bm{\widetilde Z}}$. We delay the proof of Lemma~\ref{lem:total_variance} into Setion~\ref{sec:bouding_contraction_discretization_divergence}.

\begin{lemma}[Bounded variance] 
\label{lem:total_variance}
Given assumption \ref{def:variance}, we have 
\begin{equation*}
    \E{ \|{\bm{Z}} -{\bm{\widetilde Z}} \|_2^2}\leq d \cdot \sigma^2 ,\qquad \forall \ \theta\in\R^d.
\end{equation*}
\end{lemma}


Having all the preliminary results ready, now we present a crucial lemma for proving the convergence of all the algorithms.

\begin{lemma}[One step update, restatement of Lemma~\ref{one_step_Dalalyan_main}]\label{one_step_Dalalyan}

Assume assumptions \ref{def:smooth}, \ref{def:strong_convex}, and \ref{def:variance} hold. Consider Algorithm \ref{alg:alg_main_text_independent_noise} with independently injected noise $\rho=0$, any learning rate $\eta \in (0 , \frac{1}{2L})$ and $\lrn{\theta_0^c-\theta_*}_2^2\leq d\mathcal{D}^2$ for any $c\in[N]$, where $\theta_*$ is the global minimum for the function $f$. Then
\begin{align*}
    W_2^2(\mu_{k+1}, \pi)&\leq  \bigg(1-\frac{\eta m}{2}\bigg) \cdot W^2_2(\mu_{k}, \pi)+ 400\eta^2 d L^2 H_0((K-1)^2+\kappa),
\end{align*}
where $\mu_k$ denotes the probability measure of $\theta_k$, $H_0, \kappa$ and $\gamma^2$ are defined as Definition~\ref{def:H_kappa_gamma}. 
\end{lemma}

\begin{proof}

The solution of the continuous-time process Eq.(\eqref{continuous_dynamics_main}) follows that
\begin{align}
\label{solution_continuous_dynamics}
    \bar\theta_t=\bar\theta_0 -\int_0^t \nabla f(\bar\theta_s)\d s + \sqrt{2\tau}\cdot\overline{W}_t, \qquad \forall t\geq 0.
\end{align}

Set $t\rightarrow(k+1)\eta$ and $\bar\theta_0\rightarrow\bar\theta_{k\eta}$ for Eq.(\eqref{solution_continuous_dynamics}) and consider a synchronous coupling such that $W_{(k+1)\eta}-W_{k\eta}:=\sqrt{\eta}\xi_k$
\begin{align}
\label{continuous_one_step}
    \bar\theta_{(k+1)\eta}&=\bar\theta_{k\eta}-\int_{k\eta}^{(k+1)\eta}\nabla f(\bar\theta_s)\d s + \sqrt{2\tau} (W_{(k+1)\eta}-W_{k\eta})\notag\\
    &=\bar\theta_{k\eta}-\int_{k\eta}^{(k+1)\eta}\nabla f(\bar\theta_s)\d s + \sqrt{2\tau\eta}\xi_k.
\end{align}

We first denote $\zeta_k:={\bm{\widetilde Z}_k}-{\bm{Z}_k}$, which is defined after Eq.(\eqref{sum_grad}). Subtracting Eq.(\eqref{fed_avg_langevin_dynamics}) from Eq.(\eqref{continuous_one_step}) yields that
\begin{align}
\label{decompose_full}
    &\quad \bar\theta_{(k+1)\eta}-\theta_{k+1}\notag\\
    &=\bar\theta_{k\eta}-\theta_{k}+\eta {\bm{\widetilde Z}_k} - \int_{k\eta}^{(k+1)\eta}\nabla f(\bar\theta_s)\d s\notag\\
    &=\bar\theta_{k\eta}-\theta_{k}-\eta \bigg(\nabla f(\theta_k+\bar\theta_{k\eta}-\theta_{k})-{\bm{\widetilde Z}_k}\bigg) - \int_{k\eta}^{(k+1)\eta}\bigg(\nabla f(\bar\theta_s)-\nabla f(\bar\theta_{k\eta})\bigg)\d s\\
    &=\bar\theta_{k\eta}-\theta_{k}-\eta \bigg(\underbrace{\nabla f(\theta_k+\bar\theta_{k\eta}-\theta_{k})-{\bm{Z}_k}}_{:=X_k}\bigg)- \underbrace{\int_{k\eta}^{(k+1)\eta}\bigg(\nabla f(\bar\theta_s)-\nabla f(\bar\theta_{k\eta})\bigg)\d s}_{:=Y_k} +\eta\zeta_k.\notag
\end{align}

Taking square and expectation on both sides, we have
\begin{align}
\label{reestimate}
    &\quad\ \E{\|\bar\theta_{(k+1)\eta}-\theta_{k+1} \|_2^2}\notag\\
    &=\E{\| \bar\theta_{k\eta}-\theta_{k}-\eta X_k-Y_k \|_2^2}+\E{\| \eta\zeta_k \|_2^2}+2\eta\underbrace{\E{\langle\bar\theta_{k\eta}-\theta_{k}-\eta X_k-Y_k,  \zeta_k\rangle}}_{\E{\zeta_k}=0 \text{ and mutual independence}}\notag\\
    &\leq (1+q) \cdot \E{\| \bar\theta_{k\eta}-\theta_{k}-\eta X_k \|_2^2}+ ( 1 + 1 / q ) \cdot \E{\|Y_k \|_2^2}+\E{ \| \eta\zeta_k \|_2^2}\notag\\
    &\leq (1+q) \cdot \big( (1-\eta m) \cdot \E{\| \bar\theta_{k\eta}-\theta_k \|_2^2}+4\eta L\sum_{c=1}^N p_c \cdot \left(\E{\| \theta_k^c-\theta_k \|_2^2}\right) \big)\notag\\
    &\quad\quad + ( 1 + 1/q ) \cdot \E{ \| Y_k \|_2^2 } + \eta^2\sigma^2 d\notag\\
    &\leq (1+q) \cdot \bigg(\underbrace{\left(1-\eta m\right)}_{\phi}\E{ \| \bar\theta_{k\eta}-\theta_k \|_2^2}+448\eta^3 d(K-1)^2 L^3 H_0+32(K-1)\eta^2 dL \tau N\bigg)\notag\\
    &\qquad\qquad+ (1+ 1 / q ) \cdot \E{ \| Y_k \|_2^2}+\eta^2\sigma^2  d,
\end{align}
where the first inequality follows by the AM-GM inequality for any $q>0$, the second inequality follows by Lemma \ref{contraction} and Assumption \ref{def:variance}. The third inequality follows by Lemma \ref{divergence} with $\rho=0$. Since the learning rate follows $\frac{1}{2L}\leq \frac{1}{m+L}\leq \frac{2}{m}$, the requirement of the learning rate for Lemma \ref{contraction} and Lemma \ref{divergence} is clearly satisfied.

Recall that $\phi=1-\eta m$, we get $\frac{1+\phi}{2}=1-\frac{1}{2}\eta m$. Choose $q=\frac{1+\phi}{2\phi}-1$ so that $(1+q)\phi=\frac{(1+\phi)}{2}=1-\frac{1}{2}\eta m$. In addition, we have $1+\frac{1}{q}= \frac{1+q}{q}=\frac{1+\phi}{1-\phi}\leq \frac{2}{\eta m}$.  It follows that
\begin{align}
    \label{nice_inequality_v0}
    (1+q) \cdot (1-\eta m)\leq 1-\frac{1}{2}\eta m,  \quad  1+q\leq \frac{1-\frac{1}{2}\eta m}{1-\eta m}\leq 1.5, \quad (1 + 1/q )\leq \frac{2}{m\eta},
\end{align}
where the second inequality holds because $\eta\in (0, \frac{1}{2L}]\leq \frac{1}{2m}$.

For the term $\E{ \| Y_k \|_2^2 }$ in Eq.(\eqref{reestimate}), we have the following estimate
\begin{align}
\label{y_estimate}
    \E{ \| Y_k \|_2^2}&=\E{\lrn{\int_{k\eta}^{(k+1)\eta}\bigg(\nabla f(\bar\theta_s)-\nabla f(\bar\theta_{k\eta})\bigg)\d s}_2^2}\notag\\
    &\leq\eta\int_{k\eta}^{(k+1)\eta}\E{\lrn{\nabla f(\bar\theta_s)-\nabla f(\bar\theta_{k\eta})}_2^2}  \d s\notag\\
    &\leq \eta L^2  \int_{k\eta}^{(k+1)\eta}  \left(\textcolor{black}{2\eta^2 d\kappa L\tau} +16\eta d\tau \right) \d s\notag\\
    &=\textcolor{black}{2}\eta^4 d L^4 H_0+16\eta^3 L^2 d\tau ,
\end{align}
where the first inequality follows by H\"{o}lder's inequality, the second inequality follows by Jensen's inequality, the third inequality follows by Assumption \ref{def:smooth}, and the last inequality follows by Lemma \ref{lem:discretization}. The last equality holds since $L\tau\leq L m H_0$ and $\kappa=L/m$.

Plugging Eq.(\eqref{nice_inequality_v0}) and Eq.(\eqref{y_estimate}) into Eq.(\eqref{reestimate}), we have
\begin{align*}
    \E{\|\bar\theta_{(k+1)\eta}-\theta_{k+1} \|^2_2}&\leq  (1-\frac{\eta m}{2} ) \cdot \E{\|\bar\theta_{k\eta}-\theta_k\|_2^2}\notag\\
    &\quad\quad+ 672\eta^3 d(K-1)^2 L^3 H_0+ 48\eta^2 d(K-1)L \tau N\notag\\
    &\quad\quad+\textcolor{black}{4}\eta^3 d L^3\kappa H_0+32\eta^2 d\frac{L^2}{m} \tau +\eta^2 \sigma^2 d.
\end{align*}

Choose the specific Langevin diffusion $\bar\theta$ in stationary regime, we have $W_2^2(\mu_k,\pi)=\E{\|\bar\theta_{k\eta}-\theta_k \|_2^2}$ and  $W_2^2(\mu_{k+1},\pi)\leq\E{\| \bar\theta_{(k+1)\eta}-\theta_{k+1} \|_2^2}$. Arranging the terms, we have
\begin{align*}
    W_2^2(\mu_{k+1}, \pi)&\leq  (1-\frac{\eta m}{2}) \cdot W^2_2(\mu_{k}, \pi)+ 400\eta^2 d L^2 H_0((K-1)^2+\kappa),
\end{align*}
where $\eta\leq \frac{1}{2L}$, $\kappa\geq 1$, $m\tau\leq L\tau\leq L\tau N\leq L \max_{c\in[N]}T_{c,0}\leq Lm H_0$, and $\sigma^2\leq L^2 H_0$ are applied to the result.

\end{proof}

\subsection{Convergence via independent noises}

\begin{theorem}[Restatement of Theorem~\ref{main_paper_theorem}]\label{main_theorem} 
Assume assumptions \ref{def:smooth}, \ref{def:strong_convex}, and \ref{def:variance} hold. Consider Algorithm \ref{alg:alg_main_text_independent_noise} with a fixed learning rate $\eta\in (0, \frac{1}{2L}]$ and $\lrn{\theta_0^c-\theta_*}_2^2\leq d\mathcal{D}^2$ for any $c\in[N]$, we have
\begin{align*}
    W_2(\mu_{k}, \pi) &\leq  \left(1-\frac{\eta m}{4}\right)^k \cdot \bigg(\sqrt{2d}\big(\mathcal{D} +  \sqrt{\tau/m} \big)\bigg)+30\kappa\sqrt{{\eta} m d} \cdot \sqrt{((K-1)^2+\kappa)H_0} .\notag
\end{align*}
where $\mu_k$ denotes the probability measure of $\theta_k$, $H_0, \kappa$ and $\gamma^2$ are defined as Definition~\ref{def:H_kappa_gamma}.
\end{theorem}

\begin{proof}
Iteratively applying Theorem \ref{one_step_Dalalyan} and arranging terms, we have that
\begin{align}\label{one_step_squared}
    W_2^2(\mu_{k}, \pi)&\leq  \left(1-\frac{\eta m}{2}\right)^k W^2_2(\mu_{0}, \pi)+  \frac{1-(1-\frac{\eta m}{2})^k}{1-(1-\frac{\eta m}{2})}\bigg(400\eta^2 d L^2 H_0((K-1)^2+\kappa)\bigg)\notag\\
    &\leq \left(1-\frac{\eta m}{2}\right)^k W^2_2(\mu_{0}, \pi)+  \frac{2}{\eta m}\bigg(400\eta^2 d L^2 H_0((K-1)^2+\kappa)\bigg)\notag\\
    &\leq \left(1-\frac{\eta m}{2}\right)^k W^2_2(\mu_{0}, \pi)+ 800\kappa^2 \eta m d ((K-1)^2+\kappa) H_{0},
\end{align}
where $\kappa=\frac{L}{m}$. By Lemma \ref{lem:W2_init_bound} and the initialization condition $\lrn{\theta_0^c-\theta_*}_2^2\leq d\mathcal{D}^2$ for any $c\in[N]$, we have that
\begin{align*}
W_2(\mu_0, \pi)\leq \sqrt{2d}(\mathcal{D} +  \sqrt{\tau/m} ).
\end{align*}

Applying the inequality $(1-\frac{\eta m}{2})\leq (1-\frac{\eta m}{4})^2$  completes the proof. 
\end{proof}

{\textbf{Remark on scale invariance:} Setting $K=1$, $\tau=1$ and $\gamma=0$,  we observe that the analysis is scale-invariant overall in the sense that $$\sqrt{\eta d m H_0}\lesssim \underbrace{\sqrt{\eta d m\mathcal{D}^2}}_{\text{I}} + \underbrace{\sqrt{\eta d}+\frac{\sigma\sqrt{\eta d}}{\sqrt{m}}}_{\text{II}},$$
where the important second item $\text{II}$ is consistent with Theorem 4 in \citep{dk19} in terms of scales. Moreover, our divergence term $\frac{\gamma^2}{m^2 d}$ in $H_0$ is in the same order as $\frac{\sigma^2}{m^2}$ and hence validates our result.}

\textbf{Discussions}

\textbf{Optimal choice of $K$.} To ensure the algorithm to achieve the $\epsilon$ precision based on the total number of steps $T_{\epsilon}$ and the learning rate $\eta$, we can set
\begin{align*}
    &30\kappa\sqrt{{\eta}m d} \cdot \bigg(\sqrt{((K-1)^2+\kappa)H_0} \bigg)\leq \frac{\epsilon}{2}\notag\\
    &e^{-\frac{\eta m}{4} T_{\epsilon}} \cdot \bigg(\sqrt{2d}\big(\mathcal{D} +  \sqrt{\tau/m} \big)\bigg)\leq \frac{\epsilon}{2}.
\end{align*}
This directly leads to
\begin{align*}
    \eta m\leq \min\bigg\{\frac{m}{2L}, O\bigg(\frac{\epsilon^2}{d\kappa^2 {((K-1)^2+\kappa)H_0}}\bigg)\bigg\},\quad T_{\epsilon}\geq \Omega\bigg(\frac{\log\big(\frac{d}{\epsilon}\big)}{m\eta}\bigg).
\end{align*}

Plugging into the upper bound of $\eta$, it implies that to reach the precision level $\epsilon$, it suffices to set
\begin{align}\label{def_T}
    T_{\epsilon}=\Omega\bigg(\frac{d\kappa^2 {((K-1)^2+\kappa)H_0}}{\epsilon^2}\cdot \log\bigg(\frac{d}{\epsilon}\bigg)\bigg).
\end{align}
Since $H_0 = \Omega(\mathcal{D}^2+\frac{\tau}{m})$, we observe that the number of communication rounds is around the order
\begin{align*}
    \frac{T_{\epsilon}}{K}=\Omega\bigg( K+\frac{\kappa}{K}\bigg),
\end{align*}
where the value of $\frac{T_{\epsilon}}{K}$ first decreases and then increases with respect to $K$, indicating that setting $K$ either too large or too small may lead to high communication costs and hurt the performance. Ideally, $K$ should be selected in the scale of $\Omega(\sqrt{\kappa})$. Combining the definition of $T_{\epsilon}$ in Eq.(\eqref{def_T}), this suggests an interesting result that the optimal $K$ should be in the order of $O(\sqrt{T_{\epsilon}})$. Similar results have been achieved by \citet{Stich19, lhy+20}.

\subsection{Convergence via varying learning rates}

\begin{theorem}[Restatement of Theorem~\ref{main_paper_theorem_decay}]\label{main_theorem_decay} Assume assumptions \ref{def:smooth}, \ref{def:strong_convex}, and \ref{def:variance} hold. Consider Algorithm \ref{alg:alg_main_text_independent_noise} with an initialization satisfying $\lrn{\theta_0^c-\theta_*}_2^2\leq d\mathcal{D}^2$ for any $c\in[N]$ and varying learning rate following
\begin{align*}
    \eta_{k}=\frac{1}{2L+(1/12)m k},\qquad k=1,2,\cdots.
\end{align*}
Then for any $k\geq 0$, we have
\begin{align*}
    W_2(\mu_{k}, \pi)\leq 45\kappa\sqrt{ ((K-1)^2+\kappa)H_0}\cdot\big(\eta_k m d\big)^{1/2}, \qquad \forall k \geq 0,
\end{align*}
\end{theorem}

\begin{proof}

We first denote 
\begin{align*}
    C_{\kappa}=30\kappa\sqrt{ ((K-1)^2+\kappa)H_0}.
\end{align*}
Next, we proceed to show the following inequality by the induction method
\begin{align}\label{induction}
    W_2(\mu_{k}, \pi)\leq 1.5C_{\kappa}\bigg(\frac{d}{2L+(1/12){m k}}\bigg)^{1/2}=1.5C_{\kappa}\big(\eta_k m d\big)^{1/2}, \qquad \forall k \geq 0,
\end{align}
where the decreasing learning rate follows that
\begin{align*}
    \eta_{k}=\frac{1}{2L+(1/12)m k}.
\end{align*}
(i) For the case of $k=0$, since 
\begin{align}\label{up_bd}
    C_{\kappa}&\geq 4\sqrt{\kappa} \sqrt{H_0}\geq 4\sqrt{\kappa}\sqrt{\mathcal{D}^2 + \frac{1}{m} \max_{c\in[N]} T_{c,0}}\geq 4\sqrt{\kappa/d} \bigg(\sqrt{d\mathcal{D}^2}+\sqrt{\frac{d}{m} \max_{c\in[N]} T_{c,0}}\bigg)\notag\\
    &\geq 4\sqrt{\kappa/d} W_2(\mu_0, \pi),
\end{align}
where the last inequality follows by Lemma \ref{lem:W2_init_bound} and $\lrn{\theta_0^c-\theta_*}_2^2\leq d\mathcal{D}^2$ for any $c\in[N]$. 

It is clear that $W_2(\mu_0, \pi)\leq \frac{1}{4}C_{\kappa} \sqrt{\frac{md}{L}}\leq 1.5C_{\kappa} \sqrt{\eta_0 m d}$ by Eq.(\eqref{up_bd}).

(ii) If now that Eq.(\eqref{induction}) holds for some $k\geq 0$, it follows by Lemma \ref{one_step_Dalalyan} that
\begin{align*}
    W_2^2(\mu_{k+1}, \pi)&\leq \big(1-\frac{\eta_k m}{2}\big) \cdot W_2^2(\mu_{k}, \pi)+400\eta_k^2 d L^2 H_0((K-1)^2+\kappa)\notag\\
    &\leq \big(1-\frac{\eta_k m}{2}\big) \cdot W_2^2(\mu_{k}, \pi)+ \frac{\eta_k^2m^2}{2}C_{\kappa}^2 d\notag\\
    &\leq  \big(1-\frac{\eta_k m}{2}\big) \cdot  2.25  C^2_{\kappa} \eta_k m d+ \frac{\eta_k m}{3}2.25C^2_{\kappa}\eta_k m d\notag\\
    &\leq  \big(1-\frac{\eta_k m}{6}\big) \cdot 2.25   C^2_{\kappa}\eta_k m d.\notag
\end{align*}

Since $\big(1-\frac{\eta_k m}{6}\big)\leq \big(1-\frac{\eta_k m}{12}\big)^2$, we have
\begin{align*}
    W_2(\mu_{k+1}, \pi)&\leq  \big(1-\frac{\eta_k m}{12}\big) \cdot 1.5   C_{\kappa}\big(\eta_k m d\big)^{1/2}.\notag
\end{align*}

To prove $W_2(\mu_{k+1}, \pi)\leq 1.5C_{\kappa}\big(\eta_{k+1} md\big)^{1/2}$, it suffices to show $\big(1-\frac{\eta_k m}{12}\big) \eta_k^{1/2}\leq \eta_{k+1}$, which is detailed as follows
\begin{align*}
    \big(1-\frac{\eta_k m}{12}\big) \eta_k^{1/2}&=\frac{\sqrt{12}(24L+mk-m)}{(24L+mk)^{3/2}}\notag\\
    &\leq \frac{\sqrt{12}(24L+mk-m)^{1/2}}{24L+mk}\notag\\
    &\leq \frac{\sqrt{12}}{(24L+m(k+1))^{1/2}}:=\eta_{k+1},
\end{align*}
where the last inequality follows since 
\begin{align*}
    (24L+mk-m)(24L+m k+m))\leq (24L+mk)^2.
\end{align*}
\end{proof}

The above result implies that to achieve the precision $\epsilon$, we require
\begin{align*}
     W_2(\mu_{k}, \pi)\leq 1.5C_{\kappa}\bigg(\frac{md}{2L+(1/12){mk}}\bigg)^{1/2}\leq \epsilon.
\end{align*}

The means that we only require $k={\Omega}(\frac{d}{\epsilon^2})$ to achieve the precision $\epsilon$. By contrast, the fixed learning rate requires $T_{\epsilon}=\Omega\bigg(\frac{d}{\epsilon^2}\cdot \log\big( {d}/{\epsilon}\big)\bigg)$, which is much slower than the algorithm with varying learning rate by $O\big(\log \big({d}/{\epsilon}\big)\big)$ times.

\subsection{Privacy-accuracy trade-off via correlated noises}

Note that Algorithm \ref{alg:alg_main_text_independent_noise} requires all the local clients to generate the independent noise $\xi^c_k$. Such a mechanism enjoys the convenience of the implementation and yields a potential to protect the privacy of data and alleviates the security issue. However, the scale of noises is maximized and inevitable slows down the convergence. For extensions, it can be naturally generalized to correlated noise based on a hyperparameter, namely the correlation coefficient $\rho$ between different clients. Replacing Eq.(\eqref{local_client}) with 
\begin{equation}\label{local_client_diff_seeds}
    \beta_{k+1}^c=\theta_k^c-\eta\nabla \tilde f^c(\theta_k^c)+\sqrt{2\eta\tau \rho^2}\dot{\xi}_k + \sqrt{2\eta(1-\rho^2)\tau/p_c}\xi_k^c,
\end{equation}
where $\dot{\xi}_k$ is a $d$-dimensional standard Gaussian vector shared by all the clients at iteration $k$, $\xi_k^c$ is a unique $d$-dimensional Gaussian vector generated by client $c\in [N]$ only. Moreover, $\dot\xi_k$ is dependent with $\xi_k^c$ for any $c\in[N]$. Following the same synchronization step based Eq.(\eqref{synchronization}), we have
\begin{equation}
\label{fed_avg_langevin_dynamics_pp}
\theta_{k+1}=\theta_k-\eta {\bm{\widetilde Z}_k}+\sqrt{2\eta\tau}\xi_k,
\end{equation}
where ${\bm{\widetilde Z}_k=\sum_{c=1}^N p_c \nabla \tilde f^c(\theta_k^c)}$ and $\xi_k=\rho \xi_k + \sqrt{1-\rho^2}\sum_{c=1}^N \sqrt{p_c}\xi_k^c$. Since the variance of i.i.d variables is additive, it is clear that $\xi_k$ follows the standard $d$-dimensional Gaussian distribution. The inclusion of the correlated noise implicitly reduces the temperature and naturally yields a trade-off between federation and accuracy. We refer to the algorithm with correlated noise as the hybrid Federated Averaging Langevin dynamics (gFA-LD) and present it in Algorithm \ref{alg:alg_main_text_different_seeds}.

Since the inclusion of correlated noise doesn't affect the formulation of Eq.(\eqref{fed_avg_langevin_dynamics_pp}), the algorithm property maintains the same except the scale of the temperature $\tau$ and federation are changed. Based on a target correlation coefficient $\rho\geq 0$, Eq.(\eqref{local_client_diff_seeds}) is equivalent to applying a temperature $T_{c,\rho}=\tau(\rho^2+(1-\rho^2)/p_c)$. In particular, setting $\rho=0$ leads to $T_{c, 0}=(1-\rho^2)/p_c$, which exactly recovers Algorithm \ref{alg:alg_main_text_independent_noise}; however, setting $\rho=1$ leads to $T_{c, 1}=\tau$, where the injected noise in local clients is reduced by $1/p_c$ times. Now we adjust the analysis as follows
\begin{theorem}[Restatement of Theorem \ref{correlated_noise_main}]\label{correlated_noise_supp} Assume assumptions \ref{def:smooth}, \ref{def:strong_convex}, and \ref{def:variance} hold.  Consider Algorithm \ref{alg:alg_main_text_different_seeds} with a correlation coefficient $\rho\in[0, 1]$, a fixed learning rate $\eta\in (0, \frac{1}{2L}]$ and $\lrn{\theta_0^c-\theta_*}_2^2\leq d\mathcal{D}^2$ for any $c\in[N]$, we have
\begin{align*}
    W_2(\mu_{k}, \pi) &\leq  \left(1-\frac{\eta m}{4}\right)^k \cdot \bigg(\sqrt{2d}\big(\mathcal{D} +  \sqrt{\tau/m} \big)\bigg)+30\kappa\sqrt{{\eta} m d} \cdot \sqrt{((K-1)^2+\kappa)H_{\rho}},\notag
\end{align*}
where $\mu_k$ denotes the probability measure of $\theta_k$, $H_{\rho}, \kappa$ and $\gamma^2$ are defined as Definition~\ref{def:H_kappa_gamma}.
\end{theorem}

\begin{proof}
The proof follows the same techniques as in Theorem \ref{main_theorem} except that $H_0$ is generalized to $H_{\rho}$ to accommodate to the changes of the \emph{injected noise}. The details are omitted.
\end{proof}

\begin{algorithm*}[h]\caption{Hybrid federated averaging Langevin dynamics algorithm. Denote by $\theta_k^c$ the model parameter in the $c$-th client at the $k$-th step. Denote the immediate result of one step SGLD update from $\theta_k^c$ by $\beta_k^c$. $\xi_k^c$ is an independent standard $d$-dimensional Gaussian vector at iteration $k$ for each client $c\in[N]$ and $\dot{\xi}_k$ is a $d$-dimensional standard Gaussian vector shared by all the clients. $\rho$ denotes the correlation coefficient of the injected noises. A global synchronization is conducted every $K$ steps. This is a  clean version of Algorithm~\ref{alg:alg_main_text_partial_main} based on full device updates for ease of analysis.}\label{alg:alg_main_text_different_seeds}
\begin{equation*}
    \beta_{k+1}^c=\theta_k^c-\eta\nabla \tilde f^c(\theta_k^c)+\sqrt{2\eta\tau \rho^2}\dot\xi_k + \sqrt{2\eta(1-\rho^2)\tau/p_c}\xi_k^c,
\end{equation*}
\begin{equation*}  
\theta_{k+1}^c=\left\{  
             \begin{array}{lr}  
             \beta_{k+1}^c \qquad\qquad\qquad \text{if } k+1 \text{ mod } K\neq 0 \\  
              & \\
             \sum_{c=1}^N p_c \beta_{k+1}^c \ \qquad \text{if } k+1 \text{ mod } K=0.
             \end{array}  
\right.  
\end{equation*} 
\end{algorithm*}

\section{Partial device participation}\label{sec:partial_device_participation}

Full device participation enjoys appealing convergence properties. However, it suffers from the straggler's effect in real-world applications, where the communication is limited by the slowest device. Partial device participation handles this issue by only allowing a small portion of devices in each communication and greatly increased the communication efficiency 
in a federated network.

\subsection{Unbiased sampling schemes}
\label{unbiased_sampling_schems_appendix}
The first device-sampling scheme \text{I} \citep{LS20} selects a total of $S$ devices, where the $c$-th device is selected with a probability $p_c$. The first theoretical justification for convex optimization has been proposed by \citet{lhy+20}. 

\paragraph{(Scheme \text{I}: with replacement).}
Assume $\mathcal{S}_k=\{n_1, n_2, \cdots, n_S\}$, where $n_j\in [N]$ is a random number that takes a value of $c$ with a probability $p_c$ for any $j\in\{1,2,\cdots, S\}$. The synchronization step follows that $\theta_{k}=\frac{1}{S}\sum_{c\in \mathcal{S}_k}\theta_{k}^c$.

Another strategy is to uniformly select $S$ devices without replacement. We follow  \citet{lhy+20} and assume $S$ indices are selected uniformly without replacement and the synchronization step is the same as before. In addition, the convergence also requires an additional assumption on balanced data \citep{lhy+20}. 
\paragraph{(Scheme \text{II}: without replacement).}  Assume $\mathcal{S}_k=\{n_1, n_2, \cdots, n_S\}$, where $n_j\in [N]$ is a random number that takes a value of $c$ with a probability $\frac{1}{S}$ for any $j\in\{1,2,\cdots, S\}$. Assume the data is balanced such that $p_1=\cdots=p_N=\frac{1}{N}$. The synchronization step follows that $\theta_{k}=\frac{N}{S}\sum_{c\in \mathcal{S}_k} p_c\theta_{k}^c=\frac{1}{S}\sum_{c\in \mathcal{S}_k} \theta_{k}^c$.


\begin{algorithm*}[h]\caption{Hybrid federated averaging Langevin dynamics algorithm with partial device participation. $\xi_k^c$ is the independent Gaussian vector proposed by each client $c\in[N]$ and $\dot{\xi}_k$ is a unique Gaussian vector shared by all the clients. $\rho$ denotes the correlation coefficient. A global synchronization is conducted every $K$ steps. $\mathcal{S}_k$ is a subset that contains $S$ indices according to a device-sampling rule based on scheme \text{I} or \text{II}. This is a clean version of Algorithm~\ref{alg:alg_main_text_partial_main} for ease of analysis.}\label{alg:alg_main_text_partial}
\begin{equation*}
    \beta_{k+1}^c=\theta_k^c-\eta\nabla \tilde f^c(\theta_k^c)+\sqrt{2\eta\tau \rho^2}\dot\xi_k + \sqrt{2\eta(1-\rho^2)\tau/p_c}\xi_k^c,
\end{equation*}
\begin{equation*}  
\theta_{k+1}^c=\left\{  
             \begin{array}{lr}  
             \beta_{k+1}^c \qquad\qquad\qquad\quad\text{if } k+1 \text{ mod } K\neq 0 \\  
              & \\
             \sum_{c\in \mathcal{S}_{k+1}} \frac{1}{S} \beta_{k+1}^c \ \qquad \text{if } k+1 \text{ mod } K=0.
             \end{array}  
\right.  
\end{equation*} 
\end{algorithm*}

\begin{lemma}[Unbiased sampling scheme]\label{unbiased_scheme}
For any $k \text{ mod } K=0$ based on scheme \text{I} or \text{II}, we have
\begin{align*}
    \E{\theta_k}=\E{\sum_{c\in \mathcal{S}_k} \theta_k^c}=\beta_k:=\sum_{c=1}^N p_c \beta_k^c.
\end{align*}
\end{lemma}

\begin{proof}

According to the definition of scheme \text{I} or \text{II}, we have $\theta_{k}=\frac{1}{S}\sum_{c\in \mathcal{S}_k} \theta_{k}^c$. In what follows, $\E{\theta_k}=\frac{1}{S}\E{\sum_{c\in \mathcal{S}_k} \theta_{k}^c}=\frac{1}{S}\sum_{c_0\in\mathcal{S}_k}\sum_{c=1}^N p_c \beta_k^c=\sum_{c=1}^N p_c \beta_k^c$, where $p_1=p_2=\cdots=p_N$ for scheme \text{II} in particular.
\end{proof}

\subsection{Bounded divergence based on partial device}

\begin{lemma}[Bounded divergence based on partial device]\label{divergence_partial}
Assume assumptions  \ref{def:smooth}, \ref{def:strong_convex}, and \ref{def:variance} hold.  Consider Algorithm \ref{alg:alg_main_text_partial} with a correlation coefficient $\rho\in[0, 1]$, any learning rate $\eta \in (0 , 2/m)$ and $\lrn{\theta_0^c-\theta_*}_2^2\leq d\mathcal{D}^2$ for any $c\in[N]$, we have the following results

For Scheme \text{I}, the divergence between $\theta_k$ and $\beta_k$ is upper bounded by
\begin{align*}
    \E{\|\beta_k-\theta_k \|_2^2}&\leq \frac{112}{S}K^2\eta^2 dL^2H_{\rho} +\frac{8}{S}K\eta d \tau(\rho^2+N(1-\rho^2)).\notag
\end{align*}

For Scheme \text{II}, assuming the data is balanced such that $p_1=\cdots=p_N=\frac{1}{N}$, the divergence between $\theta_k$ and $\beta_k$ is upper bounded by
\begin{align*}
    \E{\|\beta_k-\theta_k \|_2^2}&\leq \frac{N-S}{S(N-1)} \bigg(112K^2\eta^2 dL^2H_{\rho} +8K\eta d \tau(\rho^2+N(1-\rho^2))\bigg).\notag
\end{align*}
where $H_{\rho}, \kappa$ and $\gamma^2$ are defined as Definition~\ref{def:H_kappa_gamma}. 
\end{lemma}

\begin{proof} We prove the bounded divergence for the two schemes, respectively.

For \textbf{scheme \text{I}} with replacement, $\bar\theta_{k}=\sum_{c\in \mathcal{S}_k} \frac{1}{S} \beta_{k}^c$ for a subset of indices $\mathcal{S}_k$. Taking expectation with respect to $\mathcal{S}_{k}$,
we have
\begin{align}\label{scheme_1}
    \E{\lrn{\theta_{k}-\beta_{k}}_2^2}=\frac{1}{S^2}\sum_{i=1}^S\E{\lrn{\beta_{k}^{n_i}-\beta_{k}}_2^2}=\frac{1}{S}\sum_{c=1}^N p_c \lrn{\beta_{k}^c-\beta_{k}}_2^2,
\end{align}
where the first equality follows by the independence and unbiasedness of $\theta_{k}^{n_i}$ for any $i\in [S]$. 

To further upper bound Eq.(\eqref{scheme_1}), we follow the same technique as in Lemma \ref{divergence}. Since $k\text{ mod } K=0$, $k_0=k-K$ is also the communication time, which yields the same $\theta_{k_0}^{c}$ for any $c\in[N]$. in what follows,
\begin{align}\label{scheme_1_step2}
    \sum_{c=1}^N p_c\lrn{\beta_{k}^c-\beta_{k}}_2^2&=\sum_{c=1}^N p_c \lrn{\beta_k^c-\theta_{k_0}-(\beta_k-\theta_{k_0})}_2^2\notag\\
&\leq \sum_{c=1}^N p_c \lrn{\beta_k^c-\theta_{k_0}}_2^2,
\end{align}
where the last inequality follows by $\beta_{k}=\sum_{c=1}^N p_c \beta_{k}^c$ and $\E{\lrn{x-\E{x}}_2^2}\leq \E{\lrn{x}_2^2}$. Combining Eq.(\eqref{scheme_1}) and Eq.(\eqref{scheme_1_step2}), we have
\begin{align*}
    \E{\lrn{\theta_{k}-\beta_{k}}_2^2}&\leq \frac{1}{S}\sum_{c=1}^N p_c \lrn{\beta_k^c-\theta_{k_0}}_2^2\notag\\
    &\leq \frac{1}{S}\sum_{c=1}^N p_c \lrn{\beta_k^c-\theta^c_{k_0}}_2^2\notag\\
    &\leq \frac{1}{S}\sum_{c=1}^N p_c \E{\sum_{k=k_0}^{k-1} 2K\eta^2\lrn{\nabla\tilde f^c(\theta_k^c)}_2^2 + 4K\eta d \tau\big(\rho^2+(1-\rho^2)/p_c\big)}\notag\\
&\leq \frac{1}{S}\sum_{c=1}^N p_c \left(\sum_{k=k_0}^{k-1} 2K\eta^2\E{\lrn{\nabla\tilde f^c(\theta_k^c)}_2^2}+4K\eta d \tau\big(\rho^2+(1-\rho^2)/p_c\big)\right)\notag\\
&\leq \frac{28}{S}K^2\eta^2 dL^2 H_{\rho} +\frac{4}{S}K\eta d \tau(\rho^2+N(1-\rho^2))\notag\\
\end{align*}
where the last inequality follows a similar argument as in Lemma \ref{divergence}.

For \textbf{scheme \text{II}}, given $p_1=p_2=\cdots=p_N=\frac{1}{N}$, we have $\theta_{k}=\frac{1}{S}\sum_{c\in \mathcal{S}_k} \beta_{k}^{c}$, which leads to
\begin{align*}
    &\quad\E{\lrn{\theta_{k}-\beta_{k}}_2^2}=\E{\lrn{\frac{1}{S}\sum_{c\in \mathcal{S}_k} \beta_{k}^{c}-\beta_{k}}_2^2}=\frac{1}{S^2}\E{\lrn{\sum_{c=1}^N \mathbb{I}_{c\in \mathcal{S}_k}(\beta_{k}^c-\beta_{k})}_2^2},
\end{align*}
where $\mathbb{I}_{A}$ is an indicator function that equals to 1 if the event $A$ happens.

Plugging the facts that $\mathbb{P}(c\in \mathcal{S}_{k})=\frac{S}{N}$ and $\mathbb{P}(c_1,c_2\in \mathcal{S}_{k})=\frac{S(S-1)}{N(N-1)}$ for any $c_1\neq c_2\in [N]$ into the above equation, we have
\begin{align*}
    &\quad\E{\lrn{\theta_{k}-\beta_{k}}_2^2}\notag\\
    &=\frac{1}{S^2}\bigg[\sum_{c\in [N]} \mathbb{P}(c\in \mathcal{S}_{k}) \lrn{\beta_{k}^c-\beta_{k}}_2^2+\sum_{c_1\neq c_2} \mathbb{P}(c_1,c_2\in \mathcal{S}_{k})\langle \beta_{k}^{c_1}-\beta_{k}, \beta_{k}^{c_2}-\beta_{k} \rangle \bigg]\notag\\
    &=\frac{1}{SN}\sum_{c=1}^N\lrn{\beta_{k}^c-\beta_{k}}_2^2+\sum_{c_1\neq c_2}\frac{S-1}{SN(N-1)} \langle \beta_{k}^{c_1}-\beta_{k}, \beta_{k}^{c_2}-\beta_{k} \rangle\notag\\
    &=\frac{1-\frac{S}{N}}{S(N-1)}\sum_{c=1}^N\lrn{\beta_{k}^c-\beta_{k}}_2^2,
\end{align*}
where the last equality holds since $\sum_{c\in[N]}\lrn{\beta_{k}^c-\beta_{k}}_2^2 +\sum_{c_1\neq c_2}\langle \beta_{k}^{c_1}-\beta_{k},\beta_{k}^{c_2}-\beta_{k}\rangle=\lrn{\beta_{k}-\beta_{k}}_2^2=0$.

Eventually, we have
\begin{align*}
    \E{\lrn{\theta_{k}-\beta_{k}}_2^2}&=\frac{N-S}{S(N-1)} \E{\frac{1}{N} \sum_{c=1}^N\lrn{\beta_{k}^c-\beta_{k}}_2^2}\notag\\
    &\leq\frac{N-S}{S(N-1)} \E{\frac{1}{N} \sum_{c=1}^N\lrn{\beta_{k}^c-\theta_{k_0}}_2^2}\notag\\
    &\leq \frac{N-S}{S(N-1)} \bigg(28 K^2\eta^2 dL^2 H_{\rho} +4K\eta d \tau\big(\rho^2+N(1-\rho^2)\big)\bigg),
\end{align*}
where the first inequality follows a similar argument as in Eq.(\eqref{scheme_1_step2}) and the last inequality follows by Lemma \ref{divergence}.

\end{proof}

\subsection{Convergence via partial device participation}

\begin{theorem}[Restatement of Theorem~\ref{thm:partial_II}]\label{theorem_partial} Assume assumptions \ref{def:smooth}, \ref{def:strong_convex}, and \ref{def:variance} hold. Consider Algorithm \ref{alg:alg_main_text_partial} with a correlation coefficient $\rho\in[0, 1]$, a fixed learning rate $\eta\in (0, \frac{1}{2L}]$ and $\lrn{\theta_0^c-\theta_*}_2^2\leq d\mathcal{D}^2$ for any $c\in[N]$, we have
\begin{align*}
    W_2(\mu_{k}, \pi) &\leq  \left(1-\frac{\eta m}{4}\right)^k \cdot \bigg(\sqrt{2d}\big(\mathcal{D} +  \sqrt{\tau/m} \big)\bigg)\notag\\
    &\qquad+30\kappa\sqrt{\eta m d } \cdot \sqrt{ H_{\rho}((K-1)^2+\kappa)}+2\sqrt{\frac{C_K d\tau}{Sm}(\rho^2+N(1-\rho^2)) C_S},
\end{align*}
where $C_K=\frac{\eta m K}{1-e^{-\frac{\eta m K}{2}}}$, $C_S=1$ for \emph{Scheme I} and $C_S=\frac{N-S}{N-1}$ for \emph{Scheme II}.
\end{theorem}

\begin{proof}

Note that 
\begin{align*}
&\quad\ \E{\lrn{\bar\theta_{(k+1)\eta}-\theta_{k+1}}_2^2}\notag\\
&= \E{\lrn{\bar\theta_{(k+1)\eta}-\beta_{k+1}+\beta_{k+1}-\theta_{k+1}}_2^2}\notag\\
&= \E{\lrn{\bar\theta_{(k+1)\eta}-\beta_{k+1}}_2^2}+ \E{\lrn{\beta_{k+1}-\theta_{k+1}}_2^2}+\E{2\langle \bar\theta_{(k+1)\eta}-\beta_{k+1}, \beta_{k+1}-\theta_{k+1} \rangle}\notag\\
&= \E{\lrn{\bar\theta_{(k+1)\eta}-\beta_{k+1}}_2^2}+ \E{\lrn{\beta_{k+1}-\theta_{k+1}}_2^2},\notag
\end{align*}
where the last equality follows by the unbiasedness of the device-sampling scheme in Lemma \ref{unbiased_scheme}.

If $k+1 \text{ mod } K\neq 0$, we always have $\beta_{k+1}=\theta_{k+1}$ and $\E{\lrn{\beta_{k+1}-\theta_{k+1}}_2^2}=0$. Following the same argument as in Lemma \ref{one_step_Dalalyan}, both schemes lead to the one-step iterate as follows
\begin{align}\label{non_period}
    W_2^2(\mu_{k+1}, \pi)&\leq  (1-\frac{\eta m}{2}) \cdot W^2_2(\mu_{k}, \pi)+  400\eta^2 d L^2 H_{\rho}((K-1)^2+\kappa).
\end{align}

If $k+1 \text{ mod } K= 0$, combining Lemma \ref{divergence_partial} and Lemma \ref{one_step_Dalalyan}, we have
\begin{align}\label{period}
    W_2^2(\mu_{k+1}, \pi)&\leq  (1-\frac{\eta m}{2}) \cdot W^2_2(\mu_{k}, \pi)+ 450\eta^2 d L^2 H_{\rho}(K^2+\kappa) + \frac{4Kd\eta\tau}{S}(\rho^2+N(1-\rho^2)) C_S,
\end{align}
where $C_S=1$ for \emph{Scheme I} and $C_S=\frac{N-S}{N-1}$ for \emph{Scheme II}.

Repeatedly applying Eq.(\eqref{non_period}) and Eq.(\eqref{period}) and arranging terms, we have that
\begin{align*}
    W_2^2(\mu_{k}, \pi)&\leq  \left(1-\frac{\eta m}{2}\right)^k W^2_2(\mu_{0}, \pi)+  \frac{2}{\eta m}\bigg(450\eta^2 d L^2 H_{\rho}(K^2+\kappa)\bigg)\notag\\
    &\qquad+ \frac{(1-(1-\frac{\eta m}{2})^K)^{\lfloor k/K\rfloor}}{1-(1-\frac{\eta m}{2})^K}\left(  \frac{4Kd\eta\tau}{S} (\rho^2+N(1-\rho^2)) C_S  \right)\notag\\
    &\leq \left(1-\frac{\eta m}{2}\right)^k W^2_2(\mu_{0}, \pi)+ 900\eta m d \kappa^2  H_0((K-1)^2+\kappa)\notag\\
    &\quad\quad+\underbrace{\frac{\eta m K}{1-e^{-\frac{\eta m K}{2}}}}_{C_K} \frac{4Kd\eta\tau}{\eta mK S} (\rho^2+N(1-\rho^2)) C_S ,\notag\\
    &= \left(1-\frac{\eta m}{2}\right)^k W^2_2(\mu_{0}, \pi)+ 900\eta m d \kappa^2  H_0((K-1)^2+\kappa)\notag\\
    &\quad\quad+  \frac{4C_K d\tau}{Sm}(\rho^2+N(1-\rho^2)) C_S,\notag
\end{align*}
where the second inequality follows by $(1-r)^K\leq e^{-rK}$ for any $r\geq 0$.

\end{proof}

\section{Bounding contraction, discretization, and divergence}\label{sec:bouding_contraction_discretization_divergence}

\subsection{Dominated contraction property}
\label{DCP}
\begin{proof}[Proof of Lemma \ref{contraction} ]

Given a client index $c\in[N]$, applying Theorem 2.1.12 \citep{Nesterov04} leads to
\begin{align}
\label{special_inner_product}
    \langle y-x, \nabla f^c(y)-\nabla f^c(x) \rangle\geq \frac{m L}{L+m}\lrn{y-x}_2^2 + \frac{1}{L+m} \lrn{\nabla f^c(y)-\nabla f^c(x)}_2^2,\quad \forall x,y\in\mathbb{R}^d.
\end{align}

In what follows, we have
\begin{align}
\label{iteration}
    &\quad\lrn{{\bar\theta}-\theta-\eta(\nabla f({\bar\theta})-{\bm{Z}})}_2^2\notag\\
    &=\lrn{{\bar\theta}-\theta}_2^2 -2\eta \underbrace{\langle {\bar\theta}-\theta, \nabla f({\bar\theta})-{\bm{Z}}\rangle}_{\mathcal{I}}+\eta^2 \lrn{\nabla f({\bar\theta})-{\bm{Z}}}_2^2.
\end{align}

For the second item $\mathcal{I}$ in the right hand side, {by the definition of} ${\bm{Z}=\sum_{c=1}^N p_c \nabla f^c(\theta^c)}$ and the fact that $\nabla f(\bar\theta)=\sum_{c=1}^N p_c \nabla f^c({\bar\theta})$, we have
\begin{align}
\label{target_contraction}
    \mathcal{I}&=\sum_{c=1}^N p_c\big\langle {\bar\theta}-\theta, \nabla f^c({\bar\theta})-\nabla f^c(\theta^c)\big\rangle\notag\\
    &=\sum_{c=1}^N p_c\big\langle {\bar\theta}-\theta^c+\theta^c-\theta, \nabla f^c({\bar\theta})-\nabla f^c(\theta^c)\big\rangle\notag\\
    &=\sum_{c=1}^N p_c \underbrace{\big\langle \bar\theta-\theta^c, \nabla f^c({\bar\theta})-\nabla f^c(\theta^c)\big\rangle}_{\text{by Eq.} (\eqref{special_inner_product})}+\sum_{c=1}^N p_c\underbrace{\big\langle {\theta}^c-\theta, \nabla f^c({\bar\theta})-\nabla f^c(\theta^c)\big\rangle}_{\text{by the AM-GM inequality}}\notag\\
    &\geq \sum_{c=1}^N p_c \cdot \bigg(\frac{m L}{L+m}\lrn{{\bar\theta}-\theta^c}_2^2 + \frac{1}{L+m} \lrn{\nabla f^c({\bar\theta})-\nabla f^c(\theta^c)}_2^2 \bigg)\notag\\
    & \qquad -\sum_{c=1}^N p_c \cdot \bigg((m+L)\lrn{\theta^c-\theta}_2^2+\frac{1}{{4}(m+L)}\lrn{\nabla f^c({\bar\theta})-\nabla f^c(\theta^c)}_2^2\bigg)\notag\\
    &\geq -(m+L)\sum_{c=1}^N p_c \lrn{\theta^c-\theta}_2^2 + \frac{m L}{L+m}\lrn{{\bar\theta}-\theta^c}_2^2 + \frac{{3}}{{4}(L+m)} \lrn{\nabla f({\bar\theta})-{\bm{Z}}}_2^2,
\end{align}
where the last inequality follows by Jensen's inequality such that
\begin{align*}
    \sum_{c=1}^N p_c \| {\bar\theta}-\theta^c \|_2^2&\geq \lrn{\sum_{c=1}^N p_c  ({\bar\theta}-\theta^c )}_2^2=\| {\bar\theta}-\theta \|_2^2\notag\\
    \sum_{c=1}^N p_c \lrn{\nabla f^c({\bar\theta})-\nabla f^c(\theta^c)}_2^2&\geq  \lrn{\sum_{c=1}^N p_c\bigg(\nabla f^c({\bar\theta})-\nabla f^c(\theta^c)\bigg)}_2^2= \lrn{\nabla f({\bar\theta})-{\bm{Z}}}_2^2.
\end{align*}

Plugging Eq.(\eqref{target_contraction}) into Eq.(\eqref{iteration}), we have
\begin{align*}
    &\quad\lrn{{\bar\theta}-\theta-\eta \cdot (\nabla f({\bar\theta})-{\bm{Z}})}_2^2\notag\\
    &\leq \big(1-\frac{2\eta mL}{m+L}\big) \cdot \| {\bar\theta}-\theta \|_2^2+\eta\big(\underbrace{\eta-\frac{{3}}{{2}{(m+L)}}}_{\leq 0 \text{ if } \eta\leq \frac{1}{m+L}}\big) \cdot \| \nabla f({\bar\theta})-{\bm{Z}} \|_2^2\notag\\
    &\quad\quad+2\eta(m+L)\sum_{c=1}^N p_c \cdot \| \theta^c-\theta \|_2^2 \notag\\
    &\leq \left(1-\eta m\right) \|{\bar\theta}-\theta \|_2^2+4\eta L\sum_{c=1}^N p_c \cdot \| \theta^c-\theta \|_2^2,\notag
\end{align*}
where the last inequality follows by $\frac{2L}{m+L}\geq 1$, $m\leq L$, $1-2a\leq (1-a)^2$ for any $a$, and $\eta\in(0, \frac{1}{m+L}]$.

\end{proof}


\subsection{Discretization error}\label{dis_eroor}
\begin{proof}[Proof of Lemma \ref{lem:discretization}]

In the continuous-time diffusion (\eqref{continuous_dynamics_main}), we have $\nabla f(\bar\theta)=\sum_{c=1}^N p_c f^c(\bar\theta)$ for any $\bar\theta\in \mathbb{R}^d$ and it is straightforward to verify that $f$ satisfies both Assumption \ref{def:smooth} and \ref{def:strong_convex} with the same smoothness factor $L$ and convexity constant $m$. For any $s\in[0,\infty)$, there exists a certain $k \in \mathbb{N}^+$ such that $s\in [k\eta, (k+1)\eta)$. By the continuous diffusion of Eq.(\eqref{continuous_dynamics_main}) , we have
\begin{align*}
    \bar\theta_{s} - \bar\theta_{\eta\lfloor\frac{s}{\eta} \rfloor} = -\int^s_{k\eta}  \nabla f(\bar\theta_{t})d t+\sqrt{2\tau}\int_{k\eta}^s \d \overline{W}_t, 
\end{align*}
which suggests that 
\begin{align*}
    \sup_{ s \in [ k\eta, (k+1)\eta ) } \big\| \bar\theta_{s}-\bar\theta_{\eta\lfloor\frac{s}{\eta} \rfloor} \big\|_2 \leq \bigg\| \int^s_{k\eta}  \nabla f(\bar\theta_t) d t \bigg\|_2+\sup_{ s \in [ k\eta,  (k+1)\eta ) }\lrn{\int_{k\eta}^s \sqrt{2\tau} \d \overline{W}_t}_2.
\end{align*}
We first square the terms on both sides and take Young’s inequality and expectation
\begin{align*}
    \E{\sup_{ s \in [ k\eta, (k+1)\eta ) } \big\| \bar\theta_{s}-\bar\theta_{\eta\lfloor\frac{s}{\eta} \rfloor} \big\|_2^2} &\leq 2\E{\bigg\|\int^s_{k\eta}  \nabla f(\bar\theta_t) dt \bigg\|_2^2}+2\E{\sup_{ s \in [ k\eta,  (k+1)\eta ) }\lrn{\int_{k\eta}^s \sqrt{2\tau } \d \overline{W}_t}_2^2}.
\end{align*}
Then, by Cauchy Schwarz inequality and the fact that $|s-k\eta|\leq \eta$, we have
\begin{align}
    \label{eq:1st_part}
    \E{\sup_{ s \in [ k\eta, (k+1)\eta ) } \big\| \bar\theta_{s}-\bar\theta_{\eta\lfloor\frac{s}{\eta} \rfloor} \big\|_2^2}&\leq 2\eta\E{\int^s_{k\eta} \big\| \nabla f(\bar\theta_t)dt \big\|_2^2 dt}+8\sum_{i=1}^d\E{\int_{k\eta}^s 2\tau  \d t} \notag \\
    &\leq 2\eta^2 \sup_{s}\E{ \big\| \nabla f(\bar\theta_{s}) \big\|_2^2}+16 \eta d\tau ,
\end{align}
where the last inequality follows by Burkholder-Davis-Gundy inequality (\ref{BDG-inequality}) and It\^{o} isometry.

By Young's inequality, the smoothness assumption \ref{def:smooth} and $\nabla f(\theta_*)=0$ since $\theta_*$ is the global minimum of $f$, we have
\begin{align}\label{eq:2nd_part}
    \sup_s \E{ \| \nabla f(\bar\theta_{s}) \|_2^2}
    = & ~ \sup_s \E{\| \nabla f(\bar\theta_{s})-\nabla f(\theta_*)\|_2^2} \notag \\
    \leq & ~ L^2 \sup_s \E{\|\bar\theta_{s }-\theta_*\|_2^2}\notag\\
    \leq & ~ \textcolor{black}{L^2 \frac{d\tau}{m}},
\end{align}
where the second inequality follows by Theorem 17 \citep{ccbj18} since $\bar\theta_0$ is simulated from the stationary distribution $\pi$ and $\bar\theta_s$ is stationary. Combining Eq.(\eqref{eq:1st_part}) and Eq.(\eqref{eq:2nd_part}), we have
\begin{align*}
\E{\sup_{ s \in [ k\eta, (k+1)\eta ) } \big\| \bar\theta_{s}-\bar\theta_{\eta\lfloor\frac{s}{\eta} \rfloor} \big\|_2^2}
&\leq \textcolor{black}{2\eta^2 \kappa L d\tau}+16\eta d\tau .\notag
\end{align*}

\end{proof}

\subsection{Bounded divergence}\label{bounded_divergence}
\begin{proof}[Proof of Lemma \ref{divergence}] For any $k \ge 0$, consider $k_0=K\lfloor \frac{k}{K}\rfloor $ such that $k\leq k_0$ and $\theta_{k_0}^c=\theta_{k_0}$ for any $k\geq 0$. It is clear that  $k-k_0 \leq K-1$ for all $k\geq 0$. Consider the non-increasing learning rate such that $\eta_{k_0}\leq 2\eta_k$ for all $k-k_0\leq K-1$.

By the iterate Eq.(\eqref{fed_avg_langevin_dynamics}), we have
\begin{align*}
&\quad\sum_{c=1}^N p_c\E{\lrn{\theta_k^c-\theta_k}_2^2}\notag\\
&=\sum_{c=1}^N p_c\E{\lrn{\theta_k^c-\theta_{k_0}-(\theta_k-\theta_{k_0})}_2^2}\notag\\
&\leq \sum_{c=1}^N p_c\E{\lrn{\theta_k^c-\theta_{k_0}}_2^2}\notag\\
&\leq \sum_{c=1}^N p_c \E{\sum_{k=k_0}^{k-1} 2 (K-1)\eta_k^2\lrn{\nabla\tilde f^c(\theta_k^c)}_2^2 + 4(K-1)\eta_k d \tau(\rho^2+(1-\rho^2)/p_c)}\notag\\
&\leq \sum_{c=1}^N p_c \bigg(\sum_{k=k_0}^{k-1} 2(K-1)\eta_{k_0}^2\E{\lrn{\nabla\tilde f^c(\theta_k^c)}_2^2}+4(K-1)\eta_{k_0} d \tau(\rho^2+(1-\rho^2)/p_c)\bigg)\notag\\
&\leq 112(K-1)^2\eta_k^2 d L^2 H_{\rho} +8(K-1)\eta_k d\tau(\rho^2 + N(1-\rho^2)),
\end{align*}
where the first inequality holds by $\E{\| \theta-\E{\theta} \|_2^2}\leq \E{\|\theta \|_2^2}$ for a stochastic variable $\theta$; the second inequality follows by $(\sum_{i=1}^{K-1} a_i)^2\leq (K-1)\sum_{i=1}^{K-1} a_i^2$; the last inequality follows by Lemma \ref{bounded_gradient_l2} and  $\eta_{k_0}^2\leq 4\eta_k^2$. $H_{\rho}$ is defined in Definition \ref{def:H_kappa_gamma}.

\end{proof}

\subsection{Bounded variance}
\begin{proof}[Proof of Lemma \ref{lem:total_variance}] By assumption \ref{def:variance}, we have
\begin{align*}
    \E{\lrn{{\bm{Z}}-{\bm{\widetilde Z}}}_2^2}&=\E{\lrn{\sum_{c=1}^N p_c\bigg(\nabla f^c(\theta^c)-\nabla \tilde f^c(\theta^c)\bigg)}_2^2}\\
    &=\sum_{c=1}^N p_c^2\E{\lrn{\nabla f^c(\theta^c)-\nabla \tilde f^c(\theta^c)}_2^2}\\
    &\leq d \sigma^2 \sum_{c=1}^N p_c^2\leq d\sigma^2 \left(\sum_{c=1}^N p_c\right)^2:=d\sigma^2.
\end{align*}

\end{proof}

\section{Uniform upper bound}\label{sec:uniform_upper_bound}

\subsection{Discrete dynamics}

\begin{lemma}[Discrete dynamics]
\label{lem:L2_bound_local}
Assume assumptions  \ref{def:smooth}, \ref{def:strong_convex}, and \ref{def:variance} hold. We consider the generalized formulation in Algorithm \ref{alg:alg_main_text_different_seeds} with the temperature
$$T_{c,\rho}=\tau(\rho^2+(1-\rho^2)/p_c)$$ given a correlation coefficient $\rho$. For any learning rate $\eta \in (0 , 2/m)$ and $\lrn{\theta_0^c-\theta_*}_2^2\leq d\mathcal{D}^2$ for any $c\in[N]$, we have the $\ell_2$ norm upper bound as follows
\begin{align*}
\sup_k\E{\lrn{\theta_k^c-\theta_*}_2^2}\leq d\mathcal{D}^2 + {\frac{6d}{m}\bigg(\max_{c\in[N]} T_{c, \rho}+\frac{ \sigma^2}{m} + \frac{\gamma^2 }{md}\bigg)},\notag
\end{align*}
where $\gamma:=\max_{c\in[N]}\lrn{\nabla f^c(\theta_*)}_2$ and $\theta_*$ denotes the global minimum for the function $f$.
\end{lemma}

\begin{proof} First, we consider the $k$-th iteration, where $k\in \{1,2,\cdots, K-2, (K-1)_{-}\}$ and $(K-1)_-$ denotes the $(K-1)$-step before synchronization. Following the iterate of Eq.(\eqref{local_client}) in a local client of $c\in [N]$, we have
	\begin{align}\label{eq:Langevin_L2_1_local}
&\quad\ \E{\lrn{\theta_{k+1}^c-\theta_*}_2^2}\notag\\
		&= \E{\|\theta_k^c -\theta_*- \eta\nabla \tilde f^c(\theta_k^c)\|_2^2} + \sqrt{8\eta T_{c,\rho}}\E{ \langle \theta_k^c -\theta_*- \eta\nabla \tilde f^c(\theta_k^c), \xi_k \rangle } + 2\eta T_{c,\rho}\E{\|\xi_k\|_2^2} \notag \\
		&= \E{\|\theta_k^c -\theta_*- \eta\nabla \tilde f^c(\theta_k^c)\|_2^2} + 2\eta d T_{c,\rho},
	\end{align}	
	where the last equality follows from $\E{\xi_k}=0$ and the conditional independence of $\theta_k^c-\theta_*- \widetilde f^c(\theta_k^c)$ and $\xi_k$. Note that
\begin{align}\label{eq:ip_1st_local}
&\quad\ \E{\|\theta_k^c -\theta_*- \eta \widetilde f^c(\theta_k^c)\|_2^2} \notag\\
&= \E{\left\|\theta_k^c -\theta_*- \eta \nabla f^c(\theta_k^c) \right\|_2^2} + \eta^2\E{\|\nabla f^c(\theta_k^c)-\nabla \widetilde f^c(\theta_k^c)\|_2^2}  \notag\\
& \qquad\qquad + 2 \eta \E{ \langle \theta_k^c-\theta_*-\eta \nabla f^c(\theta_k^c),\nabla f^c(\theta_k^c)-\nabla\widetilde f^c(\theta_k^c) \rangle }  \notag\\
&= \E{\left\|\theta_k^c -\theta_*- \eta \nabla f^c(\theta_k^c) \right\|_2^2} + \eta^2\E{\|\nabla f^c(\theta_k^c)-\nabla \widetilde f^c(\theta_k^c)\|_2^2} \notag \\
&\leq \E{\|\theta_k^c -\theta_*- \eta \nabla f^c(\theta_k^c) \|_2^2}  + \eta^2 d\sigma^2, 
\end{align}
where the first step follows from simple algebra, the second step follows from the unbiasedness of the stochastic gradient, and the last step follows from Assumption \ref{def:variance}. For any $q>0$, we can upper bound the first term of Eq.(\eqref{eq:ip_1st_local}) as follows
\begin{align}\label{eq:ip_2nd_test_theta_star}
	&\quad\ \E{\|\theta_k^c -\theta_*- \eta \nabla f^c(\theta_k^c) \|_2^2}\notag\\
	&=\E{\|\theta_k^c -\theta_*- \eta (\nabla f^c(\theta_k^c)-\nabla f^c(\theta_*))-\eta\nabla f^c(\theta_*) \|_2^2}\notag\\
	&\leq (1+q)\E{\|\theta_k^c -\theta_*- \eta (\nabla f^c(\theta_k^c)-\nabla f^c(\theta_*)) \|_2^2}+\eta^2 \left(1+\frac{1}{q}\right) \|\nabla f^c(\theta_*)\|_2^2\notag\\
	&\leq (1+q)\underbrace{\left(1-\frac{\eta m}{2}\right)^2}_{\psi^2}\E{\lrn{\theta_k^c-\theta_*}_2^2}+\eta^2 \left(1+\frac{1}{q}\right)\gamma^2,
\end{align}
where the first inequality follows by the AM-GM inequality;  the second inequality is a special case of Lemma \ref{contraction} based on Assumption \ref{def:strong_convex}, where no local steps is involved before the synchronization step. Similar results have been achieved in Theorem 3 \citep{Dalalyan17}. In addition, $\gamma:=\max_{c\in[N]}\lrn{\nabla f^c(\theta_*)}_2$.

Choose $q=(\frac{1+\psi}{2\psi})^2-1$ so that $(1+q)\psi^2=\frac{(1+\psi)^2}{4}$. Moreover, since $\psi=1-\frac{\eta m}{2}$, we get $\frac{1+\psi}{2}=1-\frac{1}{4}\eta m$. In addition, we have $1+\frac{1}{q}= \frac{1+q}{q}= \frac{(1+\psi)^2}{(1-\psi)(1+3\psi)}\leq \frac{2}{\eta m}$.  It follows that
\begin{align}
    \label{nice_inequality}
    \eta^2\left(1+\frac{1}{q}\right)\leq \frac{2\eta}{m}.
\end{align}

Combining Eq.(\eqref{eq:Langevin_L2_1_local}), Eq.(\eqref{eq:ip_1st_local}), Eq.(\eqref{eq:ip_2nd_test_theta_star}), and Eq.(\eqref{nice_inequality}), we have the following iterate
\begin{align*}
	\E{\|\theta_{k+1}^c-\theta_*\|_2^2} 
	\leq & ~ \underbrace{\left(1-\frac{\eta m}{4}\right)^2}_{:=g(\eta)} \E{\|\theta_k^c-\theta_*\|_2^2} + 2\eta d T_{c,\rho} +\eta^2 d \sigma^2+\frac{2\eta \gamma^2}{m}. \notag
\end{align*}

Note that $\frac{1}{1-g(\eta)}=\frac{1}{\frac{\eta m}{2}(1-\frac{\eta m}{8})}\leq \frac{3}{\eta m}$ given $\eta\in (0, \frac{2}{m})$. Recursively applying the above equation $k$ times, where $k\in \{1,2,\cdots, K-1, K_{-}\}$ and $K_-$ denotes the $K$-step without synchronization, it follows that
\begin{align}\label{recursion_v2}
	\E{\|\theta_k^c-\theta_*\|_2^2} &\le g(\eta)^{k}\| \theta_0^c-\theta_*\|_2^2 + \frac{1- g(\eta)^{k}}{1 - g(\eta)} \cdot \left(2\eta d T_{c,\rho} +\eta^2 d \sigma^2+\frac{2\eta \gamma^2}{m}\right)  \\
	&\le \|\theta_0^c-\theta_*\|_2^2 + \frac{3}{\eta m} \cdot \left(2\eta d T_{c,\rho} +\eta^2 d \sigma^2+\frac{2\eta \gamma^2}{m}\right) \notag\\
	&\leq d\mathcal{D}^2 + \underbrace{\frac{6d}{m}\bigg(\max_{c\in[N]}T_{c,\rho}+\frac{ \sigma^2}{m} + \frac{\gamma^2 }{md}\bigg)}_{:=U},\notag
\end{align}
where the second inequality holds by $g(\eta)\leq 1$, the third inequality holds because $\lrn{\theta_0^c-\theta_*}_2^2\leq d\mathcal{D}^2$ for any $c\in[N]$ and $\eta< \frac{2}{m}$.
In particular, the $K$-th step before synchronization yields that
\begin{align}\label{recursion_v3}
	\E{\|\theta_{K_-}^c-\theta_*\|_2^2} &\le d\mathcal{D}^2 +U.
\end{align}
Having all the results ready, for the $K$-local step after synchronization, applying Jensen's inequality
\begin{align}\label{recursion_v4}
	\E{\|\theta_K^c-\theta_*\|_2^2} 
	= & ~\E{\bigg\|\sum_{c=1}^N p_c\theta_{K-}^c-\theta_*\bigg\|_2^2} \notag \\
	\leq & ~ \sum_{c=1}^N p_c\E{\lrn{\theta_{K-}^c-\theta_*}_2^2} \notag \\
	\leq &~ d\mathcal{D}^2 + U,
 \end{align}
Now starting from iteration $K$, we adapt the recursion of Eq.(\eqref{recursion_v2}) for the $k$-th step, where $k\in\{K+1,\cdots, 2K-1, (2K)_{-}\}$ and $(2K)_-$ denotes the $2K$-step without synchronization, we have
\begin{align}\label{recursion_v5}
	&\E{\|\theta_k^c-\theta_*\|_2^2} \notag\\
	\leq & ~ g(\eta)^{k-K} \cdot  \E{\|\theta_K^c-\theta_*\|_2^2} + \frac{1- g(\eta)^{k-K}}{1 - g(\eta)}\cdot \left(2\eta d \max_{c\in[N]} T_{c,\rho} +\eta^2 d \sigma^2+\frac{2\eta \gamma^2}{m}\right)\notag \\
	\leq &  g(\eta)^{k-K}(d\mathcal{D}^2+U)+\frac{1- g(\eta)^{k-K}}{m\eta/3} \frac{m\eta}{3} U\notag \\
	\leq & d\mathcal{D}^2+ g(\eta)^{k-K} U +  (1- g(\eta)^{k-K}) U \notag\\
	\leq & d\mathcal{D}^2+U,
\end{align}
where the second inequality follows by Eq.(\eqref{recursion_v4}), the fact that $1-g(\eta)\geq \eta m/3$ and $\eta\leq \frac{2}{m}$, and the definition of $U$. The third one holds since $g(\eta)\leq 1$.

By repeating Eq.(\eqref{recursion_v4}) and (\eqref{recursion_v5}), we have that for all $k\geq 0$, we can obtain the desired uniform upper bound.
\end{proof}

\emph{Discussions:} Since the above result is independent of the learning rate $\eta$, it can be naturally applied to the setting with decreasing learning rates. The details are omitted.
$\newline$

\subsection{Bounded gradient}
\begin{lemma}[Bounded gradient in $\ell_2$ norm]\label{bounded_gradient_l2}
Given assumptions \ref{def:smooth}, \ref{def:strong_convex}, and \ref{def:variance} hold, for any client $c$ and any learning rate $\eta \in (0 , 2/m)$ and $\lrn{\theta_0^c-\theta_*}_2^2\leq d\mathcal{D}^2$ for any $c\in[N]$, we have the $\ell_2$ norm upper bound as follows
\begin{align*}
    \E{ \|\nabla\tilde f^c(\theta_k^c) \|_2^2 }\leq 14dL^2 H_{\rho},
\end{align*}
where $H_{\rho}=  \mathcal{D}^2+ \frac{1}{m}\max_{c\in[N]}T_{c,\rho} +\frac{\gamma^2}{m^2 d}+\frac{\sigma^2}{m^2}$.
\end{lemma}

\begin{proof}

Decompose the $\ell_2$ of the gradient as follows
\begin{align*}
    \E{\lrn{\nabla\tilde f^c(\theta_k^c)}_2^2}&= \E{\lrn{\nabla\tilde f^c(\theta_k^c)-\nabla f^c(\theta_k^c)+\nabla f^c(\theta_k^c)}_2^2}\notag\\
    &= \E{\lrn{\nabla f^c(\theta_k^c)}_2^2}+\E{\lrn{\nabla\tilde f^c(\theta_k^c)-\nabla f^c(\theta_k^c)}_2^2}\notag\\
    &\qquad+2\E{\lrw{\nabla\tilde f^c(\theta_k^c)-\nabla f^c(\theta_k^c), \nabla f^c(\theta_k^c)}} \notag \\
    &\leq \E{\lrn{\nabla f^c(\theta_k^c)}_2^2}+\sigma^2d \notag \\
    &=  \E{\lrn{\nabla f^c(\theta_k^c)-\nabla f^c(\theta_*)+\nabla f^c(\theta_*)}_2^2}+\sigma^2d \notag \\
    &\leq 2\E{\lrn{\nabla f^c(\theta_k^c)-\nabla f^c(\theta_*)}_2^2}+2\E{\big\|\nabla f^c(\theta_*)\big\|_2^2}+\sigma^2d\notag\\
    &\leq 2 L^2 \E{\lrn{\theta_k^c-\theta_*}_2^2}+2 \gamma^2 +\sigma^2d\notag\\
    &\leq 2d L^2 \mathcal{D}^2 + \frac{12d L^2}{m} \cdot \bigg(\max_{c\in[N]}T_{c,\rho}+\frac{ \sigma^2}{m} + \frac{\gamma^2 }{md} \bigg)+{2\gamma^2}+\sigma^2 d \notag \\
    &\leq 14 d L^2 \cdot \bigg( \mathcal{D}^2+\frac{1}{m}\max_{c\in[N]} T_{c,\rho} +\frac{\gamma^2}{m^2 d}+\frac{\sigma^2}{m^2} \bigg):= 14d L^2 H_{\rho},
\end{align*}
where the first inequality follows by Assumption \ref{def:variance}; the second inequality follows by Young's inequality; the third inequality follows by Assumption  \ref{def:smooth} and the definition that $\gamma:=\max_{c\in[N]}\lrn{\nabla f^c(\theta_*)}_2$; the fourth inequality follows by Lemma \ref{lem:L2_bound_local}; the last inequality follows by $\kappa:=\frac{L}{m}\geq 1$.
\end{proof}


\section{Initial condition}\label{sec:initial_condition}

\begin{lemma}[Initial condition] 
\label{lem:W2_init_bound}
Let $\mu_0$ denote the Dirac delta distribution at $\theta_0$. 
Then, we have
\begin{align*}
W_2(\mu_0, \pi)\leq \sqrt{2}(\| \theta_0 - \theta_* \|_2 +  \sqrt{d\tau /m} ). 
\end{align*}
\end{lemma}

\begin{proof}
By \citet{ccbj18}, there exists an optimal coupling between $\mu_0$ and $\pi$ such that
\begin{align*}
    W_2^2(\mu_0, \pi) 
    \leq & ~ \mathbb{E}_{\theta\sim \pi} [\|\theta_0-\theta\|_2^2 ]\\
    \leq & ~ 2\mathbb{E}_{\theta\sim \pi} [\|\theta_0-\theta_*\|_2^2 ] + 2 \mathbb{E}_{\theta\sim \pi}[\|\theta-\theta_*\|_2^2] \\
    = & ~ 2\| \theta_0 - \theta_* \|_2^2 +2\mathbb{E}_{\theta\sim \pi}[\|\theta-\theta_*\|_2^2]\\
    \leq & ~ 2\| \theta_0 - \theta_* \|_2^2 + 2d\tau /m,
\end{align*}
where the second step follows from triangle inequality, the last step follows from Lemma 12 \citep{dm+16} and the temperature $\tau$ is included to adapt to the time scaling.
\end{proof}

\textbf{Burkholder-Davis-Gundy inequality} Let $\phi:[0, \infty)\rightarrow \mathbb{R}^{r\times d}$ for some positive integers $r$ and $d$. In addition, we assume $\E{\int_0^{\infty} |\textcolor{black}{\phi}(s)|^2 \d s}<\infty$ and let $Z(t)=\int_0^t \textcolor{black}{\phi}(s)\d W_s$, where $W_s$ is a $d$-dimensional Brownian motion. Then for all $t\geq 0$, we have
\begin{align}\label{BDG-inequality}
    \E{\sup_{0\leq s\leq t} |Z(s)|^2}\leq 4\E{\int_0^t|\phi(s)|^2\d s}.
\end{align}

\section{More on Differential Privacy Guarantees}
\label{dp_guarantee}

We make the following assumptions for the analysis of DP.
\begin{assumption} [Bounded $\ell_2$-sensitivity] \label{assump:bdd_sens}
The gradient of loss function $l:\R^d\times \mathcal{X}\rightarrow \R,\ (\theta,x)\mapsto l(\theta;x)$ w.r.t. 
$\theta$ has a uniformly bounded $\ell_2$-sensitivity for $\forall\theta\in\R^{d}$:
\begin{equation}
    \Delta_l:=\sup_{\theta\in\R^d}\sup_{x,x'\in\mathcal{X}}\|\nabla l(\theta;x)-\nabla l(\theta;x')\|_2<\infty
\end{equation}
\end{assumption}
For example, when $l(\theta;\cdot)$ is $M$-Lipschitz for any $\theta\in\R^{d}$, $\Delta_l\le 2M$. Following the tradition of DP guarantees, we assume that the unbiased gradient is computed as follows.
\begin{assumption}[Unbiased gradient estimates] \label{assump:grad_est}
The unbiased estimate of $\nabla f^c(\theta)$ is calculated using a subset (denoted by $\cS^c$) of $\{x_{c,i}\}_{i=1}^{n_c}$ sampled uniformly at random from all the subsets of size $\gamma n_c$ of $\{x_{c,i}\}_{i=1}^{n_c}$: 
\begin{equation}
\nabla \tilde{f}^c(\theta)=\frac{1}{\gamma p_c}\sum_{i\in \cS^c}\nabla l(\theta;x_{c,i}).
\end{equation}
\end{assumption}

\begin{theorem} \label{thm:privacy_alg3_full}
Assume assumptions \ref{assump:bdd_sens} and \ref{assump:grad_est} holds. For any $\delta_0\in(0,1)$, if $\eta\in\left(0,\frac{\tau(1-\rho^2)\gamma^2\min_{c\in[N]}p_{c}}{\Delta_l^2\log(1.25/\delta_0)}\right]$, then Algorithm $\ref{alg:alg_main_text_partial_main}$ is $(\epsilon^{(3)}_{K,T},\delta^{(3)}_{K,T})$-differentially private w.r.t. $\simeq_{s}$ after executed for $T$ ($T=EK$ with $E\in \mathbb{N}, E\ge 1$) iterations where
\begin{align*}
&\epsilon^{(3)}_{K,T}=\tilde\epsilon_K\min\left\{\sqrt{\frac{2T}{K}\log\left(\frac{1}{\delta_2}\right)} + \frac{T(e^{\tilde\epsilon_{K}}-1)}{K},\ \frac{T}{K}\right\},\\
&\delta^{(3)}_{K,T}=\frac{T}{K}\tilde\delta_K+\delta_2,
\end{align*}
with 
\begin{align*}
\tilde\epsilon_K =
\begin{cases}
&\log\left(1+\left(1-\left(1-\frac{1}{N}\right)^{S}\right)\left(e^{\epsilon_K}-1\right)\right),\quad \text{under scheme I} ,\\
& \log\left(1+\frac{S}{N}\left(e^{\epsilon_K}-1\right)\right),\quad \text{under scheme II},
\end{cases}
\end{align*}
\begin{align*}
\tilde\delta_K=
\begin{cases}
& \sum_{s=1}^S{S\choose s}\left(\frac{1}{N}\right)^{s}\left(1-\frac{1}{N}\right)^{S-s}\delta_{K,s},\quad \text{under scheme I} ,\\
&\frac{S}{N}\left(K\gamma \delta_0+\delta_1\right),\quad \text{under scheme II},
\end{cases}
\end{align*}
\begin{equation*}
\delta_{K,s}= 
\frac{\left(e^{\epsilon_K}-1\right) \delta_{K,s,0}}{e^{\epsilon_K / s}-1},
\end{equation*}
\begin{equation*}
\delta_{K,s,0}=1.25K\gamma\left(\frac{\delta_0}{1.25}\right)^{1/s^2}+\delta_1,
\end{equation*}
\begin{equation*}
\epsilon_K=\epsilon_1\min\left\{\sqrt{2K\log(1/\delta_1)} + K(e^{\epsilon_1}-1),\ 
K\right\},
\end{equation*}
$$
\epsilon_1=2\Delta_l \sqrt{\frac{\eta\log(1.25/\delta_0)}{\tau(1-\rho^2)\min_{c\in [N]}p_{c}}},
$$
and $\delta_1,\delta_2\in[0,1)$.
\end{theorem}

\begin{proof}
Let $\D_c:=\{x_{c,i}\}_{i\in[n_c]}$ denote the dataset of the $c$-th client for $c\in[N]$. Let $\D:=\cup_{c\in[N]}\D_c$ denote the whole dataset. 


As FedAvg algorithms can be divided into the processes of local updates, synchronization, and broadcasting with risks of information leakage in synchronization (local model uploading and aggregation) and broadcasting, we consider the differential privacy guarantees in synchronization and broadcasting similar to \cite{wei2020federated}. Since there is no involvement of data in model aggregation and broadcasting, they are post-processing processes. Thus, it suffices to analyze the differential privacy guarantees in local model uploading.


For any two datasets $\D\simeq_{s}\D'$, there exists $c\in[N]$ such that $\D_c\simeq_{s}\D'_{c}$ and $\D_{c'}=\D'_{c'}$ for any $c'\in[N],c'\neq c$.

Consider the function 
$m_{c}(\cS;\theta)=\theta-\frac{\eta}{\gamma} \nabla f^c(\theta;\cS)=\theta-\frac{\eta}{\gamma p_c}\sum_{x\in \cS} \nabla l^c(\theta;x)$ with $|\cS|=\gamma n_c$ ($|A|$ denotes the cardinality of set $A$). 
By Assumption \ref{assump:bdd_sens}, for any $\theta\in \R^d$, the sensitivity of $m_c(\D_c;\theta)$ is
\begin{equation}
\Delta m_c:=\sup_{\cS_c\simeq_s \cS'_c}\|m_c(\cS_c;\theta)-m_c(\cS_c';\theta)\| = \frac{\eta}{\gamma p_c}\Delta_{l}
\end{equation}

For the mechanism $\M_{c}(\cS;\theta):=m_{c}(\cS;\theta)+\sqrt{2\eta\tau\rho^2}\dot{\xi}+\sqrt{2\eta(1-\rho^2)\tau/p_c}\xi$ with $\dot{\xi}$ and $\xi$ being two independent standard $d$-dimensional Gaussian vector, since $\dot{\xi}$ is broadcasted to all the clients, it can be treated as some known constant which does not contribute to the differential privacy. Thus, the standard deviation of the added Gaussian noise is $\sqrt{2\eta\tau(1-\rho^2)/p_c}$ at each dimension. Then, according to the Gaussian mechanism \cite{dwork2014algorithmic}, $M_c(\D_c;\theta)$ is $(\epsilon_{0,c},\delta_0)$-differentially private for any $\theta\in\R^d$ with
\begin{align}
\epsilon_{0,c}=c(\delta_0)\frac{\Delta_l}{\gamma}\sqrt{\frac{\eta}{2p_c\tau(1-\rho^2)}},\quad 
c(\delta_0)=\sqrt{2\log(1.25/\delta_0)},\quad
\delta_0\in (0,1).
\end{align}

For $\cS^c$ sampled uniformly at random from all the subsets of size $\gamma n_c$ of $\D_c$, define $\tilde{M}_c(\D_c;\theta):=M_c(\cS^c;\theta)$. Then, according to Theorem 9 in \cite{NEURIPS2018_3b5020bb}, $\tilde{M}_c$ is $(\log\left(1+\gamma (e^{\epsilon_{0,c}}-1)\right), \gamma\delta_0)$-differentially private. 
Notice that for any $\epsilon_{0,c}\in[0,1]$ (i.e., $0\le \eta\le \frac{2p_c\tau(1-\rho^2)\gamma^2}{\Delta_l^2c(\delta_0)^2}$), we have $0\le e^{\epsilon_{0,c}}-1\le 2\epsilon_{0,c}$ and
\begin{equation}
\log\left(1+\gamma (e^{\epsilon_{0,c}}-1)\right)\le 
\log\left(1+2\gamma\epsilon_{0,c}
\right)\le 2\gamma \epsilon_{0,c}=
c(\delta_0)\Delta_l \sqrt{\frac{2\eta}{p_c\tau(1-\rho^2)}}=: \epsilon_{1,c}.
\end{equation}

Define 
\begin{equation} \label{eq:epsilon1_def}
\epsilon_1:=c(\delta_0)\Delta_l \sqrt{\frac{2\eta}{\tau(1-\rho^2)\min_{c'\in [N]}p_{c'}}}
=2\Delta_l \sqrt{\frac{\eta\log(1.25/\delta_0)}{\tau(1-\rho^2)\min_{c\in [N]}p_{c}}}.
\end{equation}
Then, we have $\max_{c\in[N]} \epsilon_{1,c}=\epsilon_{1}$
and $\max_{c\in[N]}\epsilon_{0,c}\in[0,1]$ if
\begin{equation} \label{eq:eta_condition0}
0\le \eta\le \frac{2\tau(1-\rho^2)\gamma^2\min_{c'\in[N]}p_{c'}}{\Delta_l^2c(\delta_0)^2}=
\frac{\tau(1-\rho^2)\gamma^2\min_{c'\in[N]}p_{c'}}{\Delta_l^2\log(1.25/\delta_0)}.
\end{equation}
Thus, for $0\le \eta\le \frac{\tau(1-\rho^2)\gamma^2\min_{c'\in[N]}p_{c'}}{\Delta_l^2\log(1.25/\delta_0)}$, 
$\tilde{M}_c(\D_c;\theta)$ is $(\epsilon_1,\gamma\delta_0)$-differentially private for any $\theta\in\R^d$. From now on, we assume that \eqref{eq:eta_condition0} holds.

Define $\M^K_c(\D_c;\theta)$ to be the $K$-fold composition of $\tilde{\M}_c(\D_c;\theta)$. According to the composition rules of $(\epsilon,\delta)$-differential privacy (Theorem 3.1 and  3.3 in \cite{dwork2010boosting}), $\M^K_c(\D_c;\theta)$ is $(\epsilon_K,\delta_K)$-differentially private with
\begin{align}
\label{eq:epsilon_K}
&\epsilon_K=\min\left\{\sqrt{2K\log(1/\delta_1)}\epsilon_1 + K\epsilon_1(e^{\epsilon_1}-1),\ 
K\epsilon_1\right\},\\
\label{eq:delta_K}
&\delta_K=K\gamma\delta_0+\delta_1.
\end{align}
for any $\delta_1\in[0,1)$.

By \eqref{eq:epsilon1_def}, if
\begin{equation} \label{eq:eta_condition1}
0\le \eta\le \frac{\tau(1-\rho^2)\min_{c\in[N]}p_c}{2\Delta_l^2c(\delta_0)^2}\log^2\left(1+\sqrt{\frac{2\log(1/\delta_1)}{K}}\right),
\end{equation}
we have $\epsilon_1\in\left[0, \log\left(1+\sqrt{\frac{2\log(1/\delta_1)}{K}}\right)\right]$ which implies that
$K\epsilon_1(e^{\epsilon_1}-1)\le \sqrt{2K\log(1/\delta_1)}\epsilon_1$
and
\begin{equation} \label{eq:epsilon_K_bound}
\epsilon_K\le 2\sqrt{2K\log(1/\delta_1)}\epsilon_1. 
\end{equation}

In the synchronization process, $S$ clients selected via device-sampling scheme I or II send their local models to the center. Thus, for scheme I (with replacement) and scheme II (without replacement), according to Theorem 10 and Theorem 9 in \cite{NEURIPS2018_3b5020bb} respectively, each synchronization process is $(\tilde\epsilon_{K}, \tilde\delta_K)$-differentially private with
\begin{align} \label{eq:DP_K_partial1}
\tilde\epsilon_K =
\begin{cases}
&\log\left(1+\left(1-\left(1-\frac{1}{N}\right)^{S}\right)\left(e^{\epsilon_K}-1\right)\right),\quad \text{under scheme I} ,\\
& \log\left(1+\frac{S}{N}\left(e^{\epsilon_K}-1\right)\right),\quad \text{under scheme II},
\end{cases}
\end{align}
\begin{align} \label{eq:DP_K_partial2}
\tilde\delta_K=
\begin{cases}
& \sum_{s=1}^S{S\choose s}\left(\frac{1}{N}\right)^{s}\left(1-\frac{1}{N}\right)^{S-s}\delta_{K,s},\quad \text{under scheme I} ,\\
&\frac{S}{N}\delta_K=\frac{S}{N}\left(K\gamma \delta_0+\delta_1\right),\quad \text{under scheme II} .
\end{cases}
\end{align}
where
\begin{equation*}
\delta_{K,s}= 
\frac{\left(e^{\epsilon_K}-1\right) \delta_{K,s,0}} {e^{\epsilon_K / s}-1},\quad 
\delta_{K,s,0}=1.25K\gamma\left(\frac{\delta_0}{1.25}\right)^{1/s^2}+\delta_1.
\end{equation*}


The aggregation and broadcasting process is post-processing and preserves the guarantees of differential privacy (Proposition 2.1 in \cite{dwork2014algorithmic}).
When executed $T$ iterations, Algorithm \ref{alg:alg_main_text_partial_main} is the $T/K$-fold composition of local updates, synchronization, and broadcasting. According to the composition rules of $(\epsilon,\delta)$-differential privacy (Theorem 3.1 and  3.3 in \cite{dwork2010boosting}), Algorithm \ref{alg:alg_main_text_partial_main} is $(\epsilon^{(3)}_{K,T},\delta^{(3)}_{K,T})$-differentially private after $T$ iterations with
\begin{align}
\label{eq:epsilon_alg3}
&\epsilon^{(3)}_{K,T}=\min\left\{\sqrt{2\frac{T}{K}\log(1/\delta_2)}\tilde\epsilon_K + \frac{T}{K}\tilde\epsilon_{K}(e^{\tilde\epsilon_{K}}-1),\ \frac{T}{K}\tilde\epsilon_K\right\},\\
\label{eq:delta_alg3}
&\delta^{(3)}_{K,T}=\frac{T}{K}\tilde\delta_K+\delta_2,
\end{align}
for $\delta_1,\delta_2\in[0,1)$ and $\delta_0\in(0,1)$. Notice that under scheme II, $e^{\tilde\epsilon_K}-1=\frac{S}{N}(e^{\epsilon_K}-1)$, thus,
$\epsilon^{(3)}_{K,T}=\tilde\epsilon_K\min\left\{\sqrt{2\frac{T}{K}\log(1/\delta_2)} + \frac{TS}{KN}(e^{\epsilon_{K}}-1),\ \frac{T}{K}\right\}$ and
$\delta^{(3)}_{K,T}=\frac{S}{N}\gamma T\delta_0+ \frac{TS}{KN}\delta_1+\delta_2$.
\end{proof}


\paragraph{Discussion on the differential privacy guarantees of Algorithm \ref{alg:alg_main_text_partial_main} under scheme II} \label{sec:DP_discussion}
By \eqref{eq:epsilon_K}, \eqref{eq:delta_K}, \eqref{eq:DP_K_partial1}, \eqref{eq:DP_K_partial2}, \eqref{eq:epsilon_alg3}, and \eqref{eq:delta_alg3}, by letting $\delta_1,\delta_2=0$, we have that Algorithm \ref{alg:alg_main_text_partial_main} is at least
$(\frac{T}{K}\log\left(1+\frac{S}{N}(e^{K\epsilon_1}-1)\right),\frac{S}{N}\gamma T\delta_0)$-differentially private.

If
\begin{equation}
\tilde\epsilon_K\le \log\left(1+\sqrt{\frac{2K\log(1/\delta_2)}{T}}\right),
\end{equation}
we have $\frac{T}{K}\tilde\epsilon_K(e^{\tilde\epsilon_K}-1)\le \sqrt{2\frac{T}{K}\log(1/\delta_2)}\tilde\epsilon_K$ and therefore
\begin{equation} \label{eq:epsilon_alg3_bound0}
\epsilon^{(3)}_{K,T}\le 2\sqrt{2\frac{T}{K}\log(1/\delta_2)}\tilde\epsilon_K=2\sqrt{2\frac{T}{K}\log(1/\delta_2)}\log\left(1+\frac{S}{N}(e^{\epsilon_K}-1)\right).
\end{equation}
Now assume $\eta$ satisfies \eqref{eq:eta_condition1}, then by \eqref{eq:epsilon_K_bound} and \eqref{eq:epsilon_alg3_bound0},
\begin{equation} \label{eq:epsilon_alg3_bound1}
\epsilon^{(3)}_{K,T}\le 2\sqrt{2\frac{T}{K}\log(1/\delta_2)}\log\left(1+\frac{S}{N}(e^{2\sqrt{2K\log(1/\delta_1)}\epsilon_1}-1)\right)
\end{equation}
Notice that if 
\begin{equation} \label{eq:eta_condition_alg3}
0\le \eta\le \frac{\tau(1-\rho^2)\min_{c\in[N]}p_c}{2\Delta_l^2c(\delta_0)^2}
\min\left\{
\log^2\left(1+\sqrt{\frac{2\log(1/\delta_1)}{K}}\right),\ 
\frac{\log^2\left(1+\frac{N}{S}\sqrt{\frac{2K\log(1/\delta_2)}{T}}\right)}{8K\log(1/\delta_1)}
\right\},
\end{equation}
by \eqref{eq:epsilon1_def}, \eqref{eq:epsilon_K_bound}, \eqref{eq:DP_K_partial1}, and \eqref{eq:DP_K_partial2}, we have 
\begin{align}
&\epsilon_1\le \frac{\log\left(1+\frac{N}{S}\sqrt{\frac{2K\log(1/\delta_2)}{T}}\right)}{2\sqrt{2K\log(1/\delta_1)}}, \\
&\epsilon_K\le 2\sqrt{2K\log(1/\delta_1)}\epsilon_1\le \log\left(1+\frac{N}{S}\sqrt{\frac{2K\log(1/\delta_2)}{T}}\right),\\
&\tilde\epsilon_K\le \log\left(1+\sqrt{\frac{2K\log(1/\delta_2)}{T}}\right).
\end{align}
Thus, we have \eqref{eq:epsilon_alg3_bound1} holds.
Plugging \eqref{eq:epsilon1_def} into \eqref{eq:epsilon_alg3_bound1}, we have
\begin{equation} \label{eq:epsilon_alg3_bound2}
\epsilon^{(3)}_{K,T}\le 2\sqrt{2\frac{T}{K}\log(1/\delta_2)}\log\left(1+\frac{S}{N}\left(\exp\left\{4c(\delta_0)\Delta_l\sqrt{\frac{\eta K\log(1/\delta_1)}{\tau(1-\rho^2)\min_{c\in[N]}p_c}}\right\}-1\right)\right).
\end{equation}

For any $K\ge 1$, if $\eta K\ll 1$ and $T\gg 1$, using $\log(1+x)\approx x$ and $e^{x}-1\approx x$ when $|x|\ll 1$, we can write \eqref{eq:eta_condition_alg3} as 
\begin{equation}
0\le \eta=O\left(\frac{\tau(1-\rho^2)N^2\min_{c\in[N]}p_c\log(1/\delta_2)}{\Delta_l^2 S^2 T\log(1/\delta_0) \log(1/\delta_1)}\right),
\end{equation}
and \eqref{eq:epsilon_alg3_bound2} as
\begin{equation}
\epsilon_{K,T}^{(3)}=O\left(\frac{S\Delta_l}{N}\sqrt{\frac{\eta T \log(1/\delta_0)\log(1/\delta_1)\log(1/\delta_2)}{\tau(1-\rho^2)\min_{c\in[N]}p_c}}\right).
\end{equation}

\end{document}